\documentclass{article}

\usepackage[main, final]{neurips_2026}
\renewcommand{\acksection}{\section*{Acknowledgments}}

\usepackage[utf8]{inputenc}
\usepackage[T1]{fontenc}
\usepackage{hyperref}
\usepackage{url}
\usepackage{booktabs}
\usepackage{longtable}
\usepackage{amsfonts}
\usepackage{nicefrac}
\usepackage{microtype}
\usepackage{xcolor}
\usepackage{amsmath}
\usepackage{multirow}
\usepackage{subcaption}
\usepackage{algorithm}
\usepackage{algorithmic}
\usepackage{graphicx}
\usepackage{calc}
\usepackage{amsthm}
\usepackage{amssymb}
\usepackage{wrapfig}
\usepackage{float}
\newtheorem{proposition}{Proposition}
\newtheorem{hypothesis}{Hypothesis}
\newtheorem{corollary}{Corollary}
\DeclareMathOperator*{\argmin}{arg\,min}

\title{Platonic Task Arithmetic}

\author{%
  Junghwan Park \\
  TelePIX \\
  \texttt{junghwan@telepix.net} \\
  \And
  Woojin Cho \\
  TelePIX \\
  \texttt{woojin@telepix.net} \\
}

\begin{document}

\maketitle

\begin{abstract}
Distinct pre-trained models specialized for the same task converge to closely similar behavior, yet the parameter updates that produce it share no common coordinate system. Weight-space task arithmetic is therefore confined to a single model, and transporting an update between models requires a structural correspondence. Drawing on Plato's allegory of the cave, we hypothesize that these model-specific updates are shadows cast by one shared, model-agnostic object, the platonic task vector. To make it operational across models of different architectures, we introduce Universal Task Descriptors, matrices whose shape is independent of architecture and embedding dimension, which record a task's functional effect and admit addition and negation as ordinary matrix operations, and we transfer a descriptor into a target in two ways. A single least-squares solve returns a linear operator folded into the target's last layer, and a bank of such operators, one per source and task, realizes any composition as a signed sum of its entries. Alternatively, a low-rank adapter of the target's encoder is trained on the same objective at the price of one optimization per edit. Despite a model-specific residual comparable in norm to the shared component, transfer from another model retains $74$ to $80$ percent of the gain the target's own descriptors attain. Experiments across six model families, eight tasks and audio--text models confirm both realizations. Code is available at \url{https://github.com/jhpark-ai/Platonic-Task-Arithmetic}.
\end{abstract}

\section{Introduction}
\label{sec:intro}

Task arithmetic~\citep{ilharco2022editing} turned fine-tuning into algebra. A task vector $\delta = \theta_{\text{ft}} - \theta_{\text{pre}}$ records what fine-tuning changed, and tasks are then composed by adding such vectors, removed by subtracting them and combined by analogy, all without further training. The paradigm has since acquired better composition rules~\citep{yang2023adamerging,ortiz2023task,yu2024language,sun2025task,gargiulo2025task}, parameter-efficient variants~\citep{chitale2023task,huang2023lorahub} and language-model instantiations~\citep{zhou2024metagpt}, but it has kept one restriction from the start, that every task vector share an architecture and an initialization with every other, which recent transport methods~\citep{rinaldi2026gradient,rinaldi2026transporting} relax only between models that correspond structurally. Knowledge learned on one model therefore cannot be combined with knowledge learned on another built differently, a restriction that bites harder as vision--language models~\mbox{\citep{radford2021learning,zhai2023sigmoid,ilharco2021openclip}} proliferate across architectures, training data and objectives.

We argue that the restriction is a property of the coordinates rather than of the knowledge. Fine-tuning heterogeneous models on the same task changes their cross-modal affinity, by which we mean the cosine similarity between an encoded image and an encoded text prompt, in a consistent and task-specific way even though the weight updates share nothing (Section~\ref{sec:motivation}). This echoes the Platonic Representation Hypothesis (PRH)~\citep{huh2024platonic}, that large models converge toward one representation up to model-specific transformations~\citep{maniparambil2024vision,tjandrasuwita2025understanding}; PRH concerns what models represent, and we ask whether the same convergence governs what fine-tuning changes, finding a clear though imperfect version of it. We therefore posit a shared, model-agnostic space in which fine-tuning on a task is one object whatever model performs it, the platonic task vector, of which each model's task vector is a shadow cast into that model's coordinates, so that task arithmetic across models is arithmetic in that space. No model exposes the space directly, and we reach it through the Universal Task Descriptor (UTD), a matrix recording how fine-tuning shifts the affinity between probe inputs and semantic prompts. A UTD has $N$ probe inputs by $M$ prompts whatever the embedding dimension, so descriptors measured on different models can be added, negated and combined directly.

A descriptor states a functional change that a target must then be edited to exhibit, and we pose the transfer as a single objective, namely finding the edit of the target whose induced descriptor matches the prescribed one on the task's probes. The edit turns out to need very little. The descriptor factorizes as $\operatorname{UTD} = \Delta F\, G^\top$, with $\Delta F$ the shifts fine-tuning induced in the probe images' embeddings and $G$ the prompt landmarks, so what has to be reproduced is one map on embeddings, of which the descriptor is the projection onto $\operatorname{span}(G)$. Modeling that map as linear, $f \mapsto (I + \alpha A) f$ for an image embedding $f$, a matrix $A$ and a strength $\alpha$ fixed in advance rather than tuned per edit, turns the transfer into a ridge-regularized least-squares problem in target-side quantities alone, whose solution folds into the target's own last layer as a weight edit that leaves architecture and parameter count untouched. Composition then moves to operator space, where adding task displacements is adding matrices, the linear solution map adds no error of its own on a shared probe set, and operators estimated on different tasks interact in an exact form that we turn into a margin condition and test on every evaluation image. A target can therefore hold a bank of operators, one entry per source and task, after which any composition at any strength is a signed sum of bank entries. The same objective can also be minimized directly, by training a low-rank adapter~\citep{hu2022lora} of the target's vision encoder, which is not confined to one layer and fits the terms of a composition jointly but pays for both with one optimization per edit and no reusable bank. The closed form is our main route and the adapter measures what a richer edit buys, and in either case what a task carries to a new target is its descriptor together with its unlabeled probe images and class names, never its weights or a label. A target therefore acquires a task without seeing one of its labels, and because the descriptor prescribes only the effect of the edit on the probes, any target that admits the edit can receive it, whatever its architecture, embedding dimension or training corpus.

Across six model families, eight tasks and $3{,}360$ addition cells, a target absorbing two tasks from two other models through the closed form lifts both by $18.1$ and $11.6$ points for about two points of CIFAR-100, retaining $80$ and $74$ percent of what its own descriptors deliver although descriptors of one task from different models agree only at cosine $0.47$, and negating one from another model removes $35.2$ points for $4.5$ of the same control. The adapter, fitted per edit to the same descriptors, is more accurate still and leaves ImageNet within a fraction of a point of zero-shot, and both routes reach a remote-sensing dataset whose images and class names enter no fit and carry unchanged to audio--text models. Our contributions are:

\begin{itemize}
  \item We posit the platonic task vector, of which every model's task vector is a shadow, make it observable through UTDs, and test the hypothesis linking the two on three predictions it makes.
  \item We pose transfer as matching a descriptor on the target and give two realizations, a closed form in which the edit is one operator folded into the target's last layer and composition is a sum of operators, exact on a shared probe set and with an exact departure and margin condition across probe sets, and a low-rank adapter fitted to the same objective.
  \item We evaluate both on addition and negation across eight tasks, six source families and six targets, the closed form on all $3{,}360$ addition and $240$ negation cells and the adapter on seeded subsets of them, as well as on a dataset that enters no fit and on audio--text models, against weight-space composition, weight transport, gradient ascent, distillation and supervised tuning.
\end{itemize}

Prior work aligns pre-trained models, through similarities to shared anchors~\citep{moschella2022relative,maiorca2023latent} or learned connectors between layers~\citep{lenc2015understanding,bansal2021revisiting}, rather than measuring what fine-tuning changes, and merges heterogeneous architectures~\citep{du2025adamms,cui2026transport} only for models trained on one task. The adapter route is an instance of relational distillation~\citep{park2019relational,tian2019contrastive} that fits a change of affinities rather than a teacher's output~\citep{hinton2015distilling} (Appendix~\ref{sec:related}).

\section{Background and Motivation}
\label{sec:motivation}
\label{sec:prelim-weight}

Given pre-trained parameters $\theta_{\text{pre}}$ and parameters $\theta_{\text{ft}}$ fine-tuned on a task $\tau$, the weight-space task vector is $\delta_\tau := \theta_{\text{ft}} - \theta_{\text{pre}}$. Task arithmetic composes such vectors by addition, $\theta_{\text{pre}} + \sum_i \delta_{\tau_i}$, by negation, $\theta_{\text{pre}} - \lambda \delta_\tau$ with strength $\lambda \geq 0$, and by analogy, $\delta_{D_2,S_2} = \delta_{D_1,S_2} - \delta_{D_1,S_1} + \delta_{D_2,S_1}$ with $D_1, D_2$ two domains and $S_1, S_2$ two subtasks. Each is defined only when every $\delta$ originates from the same $\theta_{\text{pre}}$, which confines all tasks to one base model.

\paragraph{Fine-tuning is far more consistent across models than weights suggest.}
\label{sec:functional-invariance}
If a task specifies a pattern of image--text alignment that does not depend on the model learning it, fine-tuning should move image--prompt affinities toward a shared destination from model-specific starting points, and models should agree on the induced shift up to a residual absorbing initialization, optimization and finite-data noise. We test this for $K = 6$ models and eight tasks with the descriptor of Section~\ref{sec:method-utd}, comparing descriptors pairwise as flattened vectors so that both panels of Figure~\ref{fig:utd-verification} share one scale. Different models fine-tuned on the same task agree at an off-diagonal mean cosine of $0.473 \pm 0.066$, whereas one model fine-tuned on different tasks agrees at $0.227 \pm 0.169$. Same-task agreement across models is thus roughly twice cross-task agreement within a model, and the cross-task separation rules out a generic alignment unrelated to the task. The cross-task panel is moreover structured rather than uniformly low, since the pair that agrees most is EuroSAT with RESISC45, two remote-sensing tasks, at $0.573$, the three digit-like tasks agree with each other at between $0.320$ and $0.442$, and pairs sharing little sit at or below zero, as Cars and SUN397 do at $-0.149$. The first of these lies above the cross-model mean, so the separation holds on average rather than for every pair. A descriptor therefore records which task was learned, not merely that some task was.

\begin{figure}[t]
    \centering
    \captionsetup[sub]{font=footnotesize}
    \newlength{\hmapwidth}
    \setlength{\hmapwidth}{0.27\textwidth}
    \newlength{\hmapcol}
    \setlength{\hmapcol}{0.32\textwidth}
    \newlength{\cbarcol}
    \setlength{\cbarcol}{0.05\textwidth}
    \newlength{\hmapheight}
    \settoheight{\hmapheight}{\includegraphics[width=\hmapwidth]{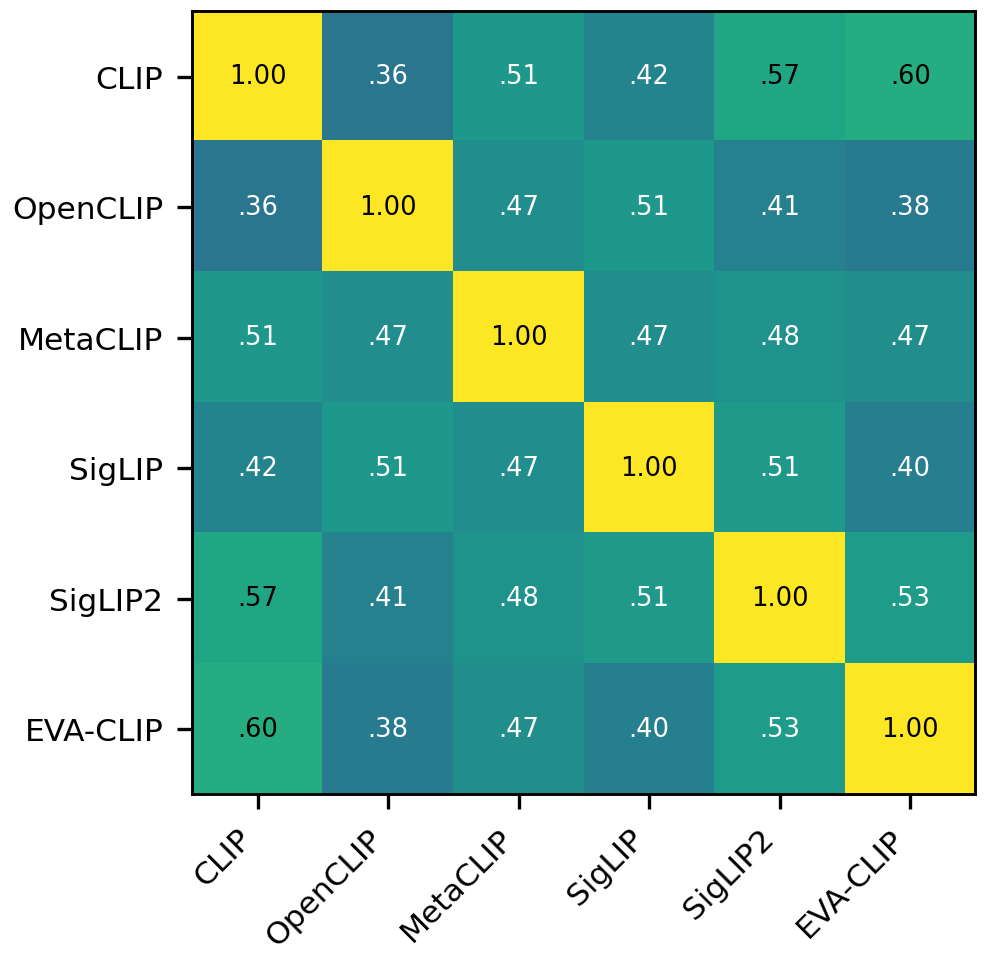}}
    \begin{minipage}[c]{0.69\textwidth}
        \centering
        \begin{minipage}[t]{\hmapcol}
            \centering
            \includegraphics[width=\hmapwidth]{figure/utd_cross_model.png}
            \subcaption{Across models, fixed task.}
            \label{fig:utd-cross-model}
        \end{minipage}%
        \vspace{-5pt}%
        \begin{minipage}[t]{\hmapcol}
            \centering
            \includegraphics[width=\hmapwidth]{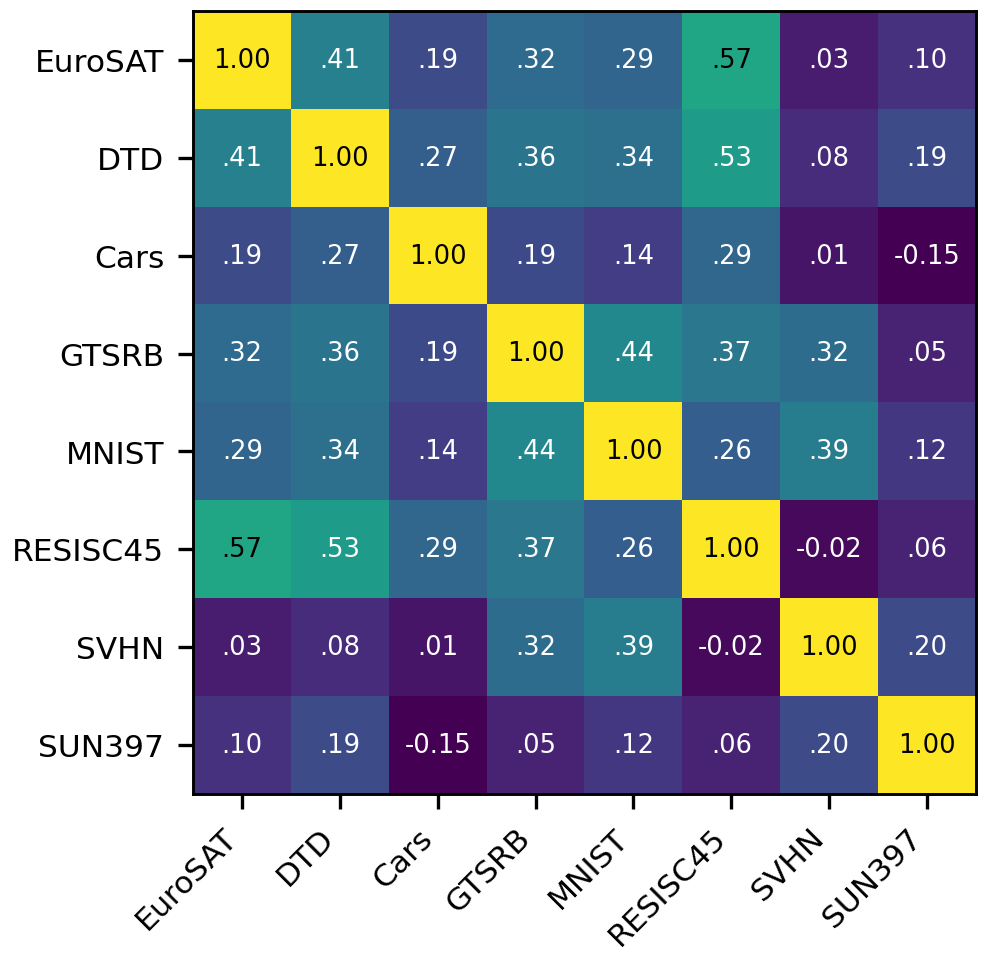}
            \subcaption{Across tasks, fixed model.}
            \label{fig:utd-cross-task}
        \end{minipage}%
        \begin{minipage}[t]{\cbarcol}
            \raggedright
            \includegraphics[height=\hmapheight, keepaspectratio=true]{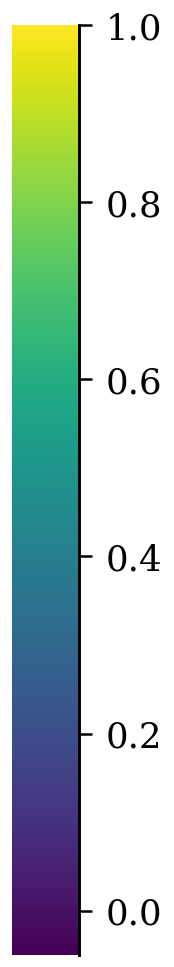}
            \subcaption*{~}
        \end{minipage}
    \end{minipage}%
    \hfill
    \begin{minipage}[c]{0.30\textwidth}
    \caption{Pairwise cosine similarity between fine-tuning-induced descriptors, flattened to a common $N \times M$ shape. (a) Across models for a fixed task, a $K \times K$ matrix averaged over the eight tasks. (b) Across tasks for a fixed model, an $8 \times 8$ matrix averaged over the six families.}
    \label{fig:utd-verification}
    \end{minipage}
\end{figure}

\section{Universal Task Descriptors and the Platonic Hypothesis}
\label{sec:method}

\subsection{The Platonic Task Vector}
\label{sec:method-platonic}

Consider $K$ vision--language models $\mathcal{M}_1, \dots, \mathcal{M}_K$, each with a vision encoder $f^{(k)}$ and a text encoder $g^{(k)}$ producing unit-norm embeddings, so that the affinity $\langle f^{(k)}(x), g^{(k)}(p) \rangle$ between an image and a prompt is well defined; we adapt only $f^{(k)}$ and hold $g^{(k)}$ fixed. An operation combines $J$ task terms indexed by $\ell$, and a source assignment $\sigma$ records which model supplies each term. Fine-tuning $\mathcal{M}_k$ on task $\tau$ yields a weight change $\delta^{(k)}_\tau \in \mathbb{R}^{P^{(k)}}$, with $P^{(k)}$ the parameter count of $\mathcal{M}_k$, and since the spaces $\mathbb{R}^{P^{(1)}}, \dots, \mathbb{R}^{P^{(K)}}$ admit no canonical correspondence these changes can be neither compared nor added across models.

The regularity of Section~\ref{sec:functional-invariance} is what a single model-agnostic object beneath the per-model changes produces, and we posit it. The platonic task vector of task $\tau$ is a task-displacement map $\Pi_\tau : \mathcal{X} \to \mathcal{V}$, defined without reference to any model, where $\mathcal{X}$ is the space of probe inputs and $(\mathcal{V}, \langle \cdot, \cdot \rangle)$ a real inner-product space shared by all models, and each $\delta^{(k)}_\tau$ is the shadow $\Pi_\tau$ casts into the coordinates of model $k$. Platonic task arithmetic composes tasks by adding their platonic task vectors and removes one by subtracting it, whichever models learned them, and since no model exposes $\mathcal{V}$ directly it is carried out through what every model can observe. In the terms of the allegory, each model sees a task only as the shadow cast in its own coordinates, and the descriptor below is a shadow that takes one shape on every model.

\subsection{Universal Task Descriptors}
\label{sec:method-utd}

Rather than translating between model-specific projections of $\Pi_\tau$, we observe it through a quantity that bypasses parameters altogether. The text encoder already serves as a measurement apparatus~\citep{moschella2022relative,maiorca2023latent}, so a descriptor built from differences of affinities before and after fine-tuning is measurable on any model. A probe set belongs to one task and consists of $N$ images $\{x_i\}_{i=1}^{N}$ and $M$ semantic prompts $\{p_j\}_{j=1}^{M}$, fixed once and reused by every model that measures or receives that task's descriptor, each prompt being the task's template applied to one of its class names. Per-image labels enter neither the choice of probes nor the objective of Section~\ref{sec:method-operator}. For a model $\mathcal{M}_k$ fine-tuned on task $\tau$, the Universal Task Descriptor is the matrix $\operatorname{UTD}^{(k)}_\tau \in \mathbb{R}^{N \times M}$ with entries
\begin{equation}
    \bigl[ \operatorname{UTD}^{(k)}_\tau \bigr]_{ij}
    \;=\;
    \bigl\langle f^{(k)}_{\text{ft}}(x_i),\, g^{(k)}(p_j) \bigr\rangle
    \,-\,
    \bigl\langle f^{(k)}_{\text{pre}}(x_i),\, g^{(k)}(p_j) \bigr\rangle .
    \label{eq:utd-def}
\end{equation}
Two properties make UTDs comparable across models. Their shape $N \times M$ is independent of embedding dimension and architecture, and each entry is the change of a cosine rather than a coordinate in an embedding space, so it refers to no basis, and differencing removes each model's own pre-trained affinities and keeps only what fine-tuning changed. Parenthesized superscripts index models and subscripts tasks, so $\operatorname{UTD}^{(\sigma(\ell))}_{\tau_\ell}$ is task term $\ell$ on its assigned model, and a composed descriptor, belonging to no single source, is written $U$ with a role label such as $U_{\text{add}}$.

\subsection{The Cross-Model Hypothesis}
\label{sec:method-connection}

The UTD is observable, whereas $\Pi_\tau$ is not. Associating each prompt $p$ with a shared semantic vector $\psi(p) \in \mathcal{V}$, define the measurement map $\mathcal{A}(\Pi) = [\langle \Pi(x_i), \psi(p_j)\rangle]_{ij} \in \mathbb{R}^{N \times M}$, the affinity shift that $\Pi$ would produce if it could be read directly. The paper rests on one hypothesis about it.
\begin{hypothesis}[Cross-model consistency]
\label{hyp:consistency}
For every model $k$ and task $\tau$,
\begin{align}
    \operatorname{UTD}^{(k)}_\tau &\;=\; \mathcal{A}(\Pi_\tau) \,+\, E^{(k)}_\tau,
    \label{eq:central-assumption} \\[-2pt]
    \bigl\| E^{(k)}_\tau \bigr\|_F &\;\le\; \eta \, \bigl\| \mathcal{A}(\Pi_\tau) \bigr\|_F ,
    \label{eq:residual-bound}
\end{align}
with $E^{(k)}_\tau$ a model-specific residual and $\eta \ge 0$ a constant uniform over models and tasks.
\end{hypothesis}

Appendix~\ref{app:theory} recovers the hypothesis to first order from the near-isometric factorization behind PRH, with a residual that grows with each model's departure from isometry. It makes three predictions that could each fail. Descriptors of one task from different models should agree more than those of different tasks from one model, as Figure~\ref{fig:utd-verification} confirms at $0.473$ against $0.227$. A descriptor from another model should recover much of what the target's own delivers, and less the larger its residual after the solve, as Section~\ref{sec:exp-addition} confirms at $80$ and $74$ percent retained and Appendix~\ref{app:transport-error} at rank correlation $-0.62$ over $150$ combinations of a source, a target and a task. And a descriptor of equal norm without the shared component should transfer nothing, as Random-UTD confirms. The residual is not small, since its dispersion across models implies $\eta \ge 0.33$ and the pairwise agreement puts it at a norm comparable to the shared component (Appendix~\ref{app:error}), and the experiments measure what transfer and composition retain at that distance.

\section{Cross-Model Task Arithmetic and Transfer to a Target}
\label{sec:method-arithmetic}

\subsection{Addition and Negation in UTD Space}
\label{sec:method-operations}

\paragraph{Addition.}
Given $J$ task terms with source assignment $\sigma$,
\begin{equation}
    U_{\text{add}}
    \;=\;
    \sum_{\ell=1}^{J} \operatorname{UTD}^{(\sigma(\ell))}_{\tau_\ell}
    \;=\;
    \mathcal{A}\Bigl( \sum_{\ell=1}^{J} \Pi_{\tau_\ell} \Bigr)
    \,+\, \sum_{\ell=1}^{J} E^{(\sigma(\ell))}_{\tau_\ell},
    \label{eq:arith-add}
\end{equation}
where the source models may all differ from one another and from the target. A sum of descriptors measured on different models is thus the measurement of the composed platonic task vector up to the terms' residuals, whose norm is at most $\eta \sum_{\ell} \| \mathcal{A}(\Pi_{\tau_\ell}) \|_F$. Matrices of different shapes have no sum, so this is a statement about terms read on one probe set, and neither realization below forms $U_{\text{add}}$ (Section~\ref{sec:method-match}).

\paragraph{Negation.}
Sign-flipping a descriptor with strength $\lambda \geq 0$ gives $U_{\text{neg}} = -\lambda \operatorname{UTD}^{(k)}_{\tau} = \mathcal{A}(-\lambda \Pi_\tau) - \lambda E^{(k)}_\tau$. Two properties follow from the form alone. The only descriptor entering the transfer is the removed task's, with no control data and no preservation term, and what is negated is the change fine-tuning induced rather than an absolute output configuration, so the target's pre-trained behavior is untouched by construction, which is the structural reason for the preservation observed in Section~\ref{sec:exp-negation}.

\subsection{Transfer as Descriptor Matching}
\label{sec:method-match}

An operation prescribes, for each of its $J$ terms, a task $\tau_\ell$, the model $\sigma(\ell)$ that supplies its descriptor and a signed weight $c_\ell$, equal to $1$ for every term of an addition and to $-\lambda$ for a negation. Transferring it into a target means finding an edit $\mathcal{E}$ of the target's vision encoder whose induced descriptor matches the prescribed one on every term's own task. Writing $\operatorname{UTD}_{\tau}[\mathcal{E}]$ for the descriptor that $\mathcal{E}$ induces on the target, computed as in Equation~\eqref{eq:utd-def} with the edited and the unedited target in place of the fine-tuned and the pre-trained model and read on the $N_\tau$ probes and $M_\tau$ prompts of task $\tau$, the objective is
\begin{equation}
    \mathcal{L}(\mathcal{E})
    \;=\;
    \sum_{\ell=1}^{J} \frac{1}{M_{\tau_\ell}}
    \Bigl\| \operatorname{UTD}_{\tau_\ell}[\mathcal{E}] - c_\ell \operatorname{UTD}^{(\sigma(\ell))}_{\tau_\ell} \Bigr\|_F^2 ,
    \label{eq:match}
\end{equation}
which weighs every probe image equally and never compares matrices of different shapes. The two realizations below differ only in which edits they admit and in how they minimize Equation~\eqref{eq:match}.

\subsection{The Descriptor Factorizes, and Transfer Is Linear}
\label{sec:method-operator}

Fix a task and let $F_{\text{pre}} \in \mathbb{R}^{N \times d}$ stack the target's unit-norm embeddings of that task's probes, with $d$ its embedding dimension, let $G \in \mathbb{R}^{M \times d}$ hold its landmarks for that task's prompts, and let $\Delta F \in \mathbb{R}^{N \times d}$ be the change an edit makes to $F_{\text{pre}}$. Since every entry of Equation~\eqref{eq:utd-def} is an inner product, the descriptor that edit induces is
\begin{equation}
    \operatorname{UTD} \;=\; \Delta F\, G^\top ,
    \label{eq:factorization}
\end{equation}
so a descriptor observes the map on image embeddings that produces it only through its projection onto $\operatorname{span}(G)$, discarding $d - \operatorname{rank}(G)$ dimensions, and reproducing a descriptor is a strictly weaker requirement than reproducing a fine-tuned encoder.

We model the map as linear and input-independent because a linear map is the first-order term of any smooth edit, the descriptor is itself a first-order object in the small-perturbation regime of Appendix~\ref{app:theory}, and a single matrix shared by all inputs is the only form that folds into a weight. Writing $f_{\text{pre}}(x)$ for the target's pre-trained embedding, we seek
\begin{equation}
    f_{\text{new}}(x) \;=\; (I + \alpha A)\, f_{\text{pre}}(x),
    \label{eq:operator}
\end{equation}
so that $\Delta F = \alpha F_{\text{pre}} A^\top$ before renormalization, and matching a target descriptor $U$ with this shift at $\alpha = 1$, which is Equation~\eqref{eq:match} for a single term to first order in the edit and up to renormalization (Appendix~\ref{app:ablations}), becomes the ridge-regularized least-squares problem
\begin{equation}
    A^{*}(U)
    \;=\;
    \argmin_{X} \;
    \bigl\| F_{\text{pre}} X G^\top - U \bigr\|_F^2
    + \mu_b \bigl\| F_{\text{pre}} X \bigr\|_F^2
    + \mu_a \bigl\| X G^\top \bigr\|_F^2
    + \mu_a \mu_b \|X\|_F^2 ,
    \label{eq:lsq}
\end{equation}
with $A = (X^{*})^\top$ and ridge coefficients $\mu_a, \mu_b > 0$, one per factor. The solution is available in closed form from the singular value decompositions of $F_{\text{pre}}$ and $G^\top$, which depend on the target and the task alone and are reused across that task's source descriptors (Appendix~\ref{app:solver}). The descriptor pins $A$ down only as read against $\operatorname{span}(G)$, and the ridge terms settle the remaining directions by minimum norm, which is why the edit can change behavior on prompts that never entered the fit (Section~\ref{sec:exp-holdout}).

\paragraph{Absorbing the operator into the target.}
For a dual encoder whose image embedding is a projection matrix $W_p$ applied to the vision tower's pooled output $h(x)$, so that $f_{\text{pre}}(x) = W_p h(x)$, Equation~\eqref{eq:operator} is absorbed exactly by
\begin{equation}
    W_p \;\longleftarrow\; (I + \alpha A)\, W_p ,
    \label{eq:absorb}
\end{equation}
which leaves architecture, parameter count and inference cost unchanged, so the probe set is needed only while solving. CLIP, OpenCLIP and MetaCLIP admit the absorption directly, EVA-CLIP and the audio models of Section~\ref{sec:exp-audio} end in an affine layer whose weight and bias both receive the edit by left multiplication, and SigLIP and SigLIP2, whose residual pooling head leaves no single map carrying the pooled feature, receive the solution of Equation~\eqref{eq:lsq} for an update to that head's last affine, with $F_{\text{pre}}$ replaced by that layer's own input (Appendix~\ref{app:head}). The scalar $\alpha$ plays the role of the merging coefficient in weight-space task arithmetic and is fixed in advance for each operation rather than chosen per cell.

\subsection{Composing by Adding Operators}
\label{sec:method-linearity}

The closed form does not minimize Equation~\eqref{eq:match} over all terms at once. It solves each term on its own task and realizes the composition in operator space. Under Equation~\eqref{eq:operator} the effect of a task on the target is a matrix, and a composed task, whose platonic task vector is the sum of its terms', is realized by the sum of their matrices, since for the matrices $A_1$ and $A_2$ of two tasks $F_{\text{pre}}(A_1 + A_2)^\top = F_{\text{pre}} A_1^\top + F_{\text{pre}} A_2^\top$ holds for every input, not only the probes. A probe set enters only in estimating each matrix, the step the proposition concerns.

\begin{proposition}[Linearity of the solution map]
\label{prop:linearity}
At full rank the map $U \mapsto A^{*}(U)$ defined by Equation~\eqref{eq:lsq} is linear, so that for scalars $c_1, c_2$ and descriptors $U_1, U_2$,
\begin{equation}
    A^{*}(c_1 U_1 + c_2 U_2) \;=\; c_1 A^{*}(U_1) + c_2 A^{*}(U_2).
    \label{eq:solve-linear}
\end{equation}
\end{proposition}

The proof is immediate from the solution's form, which is linear in $U$ (Appendix~\ref{app:solver}); ``full rank'' excludes only the rank-capped variant of Appendix~\ref{app:rank}, whose truncation is nonlinear in $U$, and numerically the identity holds to a relative difference below $5 \times 10^{-14}$ in double precision (Appendix~\ref{app:linearity}).

\paragraph{Scope of the theory.}
On one probe set the operator of a signed combination of terms is exactly that combination of their operators, so the transfer adds no composition error beyond the terms' own residuals in Equation~\eqref{eq:arith-add}, and every strength sweep and negation cell is an instance. Operators estimated on different probe sets are not covered, and no assumption-free statement could cover them, since a direction one probe set cannot observe may act on another. What holds exactly is how their sum departs from each term, namely through the other terms' descriptors read with the target's own probe and prompt weights, and a prediction survives composition whenever its margin exceeds the perturbation this leakage induces (Appendix~\ref{app:interference}). On the $2{,}240$ cells of the four targets that absorb the operator, the condition certifies $47.8$ percent of evaluation images, none of which the sum moves, and predicts the observed loss per cell and task at rank correlation $0.66$.

Since a descriptor is measured and its operator solved on the probes and prompts of one task, solving once per pair of a source model and a task gives a bank, after which any expression over those tasks, any assignment of sources and any strength is a signed sum of bank entries followed by one matrix product into the target's last layer. A bank stays valid when a task is added, and each target below holds $48$ entries, one per source family and task.

\subsection{Realization by a Low-Rank Adapter}
\label{sec:method-adapter}

The objective of Equation~\eqref{eq:match} can also be minimized without restricting the edit to one layer. We attach a low-rank adapter~\citep{hu2022lora} with parameters $\Delta\theta$ to the query and value projections of every attention layer of the target's vision encoder and minimize $\mathcal{L}$ over $\Delta\theta$ by stochastic gradient descent on the probes of all $J$ terms at once, each probe image read against the prompts of its own task. Three things change relative to the closed form. The edit reaches every layer of the encoder, so every target admits it in the same way, the two SigLIP families included, and it is not limited to what a linear map on the embedding can express. The terms of a composition are fitted jointly, so the interference that the bank bounds is optimized against directly. And every edit costs one optimization, about ten minutes on a single GPU for a pair of tasks, with nothing stored that a later edit can reuse, so each strength and each assignment of sources is a separate fit. The descriptors, probes and prompts are those of the closed form, and each route is applied at the strength its construction fixes, the closed form scaling its operator by $\alpha$ and the adapter realizing the prescribed descriptor itself (Appendix~\ref{app:adapter}).

\paragraph{What a task carries to a new target.}
A bank entry is solved, and an adapter fitted, from the source's descriptor and the task's $N$ unlabeled probe images and $M$ class names; no source weight, label or source forward pass enters, and a solved entry needs no probes. The construction is therefore label-free but not data-free, whereas a weight-space task vector needs nothing on the target side but reaches only its own base model, and every route of Table~\ref{tab:axes} that crosses models also consumes target-side data. An eighth of each task's probes costs $5.7$ points (Appendix~\ref{app:ablations}), and Section~\ref{sec:exp-holdout} evaluates a dataset whose own data enter no fit.

\section{Experiments}
\label{sec:experiments}

All experiments use seed $42$ (Appendix~\ref{app:setup}); closed-form edits are scored from cached features bit-identically to the weight edit (Appendix~\ref{app:algorithms}) and adapter edits by a forward pass through the adapted encoder.

\subsection{Experimental Setup}
\label{sec:exp-setup}

\textbf{Models.} Our six families differ in architecture, corpus and objective, namely CLIP~\citep{radford2021learning}, OpenCLIP~\citep{ilharco2021openclip}, SigLIP~\citep{zhai2023sigmoid}, SigLIP2~\citep{tschannen2025siglip}, MetaCLIP~\citep{xu2023demystifying} and EVA-CLIP~\citep{sun2023eva}, all at Large scale. All six serve as sources and all six as targets, SigLIP and SigLIP2 receiving the closed form through the head update of Section~\ref{sec:method-operator} and the adapter like every other target. Appendix~\ref{app:scale} adds four targets at $512$ and $1024$ dimensions.

\textbf{Tasks.} EuroSAT~\citep{helber2019eurosat}, DTD~\citep{cimpoi2014describing}, Cars~\citep{krause20133d}, GTSRB~\citep{stallkamp2011german}, MNIST~\citep{lecun1998mnist}, RESISC45~\citep{cheng2017remote}, SVHN~\citep{netzer2011reading} and SUN397~\citep{xiao2010sun}. CIFAR-100~\citep{Krizhevsky09learningmultiple} and ImageNet~\citep{imagenet15russakovsky} are controls throughout and enter no probe set. Section~\ref{sec:exp-holdout} composes over AID~\citep{xia2017aid}, UCMerced~\citep{yang2010bag} and RESISC45, evaluating on each in turn.

\textbf{Sources and probes.} Each source is a rank-$4$ fine-tune of a vision encoder with the text encoder frozen, giving $6 \times 8 = 48$ sources. A task contributes thirty images for each class it defines, drawn uniformly at random from its training split rather than allotted class by class, so no label enters the draw and no evaluation image is a probe. Each descriptor is read on the probes and prompts of the one task it describes, between $300$ unlabeled images and $10$ prompts for a ten-class task and $11{,}910$ images and $397$ prompts for SUN397.

\textbf{Realizations.} Every closed-form result below is a signed sum of stored operators, one per source and task. The adapter is fitted once per cell, with rank $16$, AdamW and $20$ epochs over the cell's probes capped at $540$ steps (Appendix~\ref{app:adapter}). Its cost confines it to subsets fixed at seed $42$ before any cell was fitted, namely every same-model and architecture cell of Tables~\ref{tab:addition-main} and~\ref{tab:negation-main} and, in their last blocks, one draw per target and task pair for addition and per target and task for negation. Strengths are fixed for the whole paper, $\alpha = 0.5$ on the closed form's operator for addition, the full descriptor for the adapter, and $\lambda = 1$ for negation in both.

\textbf{Baselines.} Random-UTD replaces $U$ by a Gaussian matrix of the same shape and Frobenius norm, Single-A and Single-B transfer one of a pair's two descriptors alone, and Oracle has the target supply its own descriptors. Weight-space task arithmetic~\citep{ilharco2022editing} is defined only when source and target share one parameter space, so it appears in the same-model setting as the single-model reference, and weight-space transport, which writes a source's update unchanged into another family of the same architecture, is the null that the weight-transport methods GradFix~\citep{rinaldi2026gradient} and Theseus~\citep{rinaldi2026transporting} are built to beat, and in addition these four weight-space routes apply the mean of a cell's two updates at $0.5$, so each enters at $0.25$ (Appendix~\ref{app:baselines-protocols}). Gradient ascent~\citep{thudi2022unrolling} maximizes the cross-entropy of the removed task.

\subsection{Cross-Model Task Addition}
\label{sec:exp-addition}

Table~\ref{tab:addition-main} reports each regime over all $28$ task pairs, with the cross-model block covering every ordered assignment of two distinct sources to a target distinct from both. Four readings follow.

First, the composed operator does what composition should do. Each single source lifts its own task and leaves the other near zero-shot, whereas the composition lifts both, by $18.11$ and $11.63$ points over zero-shot. It is not free, since paired per cell against the single source of the same task it costs $1.47 \pm 2.78$ points on task~A and $1.25 \pm 3.14$ on task~B at $\alpha = 0.5$, which is the interference two operators estimated apart cause each other once they act on one encoder. That cost grows with strength, and the same comparison at $\alpha = 1.0$ gives $3.35$ and $3.32$.

\begin{wraptable}[26]{r}{0.525\textwidth}
\vspace{-\baselineskip}
\centering
\scriptsize
\caption{Addition in percent, each block read against its own zero-shot row. Task~A and Task~B are the two composed tasks, the closed form acts at $\alpha = 0.5$, the adapter at full strength and the weight-space routes at $0.5$ on the mean of the two updates. Random-UTD, Oracle and the last two rows cover $168$ cells.}
\label{tab:addition-main}
\setlength{\tabcolsep}{2.8pt}
\begin{tabular}{lcccc}
\toprule
Method & Task A $\uparrow$ & Task B $\uparrow$ & CIFAR-100 & ImageNet \\
\midrule
\multicolumn{5}{l}{\textit{Source and target are the same model} ($112$ cells)} \\
Zero-shot                          & $58.70$ & $57.96$ & $76.95$ & $71.50$ \\
Weight-space arithmetic            & $78.50$ & $79.05$ & $77.85$ & $71.83$ \\
\textbf{Ours, closed form}         & $83.07$ & $76.18$ & $76.23$ & $70.13$ \\
\textbf{Ours, adapter}             & $86.85$ & $85.87$ & $72.40$ & $71.49$ \\
\midrule
\multicolumn{5}{l}{\textit{Different models within one architecture} ($168$ cells)} \\
Zero-shot                          & $56.73$ & $56.01$ & $73.30$ & $70.22$ \\
Weight-space transport             & $56.71$ & $55.99$ & $73.37$ & $70.22$ \\
GradFix                            & $70.48$ & $72.86$ & $72.77$ & $70.39$ \\
Theseus                            & $57.20$ & $56.18$ & $73.52$ & $70.26$ \\
\textbf{Ours, closed form}         & $77.64$ & $71.75$ & $71.46$ & $66.72$ \\
\textbf{Ours, adapter}             & $78.71$ & $76.03$ & $72.17$ & $69.89$ \\
\midrule
\multicolumn{5}{l}{\textit{Different models, any pair of families} ($3{,}360$ cells)} \\
Zero-shot                          & $60.98$ & $60.98$ & $75.82$ & $72.33$ \\
Random-UTD                         & $32.37$ & $32.58$ & $32.97$ & $23.37$ \\
Single-A                           & $80.56$ & $59.86$ & $75.04$ & $71.01$ \\
Single-B                           & $59.45$ & $73.85$ & $74.88$ & $70.51$ \\
\textbf{Ours, closed form}         & $79.08$ & $72.60$ & $74.01$ & $69.09$ \\
Oracle                             & $83.64$ & $76.74$ & $75.05$ & $71.35$ \\
\multicolumn{5}{l}{\quad On the adapter's $168$ cells} \\
\textbf{\quad Ours, closed form}   & $79.05$ & $72.65$ & $73.89$ & $68.93$ \\
\textbf{\quad Ours, adapter}       & $79.50$ & $76.70$ & $74.00$ & $72.15$ \\
\bottomrule
\end{tabular}

\end{wraptable}
Second, $\alpha$ trades acquisition against preservation, and past $0.5$ the trade is no longer favorable, since raising it to $1.0$ moves the two tasks by only $+0.3$ and $-1.2$ points while costing $6.4$ of CIFAR-100 and $9.5$ of \mbox{ImageNet}. We fix $\alpha = 0.5$ for every composition and sweep it in Appendix~\ref{app:addition-full}, so nothing is tuned per cell.

Third, Oracle locates the error. It sums per-task operators as the cross-model rows do, so the interference above is present in it, yet with the target's own descriptors it reaches the highest accuracy on both tasks and holds CIFAR-100 within $0.77$ points of zero-shot, so neither the operator form nor per-task estimation is the binding constraint. What separates the cross-model rows from it is the residual of Hypothesis~\ref{hyp:consistency}, and they retain $80$ and $74$ percent of its gain over zero-shot on the two tasks, although a source's operator agrees with the one the target's own descriptor yields only at cosine $0.39$ (Appendix~\ref{app:transport-error}), so transfer tolerates a large residual. Random-UTD, a norm-matched but structureless descriptor, loses on both tasks and collapses the controls, so the gain comes from descriptor structure.

Fourth, the adapter, fitted per cell to the same descriptors, is the more accurate edit. On its $168$ cells in the last block it reaches $79.50$ and $76.70$ on the two tasks against $79.05$ and $72.65$ for the closed form and leaves ImageNet $0.2$ points below zero-shot where the closed form gives up $3.4$, while the two spend CIFAR-100 alike. The architecture block repeats the pattern, and the same-model block shows its limit, where the adapter reproduces the target's own fine-tuning shift at full strength and with it part of the fine-tuning's control damage (Appendix~\ref{app:ft-cost}), giving up $4.5$ points of CIFAR-100 against $0.7$. The closed form is cheaper in return, since each of its cells is a sum of two stored matrices whereas each adapter cell is a separate optimization of about ten minutes on one GPU.

\textbf{The cells are not independent.} The $3{,}360$ cells recombine $48$ source fine-tunes on six targets, so we also read them by unit. Each family used as a source gives a mean gain of $13.82$ to $15.74$ points on the task it supplies, three fine-tuning seeds move the composed row of the seed study by a standard deviation of $0.2$ points (Appendix~\ref{app:seeds}), and all $180$ combinations of a target, a source and a task on the six tasks other than Cars and SUN397 gain on average. The exceptions are confined to those two, on which the target's own descriptor gains only $1.90$ and $2.00$ points and a transferred operator gains nothing (Appendix~\ref{app:add-per-task}), and accumulating up to eight tasks on one target is paid for in preservation (Appendix~\ref{app:ksweep}).

\subsection{Cross-Model Task Negation}
\label{sec:exp-negation}

\begin{wrapfigure}[9]{r}{0.47\textwidth}
\vspace{-1.5\baselineskip}
\centering
\includegraphics[width=\linewidth]{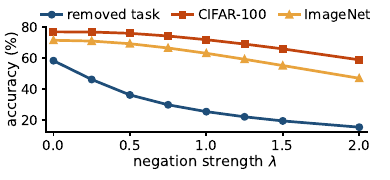}
\vspace{-14pt}
\caption{Accuracy against negation strength.}
\label{fig:negation}
\end{wrapfigure}
Negation transfers $U_{\text{neg}} = -\lambda \operatorname{UTD}^{(k)}_\tau$, and by Proposition~\ref{prop:linearity} the whole $\lambda$ sweep is one bank entry scaled. The two scalars enter only through their product, so we set $\alpha = 1$ and report $\lambda$ alone. Figure~\ref{fig:negation} sweeps $\lambda$ densely over the $160$ cross-model cells of the four targets that absorb the operator, a narrower population than Table~\ref{tab:negation-main}'s, with $\lambda = 0$ the unedited target. An edit is selective when the removed task falls much faster than the controls.

\begin{wraptable}{r}{0.53\textwidth}
\centering
\scriptsize
\caption{Negation at $\lambda = 1$ in percent, grouped as in Table~\ref{tab:addition-main}. Theseus and GradFix read ImageNet on one image per class, where zero-shot is $70.70$.}
\label{tab:negation-main}
\setlength{\tabcolsep}{3.5pt}
\begin{tabular}{lccc}
\toprule
Method & Removed task $\downarrow$ & CIFAR-100 & ImageNet \\
\midrule
\multicolumn{4}{l}{\textit{Source and target are the same model} ($32$ cells)} \\
Zero-shot                          & $58.33$ & $76.95$ & $71.50$ \\
Weight-space negation              & $28.23$ & $63.51$ & $69.95$ \\
\textbf{Ours, closed form}         & $22.07$ & $74.28$ & $67.30$ \\
\textbf{Ours, adapter}             & $26.43$ & $75.64$ & $70.21$ \\
\midrule
\multicolumn{4}{l}{\textit{Different models within one architecture} ($48$ cells)} \\
Zero-shot                          & $56.37$ & $73.30$ & $70.22$ \\
Theseus                            & $53.87$ & $71.81$ & $70.51$ \\
GradFix                            & $21.07$ & $56.96$ & $69.92$ \\
\textbf{Ours, closed form}         & $22.75$ & $68.53$ & $62.77$ \\
\textbf{Ours, adapter}             & $27.84$ & $72.71$ & $69.51$ \\
\midrule
\multicolumn{4}{l}{\textit{Different models, any pair of families} ($240$ cells)} \\
Zero-shot                          & $60.98$ & $75.82$ & $72.33$ \\
\textbf{Ours, closed form}         & $25.76$ & $71.27$ & $65.67$ \\
\multicolumn{4}{l}{\quad On the adapter's $48$ cells} \\
\textbf{\quad Ours, closed form}   & $26.44$ & $70.75$ & $64.61$ \\
\textbf{\quad Ours, adapter}       & $34.93$ & $74.06$ & $71.59$ \\
\multicolumn{4}{l}{\quad Each target as its own source, same tasks} \\
\quad Closed form                  & $26.04$ & $73.53$ & $69.10$ \\
\quad Adapter                      & $31.78$ & $74.31$ & $71.21$ \\
\bottomrule
\end{tabular}

\end{wraptable}
Along the sweep both controls stay nearly flat while the removed task drops steeply, and past $\lambda = 1$ the three curves run parallel, so the edit stops being selective and simply degrades the model. Appendix~\ref{app:negation-full} continues the sweep to near-complete erasure, where the last eight points of forgetting cost twenty-five of the control. Over the $240$ cross-model cells of Table~\ref{tab:negation-main}, $\lambda = 1$ removes $35.2$ points of the targeted task for $4.5$ of CIFAR-100 and $6.7$ of ImageNet, and $\lambda = 0.5$ trades $23.0$ points for $1.0$ and $1.7$. The last two rows of Table~\ref{tab:negation-main}, with each target as its own source on the same tasks, are the cross-model rows' reference, and a same-model source erases to the same depth for about half the control cost, the expected sign since transport error is what the cross-model setting adds. The adapter, fitted to the same descriptor on the removed task's probes, erases less and keeps more, removing $26.0$ points on its $48$ cross-model cells for $1.8$ of CIFAR-100 and $0.7$ of ImageNet where the closed form removes $34.5$ for $5.1$ and $7.7$, so the two realizations of one descriptor differ in how far the edit reaches beyond the probes rather than in whether it is selective.

\subsection{Reaching a Dataset That Enters No Fit}
\label{sec:exp-holdout}

\begin{wraptable}[18]{r}{0.53\textwidth}
\vspace{\dimexpr-\baselineskip-2.4pt\relax}
\centering
\scriptsize
\caption{Composition evaluated on a dataset that enters no fit, in percent, over three targets and by the models $k_1$ and $k_2$ supplying the two summed terms, with the closed form at $\alpha = 0.5$ and the adapter at full strength. Both terms come from one other family when $k_1 = k_2$, from the two other families when $k_1 \neq k_2$ and from the target itself in the within-model rows.}
\label{tab:holdout}
\setlength{\tabcolsep}{2.5pt}
\begin{tabular}{llcccc}
\toprule
Assignment & Route & RESISC45 & UCMerced & AID & Mean \\
\midrule
None & Zero-shot & $68.35$ & $74.29$ & $68.58$ & $70.41$ \\
\midrule
$k_1 = k_2$ & Closed form & $71.23$ & $84.09$ & $72.93$ & $76.08$ \\
 & Adapter & $71.46$ & $85.60$ & $72.83$ & $76.63$ \\
$k_1 \neq k_2$ & Closed form & $71.21$ & $84.56$ & $73.63$ & $76.47$ \\
 & Adapter & $71.42$ & $85.40$ & $73.12$ & $76.64$ \\
$k_1$ alone & Closed form & $70.69$ & $83.33$ & $72.51$ & $75.51$ \\
 & Adapter & $70.63$ & $82.18$ & $72.80$ & $75.20$ \\
Within model & Closed form & $72.86$ & $86.83$ & $74.48$ & $78.06$ \\
 & Adapter & $73.35$ & $90.79$ & $76.35$ & $80.16$ \\
\bottomrule
\end{tabular}

\end{wraptable}
An operator is a map on embeddings rather than a rule about the prompts it was fitted on, so it can act on data that entered no fit. We compose the descriptors of two remote-sensing datasets and evaluate on the third, each in turn, so neither the evaluated dataset's images nor its class names enter either fit. Apart from the within-model reference rows, the last two of Table~\ref{tab:holdout}, every term comes from a family other than the target.

With both terms from one other family, the closed form gains $5.67$ points over zero-shot and its first term alone $5.10$, so a single related operator carries most of the effect and the second adds $0.57 \pm 1.04$ points, leading on $13$ of the $18$ cells.

Two distinct families reach $76.47$, and the adapter stays within $0.6$ points of the closed form across families and leads it by $2.1$ within the target's own. An edit that merely re-scored the prompt columns it was fitted on would leave the evaluated dataset at zero-shot, so the edit changes the representation, while CIFAR-100 stays within three points (Appendix~\ref{app:holdout-full}).

\subsection{Comparison With Alternative Routes}
\label{sec:exp-comparison}

Task knowledge can also be moved in parameter space, which needs matching coordinates and is reported same-model, or transported under a structural correspondence, which needs corresponding submodules and, for GradFix, labels (Appendix~\ref{app:baselines}). On the $168$ architecture cells of Table~\ref{tab:addition-main}, where every route is defined, GradFix lifts the two tasks by $13.7$ and $16.9$ points and Theseus, applied with its released code outside the cross-scale setting it was designed for, by less than one, against $20.9$ and $15.7$ for the closed form at $\alpha = 0.5$ and $22.0$ and $20.0$ for the adapter, without a shared parameter space or a label. At $\lambda = 1$ in the single-model setting the closed form removes $36.3$ points of the targeted task for $2.7$ of CIFAR-100 and the adapter $31.9$ for $1.3$, while weight-space negation removes $30.1$ for $13.4$. On the architecture negation cells GradFix erases as deeply as the closed form and holds ImageNet, but spends $16.3$ points of CIFAR-100 against $4.8$ and $0.6$ for the two realizations, and gradient ascent erases furthest and takes both controls with it. Distilling the same sources into the target on the same probes with the adapter's own recipe, or tuning it on the probes' labels, acquires more on the composed tasks but gives up $5.3$ to $6.8$ points of CIFAR-100 in the case of distillation and $2.7$ in the case of supervised tuning, against $0.5$ for the adapter, and distillation's negation leaves no usable model (Appendix~\ref{app:baselines-kd}).

\subsection{Generality to Audio--Text Models}
\label{sec:exp-audio}

\begin{wraptable}[14]{r}{0.47\textwidth}
\vspace{\dimexpr-\baselineskip-2.25pt\relax}
\centering
\scriptsize
\caption{Audio--text negation at $\lambda = 1$ for both realizations, in percent, where each cell removes one of ESC-50 and GTZAN and retains the other. The adapter's rows are read against its own runs' zero-shot, $71.02$, $71.02$ and $75.85$ (Appendix~\ref{app:audio}).}
\label{tab:audio-neg}
\vspace{-4pt}
\setlength{\tabcolsep}{3pt}
\begin{tabular}{lccc}
\toprule
Method & Removed $\downarrow$ & Retained & UrbanSound8K \\
\midrule
Zero-shot                          & $71.05$ & $71.05$ & $76.10$ \\
\midrule
\multicolumn{4}{l}{\textit{Source and target are the same model} ($8$ cells)} \\
\textbf{Ours, closed form}         & $55.98$ & $70.73$ & $74.25$ \\
\textbf{Ours, adapter}             & $57.69$ & $70.76$ & $74.10$ \\
\midrule
\multicolumn{4}{l}{\textit{Source and target are different models} ($24$ cells)} \\
\textbf{Ours, closed form}         & $56.68$ & $70.45$ & $74.08$ \\
\textbf{Ours, adapter}             & $59.90$ & $71.00$ & $74.88$ \\
\bottomrule
\end{tabular}

\vspace{4pt}
\end{wraptable}
Four CLAP checkpoints whose audio towers differ in width ($768$ or $1024$) and depth ($12$ or $18$ blocks), so that weight-space task vectors are undefined across the two classes, serve as sources and targets, with UrbanSound8K in no probe set. Addition behaves as in vision, the closed form lifting ESC-50 by $1.66$ points and GTZAN by $2.33$ at $\alpha = 0.5$ while UrbanSound8K holds within half a point (Appendix~\ref{app:audio}). In negation at $\lambda = 1$ (Table~\ref{tab:audio-neg}) the closed form removes $14.4$ points for $2.0$ of UrbanSound8K, about $7.1$ points forgotten per point of control against $7.7$ in vision, while the adapter removes $11.1$ for $1.0$, the same trade of depth for selectivity as in vision.

\section{Conclusion and Limitations}
\label{sec:conclusion}

Weight-space task arithmetic was confined to one model because it operated on coordinates. Platonic task arithmetic operates on the object those coordinates shadow, observing it as a descriptor whose shape no architecture sets and transferring it by editing the target until it reproduces the descriptor. Solved in closed form, the edit is one operator folded into the target's last layer and operators compose by addition, while an adapter fitted to the same objective is more accurate and keeps ImageNet closer to zero-shot at the price of one optimization per edit. The residual of Hypothesis~\ref{hyp:consistency} separates the cross-model rows from Oracle, and shrinking it is the open question.

\paragraph{Limitations.}
\label{sec:limitations}
Hypothesis~\ref{hyp:consistency} is supported rather than derived, and its residual is comparable in norm to the shared component. Interference between operators from different tasks has an exact form but no smallness guarantee (Appendix~\ref{app:interference}), transfer gains at most two points on Cars and SUN397, and prior work's analogy does not carry over (Appendix~\ref{app:analogy}). Both realizations need semantic prompts and unlabeled probes of each task, so they are label-free rather than data-free and exclude single-modality models, and the closed form needs an affine last layer writing into the embedding. Lowering the cost of moving a capability between models lowers it for undesirable capabilities too, and negation is the corresponding tool for removing one.

\begin{ack}
We thank the anonymous reviewers for their constructive comments. This work was supported by TelePIX.
\end{ack}

\bibliographystyle{plainnat}
\bibliography{ptv}

\clearpage

\appendix
\section{Notation}
\label{app:notation}
Table~\ref{tab:notation} collects the symbols of the main text and of this appendix.

\begingroup
\setlength{\LTpre}{0pt}
\setlength{\LTpost}{10pt}
\setlength{\abovecaptionskip}{3pt}
\small

\begin{longtable}{@{}p{0.30\linewidth} p{0.64\linewidth}@{}}
\caption{Notation for models, tasks, descriptors and the two realizations.}
\label{tab:notation} \\
\specialrule{1pt}{2pt}{2pt}
\textbf{Symbol} & \textbf{Description} \\
\specialrule{1pt}{2pt}{2pt}
\endfirsthead
\specialrule{1pt}{2pt}{2pt}
\textbf{Symbol} & \textbf{Description} \\
\specialrule{1pt}{2pt}{2pt}
\endhead
\specialrule{1pt}{2pt}{2pt}
\endlastfoot
\multicolumn{2}{@{}l}{\textbf{Models, encoders and parameters}} \\
$\mathcal{M}_k$ & The $k$-th vision--language model, $k = 1, \dots, K$. \\
$\mathcal{M}_{\mathrm{tgt}}$ & Target model receiving the edit. \\
$\ell$, $\sigma$ & Task-term index and the source assignment $\{1,\dots,J\} \to \{1,\dots,K\}$. \\
$f^{(k)}, g^{(k)}$ & Vision and text encoders of $\mathcal{M}_k$, with $g^{(k)}$ frozen. \\
$f^{(k)}_{\mathrm{pre}}, f^{(k)}_{\mathrm{ft}}$ & Pre-trained and fine-tuned vision encoders of $\mathcal{M}_k$. \\
$d^{(k)}$, $P^{(k)}$ & Embedding dimension and parameter count of $\mathcal{M}_k$. \\
$\delta^{(k)}_{\tau}$ & Weight-space task vector $\theta^{(k)}_{\mathrm{ft}} - \theta^{(k)}_{\mathrm{pre}}$. \\
$W_p$, $h(x)$ & Projection matrix and pre-projection representation, $f_{\mathrm{pre}}(x) = W_p h(x)$. \\
\midrule
\multicolumn{2}{@{}l}{\textbf{Tasks and probes}} \\
$\tau$ & A task. \\
$J$ & Number of task terms combined by an operation. \\
$\{x_i\}_{i=1}^{N}, \{p_j\}_{j=1}^{M}$ & Probe images and prompts of the task a descriptor describes, shared by every model. \\
\midrule
\multicolumn{2}{@{}l}{\textbf{Descriptors}} \\
$\mathrm{UTD}^{(k)}_{\tau} \in \mathbb{R}^{N \times M}$ & Universal Task Descriptor of $\mathcal{M}_k$ on task $\tau$. \\
$U_{\mathrm{add}}, U_{\mathrm{neg}}$ & Composed descriptors on one probe set, which the bank never forms. \\
$U_{\mathrm{edit}}$ & Descriptor an edit actually induces, Equation~\eqref{eq:gauge}. \\
$\operatorname{UTD}_{\tau}[\mathcal{E}]$ & Descriptor an edit $\mathcal{E}$ of the target induces on task $\tau$'s grid. \\
$\mathcal{L}$, $N_\tau$, $M_\tau$ & Matching objective of Equation~\eqref{eq:match}; probe and prompt counts of task $\tau$. \\
$F_{\mathrm{pre}} \in \mathbb{R}^{N \times d}$ & Target's unit-norm probe embeddings. \\
$G \in \mathbb{R}^{M \times d}$ & Target's prompt landmarks, one unit-norm row per prompt. \\
$\Delta F \in \mathbb{R}^{N \times d}$ & Shift field, the change an edit makes to $F_{\mathrm{pre}}$, so $\mathrm{UTD} = \Delta F\,G^\top$. \\
\midrule
\multicolumn{2}{@{}l}{\textbf{Operator}} \\
$A \in \mathbb{R}^{d \times d}$ & Embedding operator, applied as $f \mapsto (I + \alpha A) f$. \\
$A^{*}(U)$ & Solution of Equation~\eqref{eq:lsq} for $U$ on its own grid, linear in $U$. \\
$\alpha \geq 0$ & Deployment strength, applied to the solved operator. \\
$\lambda \geq 0$ & Negation strength, entering only through $\alpha\lambda$, so negation is reported at $\alpha = 1$. \\
$\mu_a, \mu_b$ & Ridge coefficients of Equation~\eqref{eq:lsq}, for the probe and the prompt factor. \\
$X_\tau$, $c_\ell$ & Operator solved on task $\tau$'s grid, $A_\tau = X_\tau^\top$, and signed weight of term $\ell$. \\
$w_\tau(f)$, $v_\tau(g)$ & Probe and prompt weights on grid $\tau$, Equation~\eqref{eq:weights}. \\
$\mathcal{F}_{\tau'\tau}$, $\mathcal{G}_{\tau\tau'}$ & Cross-grid matrices reading task $\tau$ on grid $\tau'$, Equation~\eqref{eq:crossgrid}. \\
$m_\ell(x)$, $\varepsilon_\ell(x)$ & Single-term margin and the other terms' perturbation, Equations~\eqref{eq:margin} and~\eqref{eq:margin-perturbation}. \\
\midrule
\multicolumn{2}{@{}l}{\textbf{Adapter realization}} \\
$\Delta\theta$ & Parameters of the low-rank adapter on the target's vision encoder. \\
$r_{\mathrm{LoRA}}$, $\alpha_{\mathrm{LoRA}}$ & Rank and scaling factor of the adapter realization. \\
\midrule
\multicolumn{2}{@{}l}{\textbf{Shared affinity structure}} \\
$(\mathcal{V}, \langle\cdot, \cdot\rangle)$, $\mathcal{X}$ & Shared inner-product space; space of probe inputs. \\
$\psi(p) \in \mathcal{V}$ & Shared semantic vector of prompt $p$. \\
$\Pi_{\tau}$ & Platonic task vector for task $\tau$. \\
$\mathcal{A}$ & Measurement map, $[\mathcal{A}(\Pi)]_{ij} = \langle \Pi(x_i), \psi(p_j) \rangle$. \\
$E^{(k)}_{\tau}$, $\eta$ & Model-specific residual and its uniform relative bound. \\
$\hat{\eta}^{(k)}_{\tau}$ & Plug-in dispersion diagnostic of Equation~\eqref{eq:eta-hat}. \\
\midrule
\multicolumn{2}{@{}l}{\textbf{Appendix-only quantities}} \\
$d_{\star}$ & Dimension of the shared space. \\
$\phi(x)$, $T^{(k)}$ & Shared unit-norm image representation and the factorization map. \\
$\rho_{\tau}$, $\kappa^{(k)}$ & Maximum platonic displacement and the isometry defect of $T^{(k)}$. \\
$r_{\mathrm{src}}$, $\alpha_{\mathrm{src}}$ & Rank and scaling factor of every source fine-tune. \\
\end{longtable}
\endgroup

\section{Related Work}
\label{sec:related}

\paragraph{Task arithmetic and model merging.}
Since task vectors were introduced as weight differences between fine-tuned and pre-trained models~\citep{ilharco2022editing}, subsequent work has largely focused on mitigating interference during composition. A common approach is sparsification of the update, as in DARE~\citep{yu2024language}, which drops most entries of a task vector at random and rescales the rest. Another line controls the merging direction, either by learning layer-wise coefficients at test time~\citep{yang2023adamerging} or by restricting merging to directions orthogonal to task-specific gradients~\citep{sun2025task}. Complementary approaches exploit structural or geometric properties of task vectors, including low-rank decompositions via SVD~\citep{gargiulo2025task} and scaling strategies based on orthogonality and local linearity~\citep{zhou2024metagpt}. The framework has also been combined with parameter-efficient tuning, through task-arithmetic-style composition of low-rank adaptation (LoRA) weights~\citep{chitale2023task,huang2023lorahub} and through extracting orthogonal bases from task vectors for few-shot adaptation~\citep{park2026basis}. Theoretical work seeks to explain why task arithmetic works, through tangent-space approximations~\citep{ortiz2023task} or first-order expansions under cross-task linearity~\citep{zhou2024emergence}. Every method in this line presupposes a shared weight or feature space tied to a single architecture, and the object it composes is a parameter difference. We drop the first requirement and replace the second object with a functional descriptor.

\paragraph{Cross-model transfer.}
Transferring knowledge across models traditionally relies on data and additional optimization. Relative representations~\citep{moschella2022relative} align models by expressing activations as similarities to anchor points, enabling zero-shot stitching, and model stitching more broadly joins layers of independently trained networks through a learned connector~\citep{lenc2015understanding,bansal2021revisiting}. We share the idea of measuring similarities over a probe set but pursue a different goal, namely capturing fine-tuning-induced changes rather than aligning pre-trained representations, and Appendix~\ref{app:ablations} shows that using the text encoder to align two models, rather than to measure a change, transfers worse. At the parameter level, LoRA-X~\citep{farhadzadeh2025lora} and Cross-LoRA~\citep{xia2025cross} align subspaces via SVD but rely on correspondence between singular vectors and do not generalize across families. A further line transports a weight update directly, either by retaining those coordinates of a same-shaped source update that agree with target gradients computed on a few labeled examples, as GradFix does~\citep{rinaldi2026gradient}, or by constructing a closed-form map between corresponding submodules that additionally supports a mismatch in width, as Theseus does~\citep{rinaldi2026transporting}, whose natural setting is therefore models of one family at different scales. Both relax the single-model restriction while retaining a structural requirement, and Section~\ref{sec:exp-comparison} compares against both where they are defined. Full-model approaches such as AdaMMS~\citep{du2025adamms} and Transport and Merge~\citep{cui2026transport} align heterogeneous architectures by module mapping or optimal transport, but merge models trained on one task.

\paragraph{Knowledge distillation and relational transfer.}
Distillation transfers behavior from a teacher to a student by matching outputs~\citep{hinton2015distilling}, and relational variants match relations among inputs rather than individual predictions, such as the distances and angles among embeddings~\citep{park2019relational} or the mutual information between the two models' representations of one input~\citep{tian2019contrastive}, which is what lets them cross a difference in architecture or dimension. A descriptor is superficially a relational signal of this kind, but the object differs, in that distillation matches an absolute teacher output, whereas a UTD is a signed difference taken relative to each model's own pre-trained behavior. That is what makes composition over several teachers well-defined under addition and what makes negation preserve rather than destroy unrelated capability, and it is why our transfer step is a regression onto a difference rather than an imitation of an output. The adapter route is, mechanically, a relational distillation of this kind, since it fits the target's affinities by gradient descent, but what it fits is the change of the target's affinities relative to its own unedited ones, set equal to the source's change. The teacher also enters differently, since distillation needs the teacher model and its forward passes throughout the student's training, whereas here a teacher is summarized once by one matrix on the probe set, from which the closed form solves the edit with the target's own factors and the adapter fits it on the target alone. Appendix~\ref{app:baselines-kd} measures both an affinity-matching and a posterior-matching distillation of the fine-tuned sources on the paper's own cells.

\paragraph{Closed-form model editing.}
Editing a network by solving a least-squares problem rather than by gradient descent has precedent in knowledge editing for language models and in linear-probe transfer, and closed-form merging has been used to align activations between architectures~\citep{rinaldi2026transporting}. Our solve differs in what it fits, since it fits neither a map between two models' parameters nor one between their activations, but rather the operator on one model's embedding whose induced affinity shift matches a descriptor that may have been measured on entirely different models. The consequence we exploit, that the solution map's linearity lets a bank of operators be composed by addition alone, has no analogue in the methods above.

\paragraph{Platonic Representation Hypothesis.}
Neural networks trained on large-scale data converge to similar representations across architectures~\citep{lenc2015understanding,bansal2021revisiting}. The Platonic Representation Hypothesis~\citep{huh2024platonic} formalizes this, positing convergence toward a shared representation driven by data, capacity and optimization biases, with supporting empirical and theoretical evidence~\citep{maniparambil2024vision,tjandrasuwita2025understanding,ziyin2025proof}, often up to simple near-isometric transformations~\citep{maiorca2023latent}. Unlike this line, which studies convergence of pre-trained representations, we ask whether similar convergence extends to fine-tuning-induced changes, and quantify the residual at which it holds.

\section{The Closed-Form Solver}
\label{app:solver}

\subsection{Derivation}

Fix a target model and drop its index. Let $F \in \mathbb{R}^{N \times d}$ be the probe embeddings, $G \in \mathbb{R}^{M \times d}$ the prompt landmarks and $U \in \mathbb{R}^{N \times M}$ the descriptor to transfer. Writing $B = G^\top \in \mathbb{R}^{d \times M}$, the objective of Equation~\eqref{eq:lsq} is
\begin{equation}
    \mathcal{J}(X) \;=\; \| F X B - U \|_F^2 \;+\; \mu_b \| F X \|_F^2 \;+\; \mu_a \| X B \|_F^2 \;+\; \mu_a \mu_b \|X\|_F^2 .
    \label{eq:objective-full}
\end{equation}
This is an ordinary linear least-squares problem in $\operatorname{vec}(X)$ with Kronecker structure, $\operatorname{vec}(FXB) = (B^\top \otimes F)\operatorname{vec}(X)$, so it never has to be formed explicitly. Take thin singular value decompositions $F = L_a \Sigma_a R_a^\top$ and $B = L_b \Sigma_b R_b^\top$, where $L$ and $R$ carry the left and right singular vectors and the diagonal entries of $\Sigma_a$ and $\Sigma_b$ are the singular values $s_{a,i}$ and $s_{b,j}$. Substituting and using orthogonal invariance of the Frobenius norm, the objective separates entry-wise in the transformed variable $Z = R_a^\top X L_b$ and the transformed descriptor $\widetilde{U} = L_a^\top U R_b$,
\begin{equation}
    \mathcal{J} \;=\; \bigl\| \Sigma_a Z \Sigma_b - \widetilde{U} \bigr\|_F^2 \;+\; \mu_b \| \Sigma_a Z \|_F^2 \;+\; \mu_a \| Z \Sigma_b \|_F^2 \;+\; \mu_a \mu_b \|Z\|_F^2 ,
\end{equation}
so that entry $(i,j)$ contributes $(s_{a,i} s_{b,j} Z_{ij} - \widetilde{U}_{ij})^2 + (\mu_b s_{a,i}^2 + \mu_a s_{b,j}^2 + \mu_a \mu_b) Z_{ij}^2$ and is minimized on its own. Its stationarity condition reads $(s_{a,i}^2 s_{b,j}^2 + \mu_b s_{a,i}^2 + \mu_a s_{b,j}^2 + \mu_a \mu_b) Z_{ij} = s_{a,i} s_{b,j} \widetilde{U}_{ij}$, and the coefficient factorizes as $(s_{a,i}^2 + \mu_a)(s_{b,j}^2 + \mu_b)$. Collecting terms back, the solution is
\begin{equation}
    X^{*} \;=\; R_a\, \Sigma_a^{+}\, \bigl( L_a^\top\, U\, R_b \bigr)\, \Sigma_b^{+}\, L_b^\top ,
    \qquad
    [\Sigma_a^{+}]_{ii} = \frac{s_{a,i}}{s_{a,i}^2 + \mu_a},
    \quad
    [\Sigma_b^{+}]_{jj} = \frac{s_{b,j}}{s_{b,j}^2 + \mu_b},
    \label{eq:closed-form}
\end{equation}
and $A = (X^{*})^\top$. The solution is thus the product of two spectral filters, one per factor, and the two ridge terms are what makes it factorize. A single penalty $\mu \|X\|_F^2$ would instead give the filter $s_{a,i} s_{b,j} / (s_{a,i}^2 s_{b,j}^2 + \mu)$, which treats every prompt direction alike, whereas Equation~\eqref{eq:closed-form} damps the probe factor in prompt direction $j$ by an amount that grows with $s_{b,j}^2$. Since the landmarks are unit vectors with a strongly shared component, their leading singular value is far above the rest, so the probe-side inverse is regularized most where the prompts carry the most energy and hardly at all in the fine directions that separate classes. This is the two-step form of kernel ridge regression studied by \citet{stock2018comparative}, who analyze exactly the difference between one ridge on the Kronecker system and one ridge per factor. We set each coefficient relative to the mean squared singular value of its factor, $\mu_a = 10^{-2}\,\overline{s_a^2}$ and $\mu_b = 10^{-6}\,\overline{s_b^2}$, so that a single setting transfers across targets whose feature scales differ. The asymmetry is deliberate, since $F$ is the ill-conditioned factor, holding far more probe rows than the dimension they span and dominated by a single high-energy direction, whereas $G$ carries one row per prompt and is comparatively well conditioned, needing only enough regularization to be numerically safe. That safety margin matters because a task whose class names are close in meaning leaves $G$ poorly conditioned even at full row rank, and $\mu_b$ then keeps the filter finite on its weakest prompt directions, where the transformed descriptor is small in any case, since a descriptor is $\Delta F\,G^\top$ in the landmarks of the model that measured it and is therefore small in the prompt directions those landmarks barely separate, which class names close in meaning make weak on the target as well.

\paragraph{A pairwise regression view.}
The fitted quantity is the bilinear form $f_{\text{pre}}(x)^\top X\, g(p)$ on image--prompt pairs, which is a linear model on the tensor-product feature $f_{\text{pre}}(x) \otimes g(p)$ and therefore a kernel regression on pairs whose kernel is the product $\langle f_{\text{pre}}(x), f_{\text{pre}}(x') \rangle\, \langle g(p), g(p') \rangle$ of the two encoders' inner products~\citep{alvarez2012kernels}. Under this reading the transfer step predicts affinity changes for pairs the fit never saw, and the three cases of a new image, a new prompt, and both being new are the settings in which pairwise kernel methods are analyzed~\citep{stock2018comparative}. The experiments touch the harder of them. Every evaluation image is new to the fit that produced the operator scoring it, and a control benchmark and the evaluated dataset of Section~\ref{sec:exp-holdout} are scored against prompts that entered no fit either, so those measurements are the case in which image and prompt are both new. The cross-scale targets of Appendix~\ref{app:scale} add a further variation, carrying the same probes and prompts into a different embedding. The regression view also states what the solve asks of the target, namely that the observable affinity change be reconstructed stably from the target's own factors, and the ridge terms are what fix the stability.

\paragraph{Hoisting the target-side factors.}
Both decompositions in Equation~\eqref{eq:closed-form} depend on the target and on the probes and prompts in use, never on $U$. They are therefore computed once for a target and a task and cached, after which every source descriptor of that task is solved by the three matrix products $R_a \Sigma_a^{+} (L_a^\top U R_b) \Sigma_b^{+} L_b^\top$. This is what keeps a solve well under a second and builds a target's whole bank of $48$ operators in $17$ seconds, after which every composed cell is a free sum of two of them.

\subsection{Proof of Proposition~\ref{prop:linearity}}
\label{app:linearity}

Equation~\eqref{eq:closed-form} exhibits $X^{*}$ as a fixed linear map applied to $U$, since everything to the left and to the right of $U$ is independent of $U$. Hence $X^{*}(c_1 U_1 + c_2 U_2) = c_1 X^{*}(U_1) + c_2 X^{*}(U_2)$, and transposition preserves linearity, giving Equation~\eqref{eq:solve-linear}. \qed

The statement is exact in exact arithmetic, and what matters in practice is whether it survives in floating point at the conditioning we operate at. Proposition~\ref{prop:linearity} concerns one solution map, so the check is read on one grid, and the grid it is read on is the one an operator is actually solved on, namely a single task's own probes and prompts. Table~\ref{tab:linearity} therefore takes three source families' descriptors of the same task, forms the three-term expression $A^{*}(U_1 - 1.5\,U_2 + U_3)$, which mixes signs and carries a non-unit coefficient so that neither addition nor negation would exercise the identity on its own, and compares it against the same combination of the three terms' operators over all eight tasks and over the four targets that absorb the operator. The identity holds to a relative difference below $5 \times 10^{-14}$ when the solve is carried in double precision, and below $4 \times 10^{-7}$ at the single precision a stored operator is kept in, which is what a bank entry and every sum of bank entries actually carries.

\begin{table}[H]
\centering
\small
\caption{Numerical check of Proposition~\ref{prop:linearity} on the grids the paper solves on. For each target and each of the eight tasks, three source families give descriptors of that task on that task's own probes and prompts, and the operator of $U_1 - 1.5\,U_2 + U_3$ is compared against the same combination of the three terms' operators. Each row reports the worst of that target's eight tasks, namely the smallest cosine between the two operators over both precisions and the largest relative difference between them. The last two columns separate the algebra from the arithmetic, giving the relative difference when the solve is carried in double precision and when it is carried at the single precision a stored operator is kept in.}
\label{tab:linearity}
\begin{tabular}{lccc}
\toprule
Target & Cosine & Relative difference, double & Relative difference, single \\
\midrule
CLIP-L & $1.0000000000$ & $2.6 \times 10^{-14}$ & $1.6 \times 10^{-7}$ \\
OpenCLIP-L & $1.0000000000$ & $2.6 \times 10^{-14}$ & $1.9 \times 10^{-7}$ \\
MetaCLIP-L & $1.0000000000$ & $3.5 \times 10^{-14}$ & $2.4 \times 10^{-7}$ \\
EVA-CLIP-L & $1.0000000000$ & $4.8 \times 10^{-14}$ & $3.7 \times 10^{-7}$ \\
\bottomrule
\end{tabular}
\end{table}

\paragraph{What the proposition covers, and what the bank does.}
Proposition~\ref{prop:linearity} concerns one solution map, that is, one grid. It licenses scaling a solved operator by any strength, negating it and combining several source families' descriptors of one task, all exactly and without a further solve, and Table~\ref{tab:linearity} checks the last of these in floating point. It does not make the sum of two bank entries the solve of a composed descriptor, since each entry is constrained only on its own task's probes and landmarks and settled by the ridge elsewhere, so the entries are estimated without each other in view and may interfere when summed. Appendix~\ref{app:interference} states exactly how such a sum departs from each of its terms, and Section~\ref{sec:exp-addition} measures the cost.

\subsection{Composition Across Grids}
\label{app:interference}

Proposition~\ref{prop:linearity} concerns one grid. The bank of Section~\ref{sec:method-linearity} sums operators estimated on different tasks' grids, and this subsection states what can and cannot be said about such a sum. Throughout, a task $\tau$ has probes $F_\tau \in \mathbb{R}^{N_\tau \times d}$, landmarks $G_\tau \in \mathbb{R}^{M_\tau \times d}$, a descriptor $U_\tau \in \mathbb{R}^{N_\tau \times M_\tau}$ and the operator $X_\tau = X^{*}(U_\tau)$ that Equation~\eqref{eq:closed-form} returns on that grid, with $A_\tau = X_\tau^\top$ the bank entry, so that the composed operator of Algorithm~\ref{alg:edit} is $X_{\mathrm{c}} = \sum_{\ell=1}^{J} c_\ell X_{\tau_\ell}$ with signed weights $c_\ell$. The ridge coefficients are those of the grid the operator is solved on. Writing out Equation~\eqref{eq:closed-form} without the decompositions,
\begin{equation}
    X_\tau
    \;=\;
    (F_\tau^\top F_\tau + \mu_a I)^{-1} F_\tau^\top\, U_\tau\, G_\tau\, (G_\tau^\top G_\tau + \mu_b I)^{-1},
    \label{eq:closed-form-plain}
\end{equation}
which is the form every statement below reads off. The quantity analyzed throughout is the score displacement $f^\top X g$ that an operator $X$ adds to the unnormalized similarity of an image embedding $f$ and a landmark $g$, which is what Equation~\eqref{eq:lsq} matches to a descriptor on the grid. The normalized affinity the model reports differs from it by the renormalization term of Appendix~\ref{app:ablations}, so a matrix of score displacements is not the descriptor an edit induces, and Corollary~\ref{cor:margin} below is the statement about predictions that survives that difference.

\paragraph{No assumption-free guarantee exists.}
Grid $\tau$ observes an operator only through $X \mapsto F_\tau X G_\tau^\top$. If a matrix $H$ satisfies $F_\tau H G_\tau^\top = 0$ but $F_{\tau'} H G_{\tau'}^\top \neq 0$, then $X$ and $X + tH$ are indistinguishable on grid $\tau$ for every scalar $t$ while their actions on grid $\tau'$ differ by an arbitrary amount, and such an $H$ exists as soon as some probe of $\tau'$ leaves the row space of $F_\tau$ or some prompt of $\tau'$ leaves the row space of $G_\tau$, for instance the outer product of the component of such a probe orthogonal to every probe of $\tau$ with any vector not orthogonal to the prompts of $\tau'$, and symmetrically in the prompt case. The observation on grid $\tau$ therefore determines the action on grid $\tau'$ for every operator exactly when the probes and prompts of $\tau'$ lie in the row spaces of $F_\tau$ and $G_\tau$, and the second of these is a subspace of dimension at most $M_\tau < d$, which the prompts of a different task have no reason to inhabit. The ridge terms select one solution from each equivalence class, which makes the estimate unique, but they create no information about the directions the grid does not observe. Any guarantee about a sum of independently estimated operators is therefore a statement about how the particular solution the ridge selects acts off its grid, which is what the following results are.

\paragraph{How an operator acts off its grid.}
For an embedding $f \in \mathbb{R}^{d}$ and a landmark $g \in \mathbb{R}^{d}$ that need not belong to the grid, define the probe weights and prompt weights of $f$ and $g$ on grid $\tau$ as
\begin{equation}
    w_\tau(f) \;=\; F_\tau (F_\tau^\top F_\tau + \mu_a I)^{-1} f \;\in\; \mathbb{R}^{N_\tau},
    \qquad
    v_\tau(g) \;=\; G_\tau (G_\tau^\top G_\tau + \mu_b I)^{-1} g \;\in\; \mathbb{R}^{M_\tau},
    \label{eq:weights}
\end{equation}
the ridge coefficients that express $f$ through the task's probes and $g$ through its prompts. Equation~\eqref{eq:closed-form-plain} then gives, for every such pair, the identity
\begin{equation}
    f^\top X_\tau\, g \;=\; w_\tau(f)^\top\, U_\tau\, v_\tau(g),
    \qquad
    \bigl| f^\top X_\tau\, g \bigr| \;\le\; \|w_\tau(f)\|_2\, \|U_\tau\|_2\, \|v_\tau(g)\|_2 .
    \label{eq:offgrid}
\end{equation}
The operator displaces the score of an arbitrary image--prompt pair by reading the descriptor through that pair's weights on the grid, so it moves a pair only to the extent that the target's own pre-trained geometry explains the image through the task's probes and the prompt through the task's prompts. Both weight vectors are computed from the target's embeddings of the grid and of the pair, without reference to the descriptor or to the source that produced it, and the identity is exact rather than first order. Stacking the probe weights of another task's probes and the prompt weights of its landmarks gives the two cross-grid matrices, whose subscripts follow the order of their rows and columns,
\begin{align}
    \mathcal{F}_{\tau'\tau} &\;=\; F_{\tau'} (F_\tau^\top F_\tau + \mu_a I)^{-1} F_\tau^\top \;\in\; \mathbb{R}^{N_{\tau'} \times N_\tau},
    \label{eq:crossgrid} \\
    \mathcal{G}_{\tau\tau'} &\;=\; G_\tau (G_\tau^\top G_\tau + \mu_b I)^{-1} G_{\tau'}^\top \;\in\; \mathbb{R}^{M_\tau \times M_{\tau'}},
    \nonumber
\end{align}
so that the score displacement the operator of task $\tau$ induces on the grid of task $\tau'$ is $F_{\tau'} X_\tau G_{\tau'}^\top = \mathcal{F}_{\tau'\tau}\, U_\tau\, \mathcal{G}_{\tau\tau'}$. Descriptors of different shapes have no sum, but the effect of every operator can be read on any one grid, which is the device the next statement uses.

\begin{proposition}[Composition across grids]
\label{prop:interference}
Let $X_{\mathrm{c}} = \sum_{\ell=1}^{J} c_\ell X_{\tau_\ell}$ with each $X_{\tau_\ell}$ solved on its own grid. Then on the grid of the $\ell$-th term
\begin{equation}
    F_{\tau_\ell} X_{\mathrm{c}}\, G_{\tau_\ell}^\top - c_\ell U_{\tau_\ell}
    \;=\;
    c_\ell \bigl( F_{\tau_\ell} X_{\tau_\ell} G_{\tau_\ell}^\top - U_{\tau_\ell} \bigr)
    \;+\;
    \sum_{\ell' \neq \ell} c_{\ell'}\, \mathcal{F}_{\tau_\ell \tau_{\ell'}}\, U_{\tau_{\ell'}}\, \mathcal{G}_{\tau_{\ell'} \tau_\ell},
    \label{eq:interference-identity}
\end{equation}
and consequently
\begin{equation}
    \bigl\| F_{\tau_\ell} X_{\mathrm{c}}\, G_{\tau_\ell}^\top - c_\ell U_{\tau_\ell} \bigr\|_F
    \;\le\;
    |c_\ell|\, \bigl\| F_{\tau_\ell} X_{\tau_\ell} G_{\tau_\ell}^\top - U_{\tau_\ell} \bigr\|_F
    \;+\;
    \sum_{\ell' \neq \ell} |c_{\ell'}|\,
    \|\mathcal{F}_{\tau_\ell \tau_{\ell'}}\|_2\,
    \|\mathcal{G}_{\tau_{\ell'} \tau_\ell}\|_2\,
    \|U_{\tau_{\ell'}}\|_F .
    \label{eq:interference-bound}
\end{equation}
\end{proposition}

\begin{proof}
Expand $F_{\tau_\ell} X_{\mathrm{c}} G_{\tau_\ell}^\top$ term by term, subtract $c_\ell U_{\tau_\ell}$ from the $\ell$-th term and substitute Equation~\eqref{eq:closed-form-plain} into every other term, which gives Equation~\eqref{eq:interference-identity} by the definition in Equation~\eqref{eq:crossgrid}. The bound follows from the triangle inequality and $\|Y U Y'\|_F \le \|Y\|_2 \|U\|_F \|Y'\|_2$ for matrices $Y, Y'$ of conformable size.
\end{proof}

The identity separates what the composed operator does on each task's grid into how well that task's own operator fits its descriptor and how much every other term's descriptor leaks onto the grid, for any number of terms and any signs. The leakage is governed by the two cross-grid matrices, which measure how far one task's probes and prompts are explained by another's under the target's pre-trained embedding alone, and it is directional, since the two orders of a pair give different products in general, so no single cosine between two operators captures it. When the terms share a grid the cross-grid matrices reduce to that grid's two ridge smoothers and the leakage is exactly what the composed descriptor demands there, which recovers Proposition~\ref{prop:linearity}. Every quantity in Equation~\eqref{eq:interference-bound} is available from the descriptors, the bank and the target's cached features without a forward pass, and the exact leakage $\|\mathcal{F}_{\tau_\ell \tau_{\ell'}} U_{\tau_{\ell'}} \mathcal{G}_{\tau_{\ell'} \tau_\ell}\|_F$ is as cheap to compute as the spectral bound and tighter.

\paragraph{From affinities to predictions.}
Equation~\eqref{eq:interference-bound} bounds a displacement of scores, and a displacement of scores is not a change in accuracy, since a leaking term may move every class of a task alike or may even help. The following statement connects the two. Fix the $\ell$-th term and an image $x$ of task $\tau_\ell$ with label $y$, write $g_j$ for the landmark of its $j$-th class and $f = f_{\text{pre}}(x)$. Under an operator $X$ at strength $\alpha$ the model scores class $j$ by $f^\top (I + \alpha X) g_j$ before normalization, so the model that carries the $\ell$-th term alone has the margin
\begin{equation}
    m_\ell(x) \;=\; \min_{j \neq y}\; f^\top \bigl(I + \alpha c_\ell X_{\tau_\ell}\bigr) (g_y - g_j),
    \label{eq:margin}
\end{equation}
and the other terms perturb the score differences by at most
\begin{align}
    \varepsilon_\ell(x)
    \;&=\;
    \alpha \max_{j \neq y}
    \Bigl| \sum_{\ell' \neq \ell} c_{\ell'}\, f^\top X_{\tau_{\ell'}} (g_y - g_j) \Bigr|
    \label{eq:margin-perturbation} \\
    \;&\le\;
    \alpha \max_{j \neq y} \sum_{\ell' \neq \ell} |c_{\ell'}|\,
    \|w_{\tau_{\ell'}}(f)\|_2\, \|U_{\tau_{\ell'}}\|_2\, \|v_{\tau_{\ell'}}(g_y - g_j)\|_2 ,
    \nonumber
\end{align}
where the inequality is Equation~\eqref{eq:offgrid} applied term by term.

\begin{corollary}[Margin preservation under composition]
\label{cor:margin}
If $m_\ell(x) > \varepsilon_\ell(x)$, the composed model with operator $X_{\mathrm{c}}$ predicts the same class for $x$ as the model carrying the $\ell$-th term alone. Consequently, counting a tie as an error, the error of the composed model on task $\tau_\ell$ exceeds that of the single-term model by at most the probability, over the task's data, that $0 < m_\ell(x) \le \varepsilon_\ell(x)$.
\end{corollary}

\begin{proof}
The composed score difference between $y$ and any $j \neq y$ is the single-term difference plus $\alpha \sum_{\ell' \neq \ell} c_{\ell'} f^\top X_{\tau_{\ell'}} (g_y - g_j)$, which is at least $m_\ell(x) - \varepsilon_\ell(x) > 0$. The prediction is read from cosine similarities, and since the landmarks are unit vectors and the edited embedding of $x$ is nonzero, its norm divides every class score of $x$ by the same positive number, so the ordering of the classes, and hence the prediction, is that of the unnormalized scores. The second statement follows because an image the composed model misclassifies either has $m_\ell(x) \le 0$, in which case the single-term model already errs, or has $0 < m_\ell(x) \le \varepsilon_\ell(x)$.
\end{proof}

The corollary makes no first-order approximation and does not neglect renormalization, since a common positive factor leaves an argmax alone. By the second form of Equation~\eqref{eq:margin-perturbation}, what protects a task is localization of the other terms in the product of image and prompt space rather than orthogonality of the operators, since two operators with $\langle X_\tau, X_{\tau'} \rangle_F = 0$ can still move a decision direction. The paired cost of Section~\ref{sec:exp-addition}, which grows with $\alpha$ as $\varepsilon_\ell$ does, estimates the left-hand side of the corollary's second statement pooled over cells, and the measurement below computes $\varepsilon_\ell(x)$ and $m_\ell(x)$ on evaluation images the bank never saw.

\paragraph{The corollary on the paper's cells.}
Table~\ref{tab:margin} evaluates Corollary~\ref{cor:margin} on the $2{,}240$ cross-model addition cells of the four targets that absorb the operator, at $\alpha = 0.5$, on every evaluation image of both tasks of each cell, which gives $4{,}480$ records of a cell and an evaluated task. Nothing is re-encoded, since every score is a product of cached base features, a bank entry and the task's landmarks, and the scores are accumulated in double precision so that a margin exceeding the perturbation by less than single-precision resolution cannot register as a violated guarantee. Three checks tie the measurement to the rest of the paper. The identity of Equation~\eqref{eq:interference-identity} holds to a relative difference of $9.1 \times 10^{-7}$ over the $1{,}120$ combinations of a target, a task, another task and a source, which is single precision, the precision a bank entry is stored in. Single-term and composed accuracies computed from the unnormalized scores equal those computed from normalized embeddings exactly, and zero-shot accuracies agree to within one image per split, which is the invariance of the argmax the corollary's proof uses. And the paired cost the records carry, averaged over cells, is $1.83 \pm 3.23$ and $1.26 \pm 3.45$ points on the two tasks, the four-target values behind the six-target $1.47$ and $1.25$ that Section~\ref{sec:exp-addition} reports.

Five things follow from the table. No guaranteed prediction changed in any of the $4{,}480$ records, which together hold fifty million image-level decisions, twenty-four million of them guaranteed. The guarantee covers $47.8$ percent of the evaluation images, the sum actually moves the prediction of $9.9$ percent, of which $3.7$ points were correct before, and the bound $\Pr[0 < m_\ell(x) \le \varepsilon_\ell(x)]$ stands at $22.7$ percent against an observed loss of $1.45$ points, so it holds everywhere and is loose by an order of magnitude, because most images whose margin is at risk survive. The bound predicts the loss per cell and task, one record each, at rank correlation $+0.66$, between $+0.30$ and $+0.85$ within a target, and the mean perturbation $\varepsilon_\ell$ alone gives $+0.62$, whereas the cosine between the two operators gives $+0.07$, its absolute value $+0.05$, and the leakage of Proposition~\ref{prop:interference} read at the level of scores gives $+0.41$, or $+0.40$ relative to the norm of the task's own descriptor, so interference becomes predictable only when it is read at the decision margin, and the orthogonality of two operators is not the relevant quantity. The spectral bound of Equation~\eqref{eq:interference-bound} exceeds the exact leakage by a median factor of $358$, with an interquartile range of $99$ to $1{,}228$, so the exact leakage is the quantity to compute and the bound is a statement of structure rather than of size. And the leakage is directional as claimed, since the leakage a task's grid receives from another task, relative to the norm of its own descriptor, differs from the reverse by a median factor of $3.6$ and by up to $21.9$, with SUN397 receiving $1.9$ times its own descriptor norm on average and GTSRB $0.26$, while per evaluated task EuroSAT is the best protected, with $74$ percent of its images guaranteed and a loss of $1.1$ points, and SVHN and SUN397 the least, with $32$ and $36$ percent guaranteed, SVHN because its single-term margins are smallest and SUN397 because its perturbation typically exceeds its margin, at a median ratio of $1.04$.

\begin{table}[t]
\centering\small
\caption{Corollary~\ref{cor:margin} on the cross-model addition cells of the four targets that absorb the operator, at $\alpha = 0.5$, in percent of evaluation images, pooled over both tasks of every cell and weighted by split size. Guaranteed is the fraction of images with $m_\ell(x) > \varepsilon_\ell(x)$, whose prediction the corollary fixes, and changed is the fraction whose prediction the sum actually moved. Bound is the corollary's $\Pr[0 < m_\ell(x) \le \varepsilon_\ell(x)]$ and observed is the accuracy the single-term model loses to the sum, which the bound dominates in every one of the $4{,}480$ records. Single is the single-term model's accuracy. Probes come from training splits, so no evaluation image is a probe.}
\label{tab:margin}
\setlength{\tabcolsep}{4pt}
\begin{tabular}{@{}lrrrrr@{}}
\toprule
Target & Guaranteed & Changed & Bound & Observed & Single \\
\midrule
CLIP-L & 48.1 & 9.5 & 23.78 & 1.30 & 71.83 \\
OpenCLIP-L & 49.5 & 9.0 & 16.82 & -0.29 & 66.30 \\
MetaCLIP-L & 46.5 & 9.1 & 21.52 & 1.21 & 68.03 \\
EVA-CLIP-L & 47.2 & 11.8 & 28.60 & 3.58 & 75.85 \\
\midrule
Pooled & 47.8 & 9.9 & 22.68 & 1.45 & 70.50 \\
\bottomrule
\end{tabular}
\end{table}

\paragraph{Head targets.}
Nothing above uses that the row factor of a grid is the embedding itself. For SigLIP and SigLIP2, whose unnormalized image embedding $z(x) = a(x) + W_2 u(x) + b_2$ adds an attention-pooled branch $a(x)$ to the head's second affine $(W_2, b_2)$ applied to an activation $u(x)$, the update of Appendix~\ref{app:head} is solved with $F_\tau$ replaced by the matrix that stacks $[u(x_i); 1] / \lVert z(x_i) \rVert$ over the task's probes and with $X$ the rectangular update, and since the update changes the embedding by exactly $\Delta W_2 u(x) + \Delta b_2$ and the norm of the unedited embedding is common to every class of an image, the model scores class $j$ by $f_{\text{pre}}(x)^\top g_j + \alpha\, ([u(x); 1] / \lVert z(x) \rVert)^\top X g_j$ before normalization, exactly rather than to first order. Substituting that matrix for $F_\tau$ and that vector for $f$ in Equations~\eqref{eq:weights}, \eqref{eq:crossgrid}, \eqref{eq:interference-identity} and~\eqref{eq:margin} leaves every statement and every proof unchanged, so Proposition~\ref{prop:interference} and Corollary~\ref{cor:margin} cover all six targets of Tables~\ref{tab:addition-main} and~\ref{tab:negation-main}.

\paragraph{What these statements do not claim.}
They bound how far a sum of independently estimated operators departs from each of its terms on that term's own task and do not show that the departure is small, which depends on the tasks through the cross-grid matrices and on the descriptors through their norms. Nor do they show that the sum realizes the composed platonic task vector on the target, which would need each descriptor to be realized by some operator on the target up to a residual, together with enough spectral coverage of each grid for the ridge to recover it.

\subsection{Where the Cross-Model Residual Reaches the Edit}
\label{app:transport-error}

Section~\ref{sec:functional-invariance} measures how far apart two models' descriptors are, and Section~\ref{sec:exp-addition} shows through Oracle that this distance is what costs the method its accuracy. Neither says where the distance does its damage, and the two are not the same question, because a descriptor residual reaches the target only after passing through the solve. Proposition~\ref{prop:linearity} makes that passage exact. If a source model and the target's own fine-tune on task $\tau$ yield descriptors $\operatorname{UTD}^{(k)}_\tau$ and $\operatorname{UTD}^{(\text{tgt})}_\tau$ read on one grid, by which we mean one set of probes and prompts, then
\begin{equation}
    A^{*}\bigl(\operatorname{UTD}^{(k)}_\tau\bigr) - A^{*}\bigl(\operatorname{UTD}^{(\text{tgt})}_\tau\bigr)
    \;=\; A^{*}\bigl(\operatorname{UTD}^{(k)}_\tau - \operatorname{UTD}^{(\text{tgt})}_\tau\bigr),
    \label{eq:transport-error}
\end{equation}
so the operator the transfer gets wrong is the descriptor residual carried through the same linear map, whose gain in the direction of the $i$th probe and $j$th prompt singular vectors is $s_{a,i} s_{b,j} / ((s_{a,i}^2 + \mu_a)(s_{b,j}^2 + \mu_b))$ by Equation~\eqref{eq:closed-form}. That gain is small where the target's factors carry little energy and where they dominate, and largest in between, so a residual is neither uniformly suppressed nor uniformly passed. Which of the two happens is measurable, and it decides whether the raw agreement of Figure~\ref{fig:utd-verification} is the right quantity to read a transfer against.

\paragraph{Protocol.}
We evaluate both sides of Equation~\eqref{eq:transport-error} on the grids of Section~\ref{sec:exp-addition}, over the four targets that absorb the operator. For each of them, each task and each source family other than the target, we compare the source's operator with the operator the target's own fine-tune yields on that task's grid, which gives $160$ combinations of a target, a task and a source. Agreement is reported in three places, namely between the descriptors, between the operators, and between the shift fields $F_{\text{pre}} X$ that the operators induce on the $22{,}740$ probes of all eight tasks, the last being the operator as evaluation sees it, since evaluation applies an operator to images of every task and not only to those of the task it was fitted on. Against these we set the accuracy that each single-source operator reaches on its task at $\alpha = 0.5$ relative to the operator of the target's own fine-tune, normalized by the gain that own operator achieves over zero-shot, so that the quantity is the fraction of the attainable gain that transport forfeits. Normalizing matters because tasks differ by an order of magnitude in how much accuracy is available to gain, and the $150$ combinations whose attainable gain exceeds two points are the ones on which the fraction is well defined. An atomic operator does not depend on which composition it later enters, so these are the only distinct comparisons available.

\paragraph{The residual is amplified, not suppressed.}
Relative to the descriptor it is fitted to, the residual is amplified by a median factor of $1.5$ with an interquartile range of $1.1$ to $2.3$, and it is attenuated at $19$ percent of the combinations. Two models' operators consequently agree markedly less than their descriptors do, at a mean cosine of $0.39$ against $0.59$, with the shift fields between the two at $0.43$. The descriptors are read here on each task's own grid, whereas Figure~\ref{fig:utd-verification} flattens every descriptor to the grid of all eight tasks, where they agree at $0.47$. The solve therefore does not filter a moderate descriptor residual down to a smaller one, and since the operators nonetheless retain most of the attainable accuracy, accuracy asks far less of an operator than agreement with the target's own.

\paragraph{Agreement in operator space is the better predictor.}
What the solve does instead is redistribute the residual, and Table~\ref{tab:operr} shows that the redistributed version is the one that tracks the outcome. Its rank correlation with the fraction of the attainable gain that transport forfeits is twice as strong after the solve as before it. Which of five candidate sources is the best one is a harder question than how much a given source loses, and the table's last column shows that none of the three agreements answers it reliably, since all three land between $9$ and $10$ of the $32$ groups against $6.4$ by chance and the three are not separated at that count. There is no comparison to report at the level of a composed cell, because the two operators a cell sums are read on grids of different shapes and their descriptors therefore have no sum; a composed descriptor exists only when one grid serves both terms.

\begin{table}[H]
\centering
\small
\caption{Where a cross-model residual is measured decides how well it predicts what transfer costs. Each row compares a source's quantity with that of the target's own fine-tune on the same task and grid over $160$ combinations, and correlates it with the fraction of the attainable gain that the source's operator forfeits at $\alpha = 0.5$ over the $150$ combinations on which that fraction is defined. The last column counts, over the $32$ groups of five candidate sources sharing a target and a task, how often the source that the agreement ranks first is also the most accurate, where chance gives $6.4$; the three do not separate on that question.}
\label{tab:operr}
\begin{tabular}{lccc}
\toprule
Agreement measured on & Mean agreement & Rank correlation & Best source identified \\
\midrule
Descriptors                & $0.59$ & $-0.26$ & $10$ of $32$ \\
Operators                  & $0.39$ & $-0.62$ & $\phantom{0}9$ of $32$ \\
Shift fields on the probes & $0.43$ & $-0.59$ & $\phantom{0}9$ of $32$ \\
\bottomrule
\end{tabular}
\end{table}

Table~\ref{tab:operr-task} makes the failure of the descriptor cosine concrete by task. Cars and EuroSAT are the two tasks on which models' descriptors agree almost equally, at $0.81$ and $0.79$, and they transport very differently, since Cars forfeits $61$ percent of the attainable gain against EuroSAT's $11$. Measured after the solve the contradiction disappears, since Cars shares the lowest operator agreement of the eight with SUN397 at $0.07$ while EuroSAT has the highest at $0.71$. Across the eight tasks the rank correlation with the forfeited fraction is $-0.48$ for descriptors against $-0.83$ for operators and $-0.86$ for shift fields.

\begin{table}[H]
\centering
\small
\caption{Per-task view of the same $160$ combinations, each agreement averaged over the four targets and their five sources. Agreement is a cosine and the gain is that of the target's own operator over zero-shot. The last row gives the fraction of that gain which transport forfeits, averaged over those of the $150$ combinations of Table~\ref{tab:operr} that fall on the task. Cars is the task whose descriptor agreement and transport quality disagree most sharply, and SUN397 is the one on which every measure agrees that transport fails.}
\label{tab:operr-task}
\setlength{\tabcolsep}{4pt}
\begin{tabular}{lcccccccc}
\toprule
& EuroSAT & DTD & Cars & GTSRB & MNIST & RESISC45 & SVHN & SUN397 \\
\midrule
Descriptor agreement    & $0.79$ & $0.41$ & $0.81$ & $0.89$ & $0.51$ & $0.87$ & $0.41$ & $0.06$ \\
Operator agreement      & $0.71$ & $0.25$ & $0.07$ & $0.48$ & $0.56$ & $0.37$ & $0.63$ & $0.07$ \\
Shift-field agreement   & $0.73$ & $0.37$ & $0.13$ & $0.47$ & $0.57$ & $0.42$ & $0.63$ & $0.11$ \\
Gain forfeited & $0.11$ & $0.33$ & $0.61$ & $0.14$ & $0.13$ & $0.24$ & $0.25$ & $0.85$ \\
\bottomrule
\end{tabular}
\end{table}

\paragraph{What this does and does not license.}
Two readings are available and only one is supported. The measurement does not say that the descriptor is the wrong object to transfer, since a descriptor is what a source can hand a target it shares no coordinates with. What it says is that the descriptor cosine answers a different question from the one a reader might take it to answer. Figure~\ref{fig:utd-verification} establishes that descriptors carry a shared, task-specific component at all, which is the claim Section~\ref{sec:functional-invariance} makes and which the cross-task panel is needed for, and it is not a forecast of how well any particular source will serve any particular target. The quantity that forecasts that is target-dependent by construction, since it is defined only after a target's factors are fixed, so it could not have played the earlier role. Nothing here is used to select sources in the experiments, where every cell of every table is reported and none is chosen; the measurement is diagnostic and explains why a single transferred operator forfeits on average $30$ percent of the gain the target's own operator attains, over the $150$ combinations of Table~\ref{tab:operr}, while remaining far above the alternatives of Section~\ref{sec:exp-addition}.

\subsection{Is a Low-Rank Prior Helpful?}
\label{app:rank}

Low-rank edits are the norm in parameter-efficient adaptation, and one might expect a rank budget to act as a useful prior on the operator. It does not. A descriptor read against a task's own landmarks constrains the operator on at most as many directions as that task defines prompts, which across the eight tasks runs from ten to $397$, so the question is asked at six distinct ceilings rather than at one. We solve Equation~\eqref{eq:lsq} under an explicit rank cap by reduced-rank regression, which truncates in the metric that $F_{\text{pre}}$ and $G^\top$ induce rather than truncating the unconstrained solution, and select the cap by held-out reconstruction of the descriptor on a fifth of the task's probes, a criterion that uses no labels and no evaluation data.

\begin{table}[H]
\centering
\small
\caption{Held-out descriptor reconstruction error of the closed form against the rank cap $r$, pooled over the four targets that absorb the operator and their five cross-model sources, with the cap each task selects under that label-free criterion. A dash marks a cap the closed form cannot reach on that task, since a descriptor read against a task's own landmarks constrains the operator on at most as many directions as the task has prompts. Each entry averages over the ridge settings, whereas the criterion selects the cap and the ridge jointly, and on every task it takes the largest cap the sweep offers below the ceiling. On seven of the eight tasks the averaged error falls with the cap; on SVHN it rises instead, which is also the task on which the closed form is least accurate.}
\label{tab:rank}
\setlength{\tabcolsep}{4pt}
\begin{tabular}{lcccccccccc}
\toprule
& & \multicolumn{8}{c}{Rank cap $r$} & \\
\cmidrule(lr){3-10}
Task & Prompts & $2$ & $4$ & $8$ & $16$ & $32$ & $64$ & $128$ & $256$ & Selected $r$ \\
\midrule
EuroSAT & $10$ & $0.562$ & $0.490$ & $0.445$ & --- & --- & --- & --- & --- & $8$ \\
DTD & $47$ & $0.694$ & $0.657$ & $0.637$ & $0.629$ & $0.629$ & --- & --- & --- & $32$ \\
Cars & $196$ & $0.397$ & $0.375$ & $0.355$ & $0.341$ & $0.336$ & $0.334$ & $0.334$ & --- & $128$ \\
GTSRB & $43$ & $0.345$ & $0.322$ & $0.305$ & $0.295$ & $0.293$ & --- & --- & --- & $32$ \\
MNIST & $10$ & $0.725$ & $0.615$ & $0.552$ & --- & --- & --- & --- & --- & $8$ \\
RESISC45 & $45$ & $0.352$ & $0.336$ & $0.321$ & $0.310$ & $0.306$ & --- & --- & --- & $32$ \\
SVHN & $10$ & $0.816$ & $0.821$ & $0.832$ & --- & --- & --- & --- & --- & $8$ \\
SUN397 & $397$ & $0.527$ & $0.507$ & $0.489$ & $0.473$ & $0.462$ & $0.457$ & $0.454$ & $0.454$ & $256$ \\
\bottomrule
\end{tabular}
\end{table}

Table~\ref{tab:rank} reports that error at every cap, and accuracy agrees with it. Paired over the $100$ combinations of a target, a task and a source in which the cap of $16$, equal to the adapter rank $r_{\text{LoRA}}$, lies below the task's own ceiling, the selected cap reaches $75.91$ percent on the transferred task against $74.15$ at that cap. SVHN is the one task whose reconstruction error, averaged over the ridge settings, rises with the cap rather than falling, and it is also the task on which the route is least accurate, at $47.64$ percent at the selected cap. A rank budget is therefore a cost rather than a benefit at every ceiling we can test, so the unconstrained closed form loses nothing by not imposing one. It also matters for Proposition~\ref{prop:linearity}, because truncation is nonlinear in $U$, so a rank-capped solve would forfeit the linearity that lets a strength be swept and a bank entry be scaled without solving again.

\subsection{What the Descriptor Does Not Determine}
\label{app:underdetermined}

Equation~\eqref{eq:factorization} implies that a descriptor constrains the shift field only inside $\operatorname{span}(G)$, leaving $d - \operatorname{rank}(G)$ directions unconstrained, since adding to the shift field any matrix whose rows are orthogonal to every landmark produces the same descriptor. What a descriptor determines is therefore an equivalence class of shift fields, and the ridge terms select its minimum-norm representative, which is a choice and not a consequence. We tested two alternatives and report both as negative results in Appendix~\ref{app:ablations}, namely transporting the full shift field by a Procrustes map fitted between models, which is strictly worse than using the descriptor despite constraining as many directions as the embedding has rather than as many as the task has prompts, and whitening the prompt Gram, which lowers cross-model agreement further than it lowers the different-task control and therefore shrinks the gap that makes descriptors discriminative.

\section{The Hypothesis under the PRH Factorization}
\label{app:theory}

This appendix shows how Hypothesis~\ref{hyp:consistency} follows to first order from the near-isometric factorization behind PRH once the platonic task vector is posited, and defines the diagnostic whose value Section~\ref{sec:method-connection} quotes.

\paragraph{Setup.}
Identify $\mathcal{V}$ with $\mathbb{R}^{d_\star}$, in which images and prompts have unit-norm shared representations $\phi(x)$ and $\psi(p)$. PRH holds that each model's encoders factor through this space by a model-specific linear map $T^{(k)} : \mathbb{R}^{d_\star} \to \mathbb{R}^{d^{(k)}}$,
\begin{equation}
    f^{(k)}(x) \approx \frac{T^{(k)}\phi(x)}{\|T^{(k)}\phi(x)\|},
    \qquad
    g^{(k)}(p) \approx \frac{T^{(k)}\psi(p)}{\|T^{(k)}\psi(p)\|},
    \label{eq:app-factorization}
\end{equation}
that is close to an isometry, with isometry defect $\kappa^{(k)}$ the smallest constant for which $|\langle T^{(k)}\xi, T^{(k)}\xi' \rangle - \langle \xi, \xi' \rangle| \le \kappa^{(k)} \|\xi\| \|\xi'\|$ for all $\xi, \xi' \in \mathbb{R}^{d_\star}$, and it posits $\kappa^{(k)} \to 0$ as models grow~\citep{huh2024platonic}. The platonic task vector enters as the change fine-tuning makes to the shared representation, $\phi(x) \mapsto \phi(x) + \Pi_\tau(x)$, the same for every model, with $\rho_\tau = \max_i \|\Pi_\tau(x_i)\|$ small over the probes. We take $\Pi_\tau(x)$ tangent to the unit sphere at $\phi(x)$, which loses nothing, since renormalization removes a radial component to first order and leaves it unidentifiable from any observation.

\paragraph{Derivation.}
The squared norm of $T^{(k)}(\phi + \Pi_\tau)$, arguments suppressed, is $\|T^{(k)}\phi\|^2 + 2\langle T^{(k)}\phi, T^{(k)}\Pi_\tau \rangle + \|T^{(k)}\Pi_\tau\|^2$, whose middle term is $O(\kappa^{(k)}\rho_\tau)$ by tangency and near-isometry and whose last term is $O(\rho_\tau^2)$, so to first order the fine-tuned embedding is normalized as the pre-trained one is and
\begin{equation}
    f^{(k)}_{\text{ft}}(x_i) - f^{(k)}_{\text{pre}}(x_i)
    \;=\; T^{(k)}\Pi_\tau(x_i) + O(\rho_\tau^2) + O(\kappa^{(k)}\rho_\tau).
    \label{eq:app-tangent-projection}
\end{equation}
Reading this shift against $g^{(k)}(p_j)$, which differs from $T^{(k)}\psi(p_j)$ by $O(\kappa^{(k)})$, and applying near-isometry once more gives
\begin{equation}
    \bigl[\operatorname{UTD}^{(k)}_\tau\bigr]_{ij}
    \;=\; \bigl\langle \Pi_\tau(x_i), \psi(p_j) \bigr\rangle + \bigl[E^{(k)}_\tau\bigr]_{ij},
    \qquad
    \bigl|\bigl[E^{(k)}_\tau\bigr]_{ij}\bigr| \;\le\; C_1 \rho_\tau^2 + C_2\, \kappa^{(k)} \rho_\tau,
    \label{eq:residual-identification}
\end{equation}
with $C_1$ and $C_2$ constants of the sphere's geometry independent of $k$ and $\tau$. This is Hypothesis~\ref{hyp:consistency} with $\mathcal{A}(\Pi_\tau)$ as its shared term, and its residual has a part quadratic in the displacement, of one size for every model, and a part that grows with the model's isometry defect, which is where models differ. Summing over the $NM$ entries gives $\|E^{(k)}_\tau\|_F \le \sqrt{NM}\,(C_1\rho_\tau^2 + C_2\kappa^{(k)}\rho_\tau)$, so $\eta$ is small when the displacement and the defects are both small against the shared signal per entry. What the factorization does not produce is the premise that one displacement arises in every model, which is a statement about training dynamics and would require the tangent kernels~\citep{jacot2018neural} of different models' affinities to agree on the directions that matter, the regime in which task arithmetic has been analyzed~\citep{ortiz2023task}. That premise is the content of the platonic task vector, and Section~\ref{sec:functional-invariance} is its empirical support.

\subsection{Measuring the Residual}
\label{app:error}

Since $\mathcal{A}(\Pi_\tau)$ is not observed, $\eta$ cannot be read off directly, and we probe it by the plug-in dispersion diagnostic
\begin{equation}
    \hat{\eta}^{(k)}_\tau
    \;=\;
    \frac{\bigl\| \operatorname{UTD}^{(k)}_\tau - \overline{\operatorname{UTD}}_\tau \bigr\|_F}
         {\bigl\| \overline{\operatorname{UTD}}_\tau \bigr\|_F},
    \qquad
    \overline{\operatorname{UTD}}_\tau = \frac{1}{K} \sum_{k'=1}^{K} \operatorname{UTD}^{(k')}_\tau ,
    \label{eq:eta-hat}
\end{equation}
where $k'$ is a dummy index over the $K$ models and $\overline{\operatorname{UTD}}_\tau$ is the descriptor of task $\tau$ averaged over models. Its mean over all evaluated pairs of a model and a task is $0.99$, with family means between $0.77$ and $1.27$. This diagnostic and Figure~\ref{fig:utd-verification} read every descriptor on one probe set spanning the eight tasks rather than on each task's own, which the cross-task panel requires, since descriptors measured on grids of different shapes cannot be compared at all, and it is a measurement of the descriptors rather than a fit. Its numerator equals $\|E^{(k)}_\tau - \overline{E}_\tau\|_F$ with $\overline{E}_\tau$ the residual averaged over models, so that any residual shared by all models cancels, and its denominator equals $\|\mathcal{A}(\Pi_\tau) + \overline{E}_\tau\|_F$. The triangle inequality bounds the numerator by $2\eta \|\mathcal{A}(\Pi_\tau)\|_F$ and, when $\eta < 1$, the denominator from below by $(1-\eta)\|\mathcal{A}(\Pi_\tau)\|_F$, so $\hat{\eta}^{(k)}_\tau \le 2\eta / (1-\eta)$ and $\eta \ge \hat{\eta}^{(k)}_\tau / (2 + \hat{\eta}^{(k)}_\tau)$, a bound that holds trivially when $\eta \ge 1$, and the observed value therefore implies $\eta \ge 0.33$. The diagnostic gives no upper bound. The pairwise agreement of Section~\ref{sec:functional-invariance} locates $\eta$ more closely under an idealization, since if two models' residuals were orthogonal to each other and to the shared component and of equal norm $\eta \|\mathcal{A}(\Pi_\tau)\|_F$, their descriptors would meet at cosine $1 / (1 + \eta^2)$, and the measured $0.473$ then corresponds to $\eta$ close to one. The shared component and the residual are therefore of comparable norm, which is the distance from the platonic task vector at which the experiments measure transfer and composition.

\section{Experimental Setup Details}
\label{app:setup}

\subsection{Model Checkpoints}
\label{app:inventory}

Table~\ref{tab:inventory} lists every checkpoint that supplies a descriptor; the four cross-scale targets of Appendix~\ref{app:scale} only receive one. All follow the CLIP-style dual-encoder setup with separately encoded image and text streams projected into a shared space of dimension $d^{(k)}$. Sources are fine-tuned on the query and value projections of the vision encoder's self-attention only, and text encoders are frozen throughout. Whether a family can receive a descriptor as well as supply one is decided by whether one layer carries the whole embedding, which is what the edit of Equation~\eqref{eq:absorb} folds into, or, when the head sums a residual branch, whether its last affine reaches the embedding on its own (Appendix~\ref{app:head}).

\begin{table}[H]
\centering
\small
\caption{Models used in this work, with $d^{(k)}$ the embedding dimension and the adapted modules those fine-tuned to produce a descriptor. All vision sources are at Large scale. The two SigLIP families pool through an attention head whose output is a residual sum, so they receive an edit through that head's last affine rather than through Equation~\eqref{eq:absorb}.}
\label{tab:inventory}
\setlength{\tabcolsep}{4pt}
\begin{tabular}{llcl}
\toprule
Family & Identifier & $d^{(k)}$ & Adapted modules \\
\midrule
CLIP      & \texttt{openai/clip-vit-large-patch14}         & $768$  & \texttt{q\_proj}, \texttt{v\_proj} \\
OpenCLIP  & \texttt{laion/CLIP-ViT-L-14-laion2B-s32B-b82K} & $768$  & \texttt{q\_proj}, \texttt{v\_proj} \\
MetaCLIP  & \texttt{facebook/metaclip-l14-400m}            & $768$  & \texttt{q\_proj}, \texttt{v\_proj} \\
EVA-CLIP  & \texttt{EVA02-L-14} (via \texttt{open\_clip})  & $768$  & \texttt{attn.\{q,v\}\_proj} \\
SigLIP    & \texttt{google/siglip-large-patch16-256}       & $1024$ & \texttt{q\_proj}, \texttt{v\_proj} \\
SigLIP2   & \texttt{google/siglip2-large-patch16-256}      & $1024$ & \texttt{q\_proj}, \texttt{v\_proj} \\
\midrule
CLAP      & \texttt{unfused}, \texttt{fused} ($768$-wide, $12$ blocks)   & $512$ & \texttt{query}, \texttt{value} \\
CLAP      & \texttt{large}, \texttt{large-ms} ($1024$-wide, $18$ blocks) & $512$ & \texttt{query}, \texttt{value} \\
\bottomrule
\end{tabular}
\end{table}

\subsection{Sources, Probes and Hyperparameters}
\label{app:datasets}

\textbf{Source models.} Each of the $6 \times 8 = 48$ vision sources is a low-rank fine-tune~\citep{hu2022lora} of rank $r_{\text{src}} = 4$ and scaling factor $\alpha_{\text{src}} = 8$ of the source vision encoder, trained with AdamW for $3$ epochs, or $1$ for SUN397. The EuroSAT and DTD sources of CLIP, OpenCLIP, MetaCLIP and SigLIP use learning rate $3 \times 10^{-4}$, batch $64$ and weight decay $10^{-4}$, and the remaining forty use learning rate $10^{-4}$, batch $32$ and weight decay $0.01$. The audio sources are the corresponding rank-$4$ CLAP fine-tunes. Sources exist only to supply descriptors and are never transferred as parameters, so the target reads a source only through its descriptor.

\textbf{Probe grid.} Probes are unlabeled images and prompts are class-name templates. A grid holds one task, namely the task whose descriptor is being measured and transferred, so a solve never sees another task's data. A task's pool holds thirty images for each class that task defines, so a grid ranges from $300$ images and $10$ prompts for a ten-class task up to $11{,}910$ images and $397$ prompts for SUN397. The images are drawn uniformly at random from the task's data as a whole rather than allotted class by class, so the draw reads no label and a pool is representative of its classes only in expectation. Every task carries fewer prompts than any target has embedding dimensions, the largest being SUN397's $397$ against the $512$ of the narrowest target, so $G$ can have full row rank on every grid the paper solves on, and does on all $48$ grids of the six Large targets. In the audio setting each task is split once at seed $42$ into a probe half, which also serves to fine-tune the sources, and a disjoint test half used only for evaluation. ESC-50's grid holds all $1{,}000$ clips of its probe half, fewer than its quota of $1{,}500$, with $50$ prompts, and GTZAN's holds $300$ of the $499$ clips of its probe half with $10$ prompts. The datasets of Section~\ref{sec:exp-holdout} are probed at twenty images per class rather than thirty, which gives grids of $900$, $600$ and $420$ images with $45$, $30$ and $21$ prompts.

\textbf{Probe splits.} Probes carry no labels into either realization, and every task supplies them from its training split, the split its sources were fine-tuned on, while its accuracy is measured on its held-out test split, so no evaluation image enters any probe set and probe and evaluation data are disjoint by construction throughout the paper. The baselines that consume labeled or calibration examples, namely GradFix, Theseus, the larger gradient-ascent budget, distillation and supervised tuning, read the same training-split probes. The two controls never enter a probe set at all.

\textbf{Solver.} Ridge coefficients $\mu_a = 10^{-2}\,\overline{s_a^2}$ and $\mu_b = 10^{-6}\,\overline{s_b^2}$, full rank, float64 for the solve and float32 for evaluation, seed $42$.

\textbf{Compute.} All experiments run on four NVIDIA RTX A6000 GPUs under PyTorch 2.6 and Python 3.10.

\textbf{Controls.} CIFAR-100 and ImageNet at $10$ images per class with seed $42$, except that Appendix~\ref{app:holdout-full} scores CIFAR-100 on its full test split; UrbanSound8K at $50$ per class in the audio setting, since it has only $10$ classes and $10$ per class would be a $100$-clip control.

\subsection{Algorithms}
\label{app:algorithms}

\begin{algorithm}[H]
\caption{Atomic solves for one target}
\label{alg:atomic}
\begin{algorithmic}[1]
\REQUIRE target $\mathcal{M}_{\text{tgt}}$; for each task $\tau$ its probes $\{x_i\}$ and prompts $\{p_j\}$; source descriptors $\{\operatorname{UTD}^{(k)}_\tau\}$
\ENSURE operator bank $\{A^{*(k)}_{\tau}\}$, with $\Sigma_a^{+}$ and $\Sigma_b^{+}$ the spectral filters of Equation~\eqref{eq:closed-form}
\FOR{each task $\tau$}
    \STATE $F_\tau \leftarrow$ row-stack unit-norm $f^{(\text{tgt})}_{\text{pre}}(x_i)$; \quad $G_\tau \leftarrow$ row-stack unit-norm $g^{(\text{tgt})}(p_j)$
    \STATE $L_a \Sigma_a R_a^\top \leftarrow \operatorname{svd}(F_\tau)$; \quad $L_b \Sigma_b R_b^\top \leftarrow \operatorname{svd}(G_\tau^\top)$ \COMMENT{target and task only}
    \FOR{each source model $k$}
        \STATE $A^{*(k)}_{\tau} \leftarrow \bigl(R_a \Sigma_a^{+} (L_a^\top \operatorname{UTD}^{(k)}_\tau R_b) \Sigma_b^{+} L_b^\top\bigr)^{\!\top}$ \COMMENT{three matrix products}
    \ENDFOR
\ENDFOR
\RETURN $\{A^{*(k)}_{\tau}\}$
\end{algorithmic}
\end{algorithm}

\begin{algorithm}[H]
\caption{Any edit, from the bank}
\label{alg:edit}
\begin{algorithmic}[1]
\REQUIRE operator bank; operation with terms $(\ell, \sigma(\ell), \tau_\ell)$ and signed weights $c_\ell$; strength $\alpha$
\ENSURE edited target
\STATE $A \leftarrow \sum_\ell c_\ell\, A^{*(\sigma(\ell))}_{\tau_\ell}$ \COMMENT{one embedding space, no solve}
\STATE $W_p \leftarrow (I + \alpha A)\, W_p$ \COMMENT{or the head update of Appendix~\ref{app:head} on SigLIP and SigLIP2}
\RETURN the target with its last layer replaced
\end{algorithmic}
\end{algorithm}

Addition sets every $c_\ell = 1$ and negation uses a single term with $c_1 = -\lambda$. Line~1 of Algorithm~\ref{alg:edit} is the whole of composition. It is well defined because every entry of the bank is a matrix on the target's one embedding space, whichever task's probes and prompts produced it, and it is the composition Equation~\eqref{eq:operator} prescribes; Proposition~\ref{prop:linearity} adds that when the summed entries share a grid the line also returns the operator of the composed descriptor (Appendix~\ref{app:linearity}). Because $\operatorname{normalize}\bigl((I + \alpha A) W_p h(x)\bigr)$ depends on $W_p h(x)$ only through its direction, an edit is scored from unit-norm base features cached once per target with one matrix product per point of a sweep, bit-identically to the real weight edit of Equation~\eqref{eq:absorb}, which is why every sweep in the paper is dense. In cost, Algorithm~\ref{alg:atomic} spends $(N + M)$ encoder passes and two singular value decompositions once per task, both reused across that task's source descriptors, and $O(Nd\,M + d^2 M)$ per solve, while Algorithm~\ref{alg:edit} spends $O(d^2)$ per term plus one $d \times d$ matrix product to absorb.

\section{The Adapter Realization}
\label{app:adapter}

\paragraph{Objective.}
The adapter route minimizes Equation~\eqref{eq:match} over the parameters $\Delta\theta$ of a low-rank adapter~\citep{hu2022lora} of rank $r_{\text{LoRA}} = 16$ and scaling factor $\alpha_{\text{LoRA}} = 32$, attached to the query and value projections of every attention layer of the target's vision encoder, with the text encoder and every other weight frozen. A minibatch draws probe images from the union of the terms' probes, each image is read against the prompts of its own task only, and its loss is the mean squared difference, over those prompts, between the change of its affinities and the prescribed shift $c_\ell\, [\operatorname{UTD}^{(\sigma(\ell))}_{\tau_\ell}]_{i\cdot}$, so every probe image weighs equally, which is the normalization written into Equation~\eqref{eq:match}. The unedited affinities the change is measured against are computed by the same forward pass with the adapter disabled, so an adapter at its initialization, whose second factor is zero, induces exactly zero change, and its loss is then the mean squared prescribed shift.

\paragraph{Recipe.}
AdamW with learning rate $3 \times 10^{-4}$ and weight decay $10^{-4}$, batch $64$ and a cosine schedule, for $20$ epochs over the edit's probes capped at $540$ optimizer steps, so a single edit sees at most $34{,}560$ probe presentations, the cap binding only when the edit's probes exceed $1{,}728$ images. The forward pass runs in bfloat16 with the loss in single precision. Every cell starts from one fixed adapter initialization and one fixed order of minibatches, both drawn at seed $42$, so two cells differ only in their descriptors and probes. The adapter realizes the prescribed shift itself, $c_\ell = 1$ for every addition term and $c_1 = -\lambda$ with $\lambda = 1$ for negation, as discussed under strength below.

\paragraph{Evaluation.}
An edited target is scored by a real forward pass through the adapted encoder on the full evaluation splits and on the two control draws the closed form is scored on. The unedited target scored by the same bfloat16 forward pass reproduces the single-precision zero-shot accuracies of the closed form to within $0.5$ points on every split of every target in Tables~\ref{tab:addition-main} and~\ref{tab:negation-main} and by $+0.03$ points on average, so their zero-shot rows serve both routes. On the held-out datasets of Table~\ref{tab:holdout} the largest such offset is $0.72$ points, on the UCMerced split of CLIP-L.

\paragraph{Fit and cost.}
Writing the relative residual of a cell as its final loss over the loss of the untrained adapter, the median over the $438$ addition cells is $0.041$ and over the $134$ negation cells $0.077$, so the adapter reproduces the prescribed shift on the probes closely. A median addition cell trains for $540$ steps in about ten minutes on one A6000, and a median negation cell for $420$ steps in about four; the closed form solves a target's whole bank of $48$ operators in seconds, after which each cell is a matrix sum.

\paragraph{Which cells the adapter covers.}
Its rows in the any-family blocks cover one draw per target and task pair for addition and one per target and task for negation, fixed at seed $42$ before any cell was fitted, so each subset contains every target and every task in the same proportion as the full block and its zero-shot row equals the full block's. The closed form on those subsets reaches $79.05$ and $72.65$ on the composed tasks against $79.08$ and $72.60$ over all $3{,}360$ cells, and removes $34.5$ points of the targeted task against $35.2$ over all $240$, so the subsets are representative of the populations they are drawn from. Both operations keep every same-model and every architecture cell, which are the ones the baselines are read against.

\paragraph{Repeatability.}
Ten cells belong to both the architecture block and the any-family draw and were fitted twice, from the same initialization, batch order and data, and their accuracies differ by $0.19$ points on average and by at most $1.10$, which is the nondeterminism of the bfloat16 forward and backward passes on the GPU and a floor on how finely two adapter rows can be compared cell by cell.

\paragraph{Strength.}
The adapter's addition rows realize the composed descriptor at full strength, $c_\ell = 1$, which is what Equation~\eqref{eq:match} prescribes for addition. We first fitted the $28$ any-family cells of the OpenCLIP-L target at $c_\ell = 0.5$, the strength the closed form applies to its operator, and there the adapter realizes half of each task's shift and gains $6.41$ and $1.86$ points less than the closed form on the two tasks, while holding CIFAR-100 and ImageNet $1.04$ and $1.61$ points closer to zero-shot. The closed form's strength scales an operator that acts on every embedding, whereas the adapter's scales the shift it is fitted to reproduce on the probes, so the two numbers do not name the same edit size, and we report each route at the strength its own construction fixes, $\alpha = 0.5$ for the operator as in Appendix~\ref{app:addition-full} and full strength for the fitted descriptor.

\section{Full Addition Results}
\label{app:addition-full}

\subsection{Strength Sweep}

Table~\ref{tab:add-alpha} gives the full sweep behind Table~\ref{tab:addition-main}. Acquisition has stopped growing by $\alpha = 0.5$, where the mean over the two composed tasks peaks, task A gaining only $0.3$ points more at $\alpha = 1$ while task B loses $1.2$, and control damage keeps growing, so the useful range ends where the sweep begins. This is the expected shape, because the objective of Equation~\eqref{eq:lsq} contains no term for what must be preserved, so nothing in the fit resists moving inputs it was never shown.

\begin{table}[H]
\centering
\small
\caption{Addition against strength, over all $3{,}360$ cross-model cells of the six targets, as change from zero-shot in percentage points. Single-source rows are paired to the same cells, so the difference between Ours and Single-A on task A is the per-cell cost of composing.}
\label{tab:add-alpha}
\begin{tabular}{llcccc}
\toprule
$\alpha$ & Method & $\Delta$ Task A & $\Delta$ Task B & $\Delta$ CIFAR-100 & $\Delta$ ImageNet \\
\midrule
\multirow{3}{*}{$0.5$} & Ours & $+18.11$ & $+11.63$ & $\phantom{0}-1.81$ & $\phantom{0}-3.25$ \\
 & Single-A & $+19.58$ & $-1.12$ & $\phantom{0}-0.77$ & $\phantom{0}-1.32$ \\
 & Single-B & $-1.53$ & $+12.87$ & $\phantom{0}-0.94$ & $\phantom{0}-1.83$ \\
\midrule
\multirow{3}{*}{$1.0$} & Ours & $+18.44$ & $+10.45$ & $\phantom{0}-8.18$ & $-12.74$ \\
 & Single-A & $+21.79$ & $-4.39$ & $\phantom{0}-3.64$ & $\phantom{0}-5.47$ \\
 & Single-B & $-5.46$ & $+13.77$ & $\phantom{0}-4.55$ & $\phantom{0}-7.54$ \\
\midrule
\multirow{3}{*}{$1.5$} & Ours & $+16.59$ & $\phantom{0}+7.16$ & $-17.21$ & $-23.93$ \\
 & Single-A & $+21.53$ & $-8.26$ & $\phantom{0}-7.98$ & $-10.96$ \\
 & Single-B & $-10.00$ & $+12.37$ & $-10.04$ & $-14.98$ \\
\midrule
\multirow{3}{*}{$2.0$} & Ours & $+14.50$ & $\phantom{0}+3.65$ & $-26.46$ & $-34.14$ \\
 & Single-A & $+20.69$ & $-12.13$ & $-13.02$ & $-16.80$ \\
 & Single-B & $-14.40$ & $+10.08$ & $-16.33$ & $-22.70$ \\
\bottomrule
\end{tabular}
\end{table}

\subsection{Per-Target Breakdown}

Table~\ref{tab:add-per-target} resolves Table~\ref{tab:addition-main} by target.

\begin{table}[H]
\centering
\small
\caption{Addition at $\alpha = 0.5$ by target, $560$ cells each, as change from zero-shot. The four targets that absorb the operator acquire the composed tasks by $13.04$ to $18.17$ points and differ far more in what they give up, EVA-CLIP paying roughly five times what OpenCLIP pays on ImageNet. The two that receive an update to a pooling head instead gain a little less and sit inside the range the others occupy on both controls, so the absorption route is not what decides the trade. What a target gains and what it pays are separate axes.}
\label{tab:add-per-target}
\begin{tabular}{lcccc}
\toprule
Target & $\Delta$ Composed tasks & $\Delta$ CIFAR-100 & $\Delta$ ImageNet & Zero-shot CIFAR-100 \\
\midrule
CLIP-L     & $+17.07 \pm 9.54$ & $-1.88$ & $-4.07$ & $74.2$ \\
OpenCLIP-L & $+18.17 \pm 10.86$ & $-0.62$ & $-1.42$ & $74.7$ \\
MetaCLIP-L & $+16.24 \pm 9.68$ & $-1.83$ & $-2.74$ & $71.0$ \\
EVA-CLIP-L & $+13.04 \pm 8.06$ & $-2.46$ & $-7.43$ & $87.9$ \\
\midrule
SigLIP-L   & $+11.75 \pm 8.46$ & $-1.62$ & $-1.48$ & $72.0$ \\
SigLIP2-L  & $+12.93 \pm 7.35$ & $-2.44$ & $-2.35$ & $75.1$ \\
\bottomrule
\end{tabular}
\end{table}

\subsection{Per-Task Breakdown}
\label{app:add-per-task}

Table~\ref{tab:add-per-task} resolves the same cells by composed task, each task entering $840$ of them, and sets beside each the gain that the target's own descriptors reach on that task through the same operator, which is Oracle. Six of the eight tasks retain between $62$ and $87$ percent of that gain, and on those six every one of the $180$ combinations of a target, a source and a task gains on average. By source family, the mean gain on the task a family supplies runs from $13.82$ points for SigLIP to $15.74$ for MetaCLIP. Cars and SUN397, the two tasks with the most classes, are where the target's own descriptors gain least, about two points, and where a transferred descriptor gains nothing, and all $38$ combinations whose mean gain is not positive fall on those two tasks.

\begin{table}[H]
\centering
\small
\caption{Addition at $\alpha = 0.5$ by composed task over the $3{,}360$ cross-model cells of the six targets, as change from zero-shot in percentage points, with the spread over cells. Oracle averages the $42$ cells per task in which the target supplies its own descriptors. The retained fraction is the ratio of the two gains and is omitted where Oracle gains under three points, since the ratio is then dominated by noise.}
\label{tab:add-per-task}
\setlength{\tabcolsep}{4pt}
\begin{tabular}{lccc}
\toprule
Task & $\Delta$ Ours & $\Delta$ Oracle & Fraction retained \\
\midrule
EuroSAT  & $+33.55 \pm 4.23$  & $+38.76$ & $0.87$ \\
DTD      & $+6.37 \pm 2.47$   & $+10.20$ & $0.62$ \\
Cars     & $-0.30 \pm 2.42$   & $+1.90$  & --- \\
GTSRB    & $+30.24 \pm 5.03$  & $+35.81$ & $0.84$ \\
MNIST    & $+28.45 \pm 13.35$ & $+33.75$ & $0.84$ \\
RESISC45 & $+12.85 \pm 2.40$  & $+17.36$ & $0.74$ \\
SVHN     & $+8.99 \pm 6.51$   & $+13.92$ & $0.65$ \\
SUN397   & $-1.23 \pm 2.78$   & $+2.00$  & --- \\
\bottomrule
\end{tabular}
\end{table}

\subsection{Per-Cell View}

Table~\ref{tab:addition-main} reports means, and a mean can hide a bimodal population, so Figure~\ref{fig:pairs} plots the cells themselves. Each marker is one cell on the OpenCLIP-L target, placed by its accuracy on the two composed tasks, each normalized by that task's fine-tuning ceiling. Zero-shot cells fall on a grid set by the zero-shot accuracy of each task, and the composed cells move up and to the right of it together. A single-source edit cannot produce that pattern, because it leaves one of the two coordinates where it was.

\begin{figure}[H]
    \centering
    \includegraphics[width=0.42\linewidth]{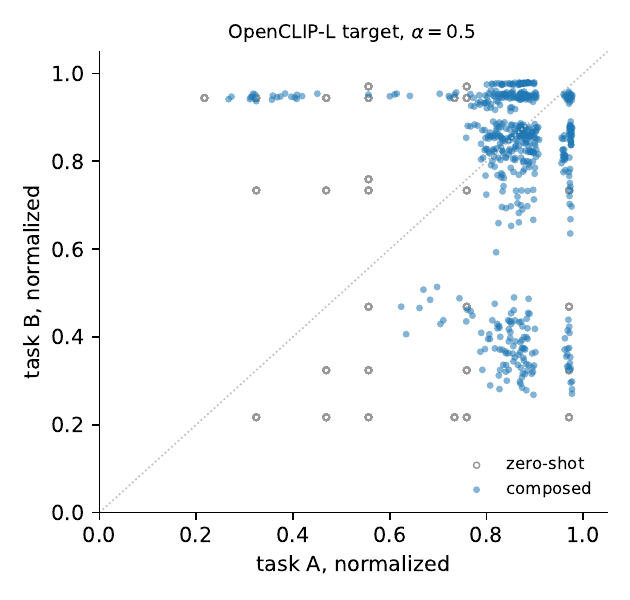}
    \caption{Addition at $\alpha = 0.5$ on the OpenCLIP-L target, one marker per cell, with each axis normalized by the per-task fine-tuning ceiling of Table~\ref{tab:ft-ceilings}. The composed operator moves both coordinates at once.}
    \label{fig:pairs}
\end{figure}

\subsection{The Cost of Acquiring a Task}
\label{app:ft-cost}

The strength $\alpha$ buys accuracy on the composed tasks and gives up accuracy on the two controls, and how large that second quantity is cannot be judged in isolation. The natural reference is what acquiring a task costs when the task is learned directly, since that is the single-task adaptation descriptor transfer approximates and no transfer method can be expected to leave the controls in better shape than the adaptation it reproduces. Table~\ref{tab:ft-cost} applies each target's own source for one task to that target and reports both controls as a change from that target's zero-shot accuracy.

\begin{table}[H]
\centering
\small
\caption{Change in the two control benchmarks when a target is fine-tuned on one task, in percentage points against that target's own zero-shot accuracy. Each entry applies the target's own source for the row task to the target itself, which is the single-task adaptation descriptor transfer approximates. The upper block is CIFAR-100 and the lower block ImageNet.}
\label{tab:ft-cost}
\begin{tabular}{lccccc}
\toprule
Task & CLIP-L & OpenCLIP-L & MetaCLIP-L & EVA-CLIP-L & Mean \\
\midrule
\multicolumn{6}{l}{\textit{CIFAR-100}} \\
EuroSAT & $-0.60$ & $-0.00$ & $-9.70$ & $-4.40$ & $-3.68$ \\
DTD & $-0.60$ & $-0.10$ & $-0.80$ & $-2.80$ & $-1.08$ \\
Cars & $+0.90$ & $+0.30$ & $+1.80$ & $+0.40$ & $+0.85$ \\
GTSRB & $-27.90$ & $-29.70$ & $-23.80$ & $-41.90$ & $-30.83$ \\
MNIST & $-2.80$ & $-3.80$ & $-3.80$ & $-4.70$ & $-3.78$ \\
RESISC45 & $+0.70$ & $-2.80$ & $+0.90$ & $-5.10$ & $-1.58$ \\
SVHN & $-27.60$ & $-13.50$ & $-18.30$ & $-11.50$ & $-17.73$ \\
SUN397 & $-0.00$ & $+0.60$ & $+1.90$ & $+0.30$ & $+0.70$ \\
\cmidrule(lr){1-6}
Mean & $-7.24$ & $-6.13$ & $-6.48$ & $-8.71$ & $-7.14$ \\
\midrule
\multicolumn{6}{l}{\textit{ImageNet}} \\
EuroSAT & $+0.12$ & $-0.15$ & $+0.04$ & $-0.30$ & $-0.07$ \\
DTD & $-0.04$ & $-0.64$ & $+0.20$ & $-1.99$ & $-0.62$ \\
Cars & $+0.91$ & $+0.20$ & $+0.63$ & $+0.33$ & $+0.52$ \\
GTSRB & $-0.17$ & $-0.63$ & $-0.12$ & $-4.64$ & $-1.39$ \\
MNIST & $-0.33$ & $-0.22$ & $-0.06$ & $-0.40$ & $-0.25$ \\
RESISC45 & $+0.38$ & $+0.32$ & $+0.45$ & $-0.64$ & $+0.13$ \\
SVHN & $-0.25$ & $-0.47$ & $-0.46$ & $-1.72$ & $-0.72$ \\
SUN397 & $+0.05$ & $+0.49$ & $+0.60$ & $-0.68$ & $+0.12$ \\
\cmidrule(lr){1-6}
Mean & $+0.08$ & $-0.14$ & $+0.16$ & $-1.25$ & $-0.29$ \\
\bottomrule
\end{tabular}
\end{table}

Two readings follow. Fine-tuning on a single task moves CIFAR-100 by $-7.14$ points on average over the $32$ pairs of a target and a task and by as much as $-41.90$, so the $-1.70$ points that composing two operators costs on the same four targets at $\alpha = 0.5$ sits inside the range single-task adaptation itself occupies rather than outside it. The damage is also concentrated rather than uniform, in that GTSRB accounts for $-30.83$ points and SVHN for $-17.73$ while the remaining six tasks average $-1.43$, so the average is carried by two of the eight rather than spread across them.

The comparison does not favor the operator on every axis. Fine-tuning leaves ImageNet essentially where it found it, at $-0.29$ points on average, whereas the operator gives up $3.92$ points of it on those targets at $\alpha = 0.5$. A weight edit confined to one task's own gradient directions stays local in a way a dense map on the embedding space does not, which is the same structural difference that separates GradFix from the closed form under negation in Appendix~\ref{app:baselines-neg}.

\subsection{Multi-Task Scaling}
\label{app:ksweep}

\begin{figure}[H]
    \centering
    \includegraphics[width=0.55\linewidth]{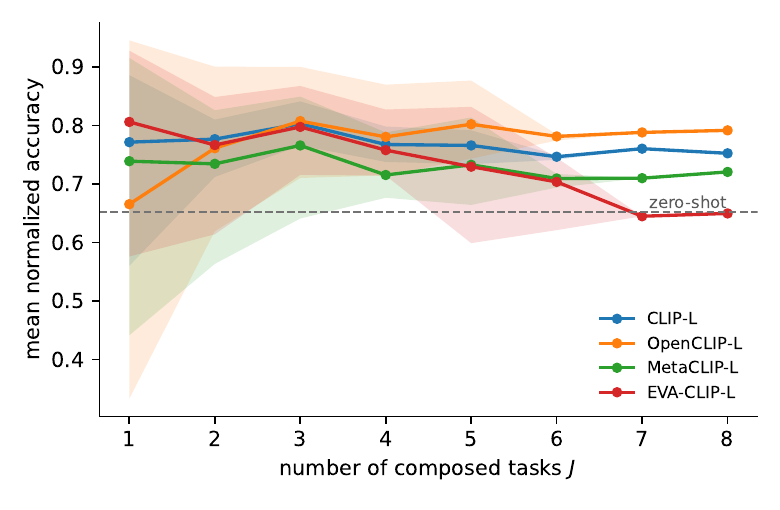}
    \caption{Mean accuracy over the $J$ composed tasks, normalized by the per-task fine-tuning ceiling, as the operators of $J$ tasks accumulate on one target at $\alpha = 0.5$, each measured on the source family the text assigns to its task. The band spans three orderings, and every prefix is the sum of the same stored operators, one per task.}
    \label{fig:ksweep}
\end{figure}

Figure~\ref{fig:ksweep} sweeps $J = 1, \dots, 8$ on each of the four targets that absorb the operator, with one source family per task, using the assignment EuroSAT\,$\to$\,CLIP, DTD\,$\to$\,SigLIP, Cars\,$\to$\,MetaCLIP, GTSRB\,$\to$\,OpenCLIP, MNIST\,$\to$\,SigLIP2, RESISC45\,$\to$\,SigLIP2, SVHN\,$\to$\,EVA-CLIP and SUN397\,$\to$\,SigLIP. Four of these sources are also targets, CLIP for EuroSAT, MetaCLIP for Cars, OpenCLIP for GTSRB and EVA-CLIP for SVHN, so on every target exactly one of the eight terms is same-model. Three fixed permutations drawn at seed $42$ give the band. Because a prefix is a sum of stored operators and matrix addition is order-independent, two orderings that agree on their first $J$ tasks give the same operator, so the visible spread is the spread over which $J$ tasks a prefix contains rather than over the order in which they arrive. Every point of the sweep is a partial sum of the same eight operators, one per task and each fitted on that task's own probes and prompts, so the whole figure is read off one bank without a further solve.

\paragraph{Strength has to be read per task, not per edit.}
The composed operator is a sum of $J$ atomic operators, so its Frobenius norm grows with $J$ and a fixed $\alpha$ raises the effective strength as tasks accumulate. This is visible in the sweep, where at a fixed $\alpha = 0.5$ on the CLIP-L target the mean gain over the composed tasks stays in a band between $+9.78$ and $+19.21$ points across all eight values of $J$, while CIFAR-100 moves from $-1.27$ at $J = 1$ to $-15.30$ at $J = 8$ and ImageNet from $-2.41$ to $-16.23$. Acquisition therefore holds up as tasks accumulate and the price is paid on the preservation axis, which is what the triangle bound of Section~\ref{sec:method-operations}, linear in $J$, predicts for the residual in the worst case. The remedy is the one weight-space task arithmetic applies when it divides the merging coefficient by the number of tasks, and the sweep's runs at $\alpha = 0.5 / J$ show what it restores and what it gives up, since on CLIP-L at $J = 8$ and $\alpha = 0.0625$ CIFAR-100 moves by $+0.50$ points and ImageNet by $-0.32$ while the mean gain on the composed tasks falls to $+4.08$.

\paragraph{What accumulation costs is a property of the target.}
The four targets differ far more in what they spend than in what they acquire. At $J = 8$ and $\alpha = 0.5$ CIFAR-100 has moved by $-3.20$ points on OpenCLIP-L against $-6.30$ on MetaCLIP-L, $-12.10$ on EVA-CLIP-L and $-15.30$ on CLIP-L, while the mean gain on the composed tasks is still $+16.67$, $+9.27$ and $+11.17$ on OpenCLIP-L, MetaCLIP-L and CLIP-L. OpenCLIP-L and MetaCLIP-L retain the most here as they do under negation in Appendix~\ref{app:negation-full}, so retention follows the target rather than the operation. EVA-CLIP-L is the exception on the acquisition axis, since its mean gain falls from $+11.65$ at $J = 1$ to $-4.41$ at $J = 8$ while ImageNet gives up $31.30$ points, so on that target accumulation stops paying well before eight tasks. How far an edit reaches beyond what it was asked to change therefore follows the target at least as much as the number of tasks the edit carries.

\section{Full Negation Results}
\label{app:negation-full}

\begin{table}[H]
\centering
\small
\caption{Negation strength sweep, as change from zero-shot in percentage points, over the $160$ cross-model cells and the $32$ same-model cells that share the same eight tasks and four targets. The zero-shot references are $58.33$ on the removed task, $76.95$ on CIFAR-100 and $71.50$ on ImageNet. Beyond $\lambda = 1$ the controls degrade about as fast as the removed task or faster, which is the operating boundary, and the sweep is continued past the useful range to show what the last points of erasure cost.}
\label{tab:neg-lambda}
\setlength{\tabcolsep}{4pt}
\begin{tabular}{lccccccc}
\toprule
& \multicolumn{3}{c}{Cross-model source} & & \multicolumn{3}{c}{Same-model source} \\
\cmidrule{2-4} \cmidrule{6-8}
$\lambda$ & Removed task $\downarrow$ & CIFAR-100 $\uparrow$ & ImageNet $\uparrow$ & & Removed task $\downarrow$ & CIFAR-100 $\uparrow$ & ImageNet $\uparrow$ \\
\midrule
$0.25$ & $-12.14 \pm 10.02$ & $\phantom{0}-0.11$ & $\phantom{0}-0.50$ & & $-13.98 \pm 10.10$ & $\phantom{0}-0.03$ & $\phantom{0}-0.28$ \\
$0.5$ & $-22.13 \pm 15.63$ & $\phantom{0}-0.94$ & $\phantom{0}-2.23$ & & $-24.75 \pm 14.79$ & $\phantom{0}-0.39$ & $\phantom{0}-1.08$ \\
$0.75$ & $-28.52 \pm 17.65$ & $\phantom{0}-2.69$ & $\phantom{0}-4.99$ & & $-31.74 \pm 16.86$ & $\phantom{0}-1.37$ & $\phantom{0}-2.37$ \\
$1.0$ & $-32.93 \pm 18.45$ & $\phantom{0}-5.07$ & $\phantom{0}-8.41$ & & $-36.26 \pm 17.39$ & $\phantom{0}-2.67$ & $\phantom{0}-4.21$ \\
$1.25$ & $-36.27 \pm 18.64$ & $\phantom{0}-7.90$ & $-12.26$ & & $-39.60 \pm 17.14$ & $\phantom{0}-4.48$ & $\phantom{0}-6.69$ \\
$1.5$ & $-38.93 \pm 18.49$ & $-11.00$ & $-16.32$ & & $-42.37 \pm 16.76$ & $\phantom{0}-6.62$ & $\phantom{0}-9.75$ \\
$2.0$ & $-42.98 \pm 18.16$ & $-18.08$ & $-24.53$ & & $-46.72 \pm 16.72$ & $-11.77$ & $-17.05$ \\
$2.5$ & $-45.80 \pm 17.61$ & $-25.18$ & $-32.35$ & & $-49.25 \pm 17.13$ & $-17.66$ & $-24.95$ \\
$3.0$ & $-47.91 \pm 16.91$ & $-31.70$ & $-39.28$ & & $-50.82 \pm 17.12$ & $-23.94$ & $-32.42$ \\
$3.5$ & $-49.60 \pm 16.25$ & $-37.62$ & $-45.17$ & & $-51.85 \pm 16.88$ & $-29.70$ & $-38.74$ \\
$4.0$ & $-51.00 \pm 15.78$ & $-42.77$ & $-49.95$ & & $-52.64 \pm 16.65$ & $-35.16$ & $-44.14$ \\
\bottomrule
\end{tabular}
\end{table}

Table~\ref{tab:neg-lambda} covers the four targets that absorb the operator, which the fine sweep was run on. Pooled over all six targets, the $240$ cross-model cells of Table~\ref{tab:negation-main}, $\lambda = 0.5$ removes $23.04$ points of the targeted task for $1.00$ of CIFAR-100 and $1.73$ of ImageNet, against $35.22$, $4.55$ and $6.66$ at $\lambda = 1$.

\paragraph{How far negation can be pushed.}
Table~\ref{tab:neg-lambda} continues past $\lambda = 1$ because the returns there are strongly asymmetric and worth recording. Over the cross-model cells the removed task falls from $15.4$ to $7.3$ percent between $\lambda = 2$ and $\lambda = 4$, eight further points of forgetting, while CIFAR-100 falls from $58.9$ to $34.2$ and ImageNet from $47.0$ to $21.6$. The last points of erasure are therefore the most expensive ones, which follows from Equation~\eqref{eq:lsq} carrying no term for what must be preserved. Once the targeted behavior is gone there is nothing in the objective that stops the operator from continuing to move everything else, so the practical setting is the smallest $\lambda$ that reaches the erasure one wants and not the largest one that is stable.

\paragraph{Retention depends on the target, not on the source.}
Pooled over the $160$ cross-model cells of the four targets that absorb the operator at $\lambda = 1$, they differ more in what they retain than in what they erase. OpenCLIP-L gives up $1.9$ points of CIFAR-100 for $27.6$ points of erasure, while CLIP-L, MetaCLIP-L and EVA-CLIP-L give up $5.6$, $4.7$ and $8.1$ points for $34.6$, $31.3$ and $38.2$, so the points forgotten per point of the control range from $4.7$ on EVA-CLIP-L to $14.8$ on OpenCLIP-L. OpenCLIP-L retains the most at every strength we sweep and for every source family, so the effect is a property of the target and of how far the removed task sits from the control rather than of the source assignment.

\section{Sensitivity to the Source Fine-Tuning Seed}
\label{app:seeds}

Every source in the study is one fine-tuning run, so a descriptor, the operator solved from it and the accuracy of every cell that operator enters all inherit whatever the initialization of the source's low-rank update and the order of its training batches happened to be. This section measures how much of a reported number that draw accounts for. We repeat the fine-tuning of the EuroSAT and DTD sources of CLIP-L, OpenCLIP-L and MetaCLIP-L at three random seeds, giving three independent sets of $6$ sources, and run the addition and negation protocols of Section~\ref{sec:exp-addition} and Section~\ref{sec:exp-negation} once per set.

Everything else is held at the values the rest of the paper uses. The probes, the prompt landmarks, the ridge coefficients, the strengths, the evaluation splits and the control protocol are unchanged, each operator is solved on its own task's probes and prompts as in the main study, and the split on which a task is evaluated is drawn at a fixed seed of its own, so it does not move with the fine-tuning seed. The only quantity that varies across a triple is therefore the low-rank update a source fine-tune produced. As a check that the frame is the one the rest of the paper reports, the unedited accuracy of each of the three targets on both tasks and both controls reproduces to the last digit the value measured on the grid from which Table~\ref{tab:addition-main} is built.

\begin{table}[t]
\centering\small
\caption{Repeating the fine-tuning of the EuroSAT and DTD sources at three random seeds moves a composed or negated cell by a few tenths of a point. Entries are percentages, averaged over cells first and then over the three seeds, and the interval is the standard deviation over seeds. Rows marked as the study's own sources repeat the measurement with the checkpoints every other table of the paper reads, in the same frame, and are not part of that mean. Targets are CLIP-L, OpenCLIP-L and MetaCLIP-L, each composing or negating with sources drawn from the other two families.}
\label{tab:seeds}
\setlength{\tabcolsep}{4pt}
\begin{tabular}{@{}lrrrr@{}}
\toprule
Addition at $\alpha = 0.5$ & EuroSAT & DTD & CIFAR-100 & ImageNet \\
\midrule
Zero-shot & $50.31$ & $56.29$ & $73.30$ & $70.22$ \\
Two sources, one per task ($6$ cells) & $85.79 \pm 0.23$ & $65.02 \pm 0.13$ & $72.53 \pm 0.02$ & $69.23 \pm 0.01$ \\
\quad the study's own sources & $85.51$ & $64.92$ & $72.62$ & $69.31$ \\
One source, both tasks ($6$ cells) & $85.57 \pm 0.23$ & $64.96 \pm 0.31$ & $72.53 \pm 0.15$ & $69.36 \pm 0.08$ \\
Target's own sources ($3$ cells) & $90.61 \pm 0.16$ & $69.06 \pm 0.26$ & $72.32 \pm 0.12$ & $69.74 \pm 0.06$ \\
\quad the study's own sources & $90.30$ & $68.74$ & $72.37$ & $69.69$ \\
A single source, its own task ($12$ cells) & $86.01 \pm 0.23$ & $65.89 \pm 0.10$ & $73.01 \pm 0.04$ & $69.73 \pm 0.01$ \\
\bottomrule
\end{tabular}

\vspace{4pt}

\begin{tabular}{@{}lrrr@{}}
\toprule
Negation & Removed task & CIFAR-100 & ImageNet \\
\midrule
Cross-model source, $\lambda = 0.5$ ($12$ cells) & $26.07 \pm 0.22$ & $72.95 \pm 0.01$ & $69.62 \pm 0.01$ \\
\quad the study's own sources & $25.68$ & $72.83$ & $69.66$ \\
Cross-model source, $\lambda = 1$ & $12.83 \pm 0.08$ & $70.94 \pm 0.07$ & $67.75 \pm 0.03$ \\
\quad the study's own sources & $12.27$ & $70.99$ & $67.82$ \\
Same-model source, $\lambda = 1$ ($6$ cells) & $14.03 \pm 0.18$ & $71.06 \pm 0.16$ & $67.79 \pm 0.10$ \\
\bottomrule
\end{tabular}
\end{table}

Table~\ref{tab:seeds} reports the result. Composing two tasks from two other models lifts EuroSAT to $85.79 \pm 0.23$ and DTD to $65.02 \pm 0.13$ percent, where the interval is the standard deviation over the three seeds of the mean over cells, and the two controls move by less than a tenth of a point across seeds. Negation behaves the same way, since at $\lambda = 1$ a cross-model source takes the removed task to $12.83 \pm 0.08$ while CIFAR-100 stands at $70.94 \pm 0.07$. The spread is small against the effects the paper reports, since the composed tasks rise by about $35$ and $9$ points over zero-shot and negation removes about $40$, so a seed accounts for under one percent of the EuroSAT gain and of the erasure and for about one and a half percent of the smaller DTD gain.

Reading the same numbers one cell at a time gives the quantity a single reported entry carries. Taking the standard deviation over the three seeds of one cell's own accuracy and then averaging over cells, addition moves by $0.40$ points on EuroSAT, $0.30$ on DTD, $0.21$ on CIFAR-100 and $0.07$ on ImageNet, and negation by $0.26$ on the removed task, with no cell anywhere in the study exceeding $0.7$. The rows measured with the sources the rest of the paper reads sit within $0.6$ points of the corresponding seeded mean on every axis of both operations, above it on some rows and below it on others, so the sources the paper reads are not a favorable draw.

Two features of the construction explain why the sensitivity is this low. A descriptor records a change in image--prompt affinity rather than a change in parameters, and two fine-tuning runs that reach comparable accuracy on a task move those affinities in nearly the same way even though their low-rank updates barely agree in weight space, since over the $18$ pairs of seeds measured here two seeds' descriptors meet at a cosine between $0.93$ and $0.97$ on the task's own grid while their merged updates meet at only $0.14$ to $0.24$. The solve then reads that descriptor through the target's own probes and prompts and damps the directions it constrains least, so what survives into the operator is the part of the descriptor the seeds agree on.

\section{Reaching a Dataset That Enters No Fit, in Full}
\label{app:holdout-full}

Table~\ref{tab:holdout} of the main text pools the cross-model assignments by the models supplying the two terms. With both terms from one other family, its $k_1 = k_2$ row, the change from zero-shot is $+2.88$ on RESISC45, $+9.80$ on UCMerced and $+4.35$ on AID, against $+2.34$, $+9.04$ and $+3.93$ for the first of the two terms transferred alone, so the composition leads on all three datasets. The expression is $\operatorname{RS}_1 + \operatorname{RS}_2 \to \operatorname{RS}_{\text{target}}$, where $\operatorname{RS}_1$ and $\operatorname{RS}_2$ are the descriptors of two of the three remote-sensing datasets RESISC45, UCMerced and AID and $\operatorname{RS}_{\text{target}}$ is the third, which is evaluated and enters no source term. Each of the two operators is fitted on the probes and prompts of its own dataset alone, as everywhere else in the paper, so the evaluated dataset enters neither fit through its images nor through its class names. Table~\ref{tab:holdout-full} compares the source assignments of the closed form, namely a within-model reference in which both terms come from the target itself, three cross-model assignments in which both terms come from one other family, and a distinct assignment in which the two terms come from the two families other than the target.

\begin{table}[H]
\centering
\small
\caption{Composition evaluated on a dataset that no source term describes, at $\alpha = 0.5$ and pooled over the three targets. Each cross-model row draws both terms from the named family, on the two targets for which that family is not the target; the distinct row draws its two terms from the two families other than the target, in both orders. Single $\operatorname{RS}_1$ transfers the first of the two terms alone, and CIFAR-100 is scored here on its full test split, which is why its zero-shot differs from that of Table~\ref{tab:addition-main}. Cross-model assignments track the within-model reference closely, which is the claim the experiment supports.}
\label{tab:holdout-full}
\setlength{\tabcolsep}{3pt}
\begin{tabular}{llccccc}
\toprule
Assignment & Method & RESISC45 & UCMerced & AID & Mean & CIFAR-100 \\
\midrule
None         & Zero-shot                    & $68.35$ & $74.29$ & $68.58$ & $70.41$ & $73.48$ \\
Within-model & Composition                  & $72.86$ & $86.83$ & $74.48$ & $78.06$ & $72.26$ \\
Within-model & Single $\operatorname{RS}_1$ & $72.20$ & $85.87$ & $73.47$ & $77.18$ & $72.76$ \\
Within-model & Oracle (target's own source) & $87.43$ & $91.59$ & $91.60$ & $90.21$ & $72.93$ \\
Cross-model, CLIP     & Composition         & $71.73$ & $84.76$ & $73.93$ & $76.81$ & $71.20$ \\
Cross-model, MetaCLIP & Composition         & $71.19$ & $85.36$ & $72.28$ & $76.27$ & $74.45$ \\
Cross-model, OpenCLIP & Composition         & $70.76$ & $82.14$ & $72.60$ & $75.17$ & $70.49$ \\
Cross-model, distinct & Composition         & $71.21$ & $84.56$ & $73.63$ & $76.47$ & $72.09$ \\
\bottomrule
\end{tabular}
\end{table}

\paragraph{The control moves by at most three points.}
CIFAR-100 enters no source term and no fit, and moves by at most three points on every row, from $73.48$ at zero-shot to between $70.49$ and $74.45$, so on every row the composition moves the control by less than it gains on the evaluated dataset. This is the same asymmetry the addition results of Appendix~\ref{app:addition-full} show at $\alpha = 0.5$, in a setting where the evaluated dataset is not among the composed tasks either.

\paragraph{Generalization beyond the fitted prompts.}
Each operator is fitted on the landmarks of its own dataset, so the two that a cell sums carry no coordinate indexed by the evaluated dataset's class names, and an edit that merely re-scored the columns it was fitted on would leave that dataset exactly at zero-shot. Every composition row of Table~\ref{tab:holdout-full} is above it instead, by between $4.76$ and $7.65$ points. The operator is therefore moving the representation rather than the scores of the prompts it saw, which is what Equation~\eqref{eq:factorization} and Appendix~\ref{app:underdetermined} predict, since the ridge terms settle the directions the descriptor leaves free by minimum norm rather than by leaving them untouched. The oracle row, in which the target's own source supplies the evaluated dataset's own descriptor, reaches $90.21$, so operators of related datasets are a partial substitute for a task's own probes rather than a replacement for them.

\paragraph{The adapter.}
The adapter fits both terms of a cell jointly on their own probes at full strength, with the recipe of Appendix~\ref{app:adapter}, on the same pool, targets and assignments. Pooled as in Table~\ref{tab:holdout}, it reaches $76.63$ with both terms from one other family, $76.64$ with one term from each of the two others and $75.20$ with the first term alone, against $76.08$, $76.47$ and $75.51$ for the closed form, and $80.16$ within the model against $78.06$. It also leaves CIFAR-100 within $0.3$ points of its own unedited accuracy of $73.50$ on every row, where the closed form gives up between $0.8$ and $1.4$ points, which matches the adapter's lighter touch on ImageNet in Table~\ref{tab:addition-main} and on both controls in Table~\ref{tab:negation-main}. At strength $0.5$, the one the closed form applies to its operator, the adapter reaches $75.47$ with both terms from one other family and $74.26$ with the first term alone, below the closed form's $76.08$ and $75.51$, and $78.87$ within the model, so across families it acquires less at half strength, as it does in addition (Appendix~\ref{app:adapter}).

\section{The Analogy Operation}
\label{app:analogy}

Section~\ref{sec:prelim-weight} lists three operations that weight-space task arithmetic defines, and this paper carries over two of them. The third is absent because of a measurement rather than an omission, which is why Table~\ref{tab:axes} lists the analogy among the operations weight-space task arithmetic defines and not among ours, and this appendix reports that measurement.

\paragraph{What the operation needs, and where the data comes from.}
An analogy relates four tasks forming a product of two domains $D_1, D_2$ and two subtasks $S_1, S_2$, so that three of the four predict the fourth. Writing a task as the pair of its domain and its subtask and taking $(D_2, S_2)$ as the one held out, the descriptor form of the expression is
\begin{equation}
    U_{\text{ana}}
    \;=\;
    \operatorname{UTD}^{(k_1)}_{D_1, S_2}
    \,-\, \lambda\, \operatorname{UTD}^{(k_2)}_{D_1, S_1}
    \,+\, \operatorname{UTD}^{(k_3)}_{D_2, S_1},
    \label{eq:arith-ana}
\end{equation}
with $k_1, k_2, k_3$ the models the three terms are measured on and $\lambda \geq 0$ the strength of the subtraction, whose job is to strip the $D_1$ context off the first term so that the third term can supply $D_2$ in its place. As everywhere else in the paper, our route solves one operator per task on that task's own probes and prompts and sums the three, so the expression is formed on operators and Equation~\eqref{eq:arith-ana} names the composition they realize. None of the eight tasks of Section~\ref{sec:exp-setup} form such a product, since the remote-sensing group is one domain with several label sets and the digit-like group is one label set across several domains, so each supplies a single axis. \citet{ilharco2022editing} build the product by crossing an arbitrary halving of a class list with two rendering styles, photographs and sketches, and we follow that construction on DomainNet~\citep{peng2019moment}, which draws the same $345$ classes in six domains. Both domains then hold the identical label set, so a task's prompts depend on its subtask alone, the subtask axis lives in the landmarks and the domain axis lives in the shift field.

\paragraph{Two products, the second built to make the operation more likely to succeed.}
The first product is the one just described, two halves of the class list under a fixed permutation crossed with photographs and sketches. The second changes both axes in the direction that should help. Its subtask axis is semantic rather than arbitrary, one subtask naming the $115$ classes that are living or occur in nature and the other $115$ manufactured classes drawn at seed $42$, the remaining classes entering no task, so that class count does not ride on the axis; two arbitrary halves of one label set are two samples of one distribution, and a model fine-tuned on either moves the same way. Its domain axis is photographs against paintings rather than sketches, which differ in rendering while keeping photographic content, so the domain term is smaller and the subtask term matters relatively more. Each of the four tasks is a fine-tune of the same kind the rest of the paper uses, and each is held out in turn on each of the six model families, which is what a cell of Table~\ref{tab:analogy} is.

\paragraph{Weight-space task arithmetic is the baseline, in the regime that favors it.}
Every task of a product is a fine-tune of one base model, so a weight-space task vector is defined between any two of them and the operation can be run in the parameter space that introduced it, with no cross-model transport to pay for. We therefore report both routes on identical cells, at the strength Section~\ref{sec:exp-addition} fixes for every route, and score them on one seeded subsample of $1{,}200$ images of the held-out task and $500$ CIFAR-100 images, so that the two differ in nothing but the operation. Alongside the assembled expression we report each of its two positive terms transferred alone, which is the control that decides whether the analogy structure contributes anything, and an oracle in which the held-out task supplies its own update.

\begin{table}[H]
\centering
\small
\caption{The analogy operation on two products of two domains and two subtasks, in percent, with each of the four tasks held out in turn on each of the six model families, giving 24 cells per column. The first product follows the construction of \citet{ilharco2022editing}, crossing an arbitrary halving of one class list with two rendering styles, and the second replaces the halving by a semantic split into natural and manufactured classes and narrows the style gap, both of which should favor the operation. Every route is applied at the strength Section~\ref{sec:exp-addition} fixes. The first and third terms are the two positive terms of the expression taken alone, and $\lambda$ scales the subtracted term, so $\lambda = 0$ is the sum of the two positive terms. CIFAR-100 enters no term and no fit, and its column is identical for the first term, the third term and the oracle because those three rows apply the same four single-task edits and differ only in which held-out task each is paired with.}
\label{tab:analogy}
\setlength{\tabcolsep}{3pt}
\begin{tabular}{lcccc}
\toprule
 & \multicolumn{2}{c}{Arbitrary halves, photographs and sketches} & \multicolumn{2}{c}{Semantic split, photographs and paintings} \\
\cmidrule(lr){2-3} \cmidrule(lr){4-5}
Method & Held-out task $\uparrow$ & CIFAR-100 & Held-out task $\uparrow$ & CIFAR-100 \\
\midrule
Zero-shot & $89.99$ & $76.63$ & $87.88$ & $76.63$ \\
\midrule
\multicolumn{5}{l}{\textit{Weight-space task arithmetic}} \\
First term alone & $90.63$ & $77.49$ & $88.47$ & $77.34$ \\
Third term alone & $90.80$ & $77.49$ & $88.70$ & $77.34$ \\
Analogy, $\lambda = 1$ & $90.90$ & $76.45$ & $88.61$ & $76.45$ \\
\textbf{Analogy, $\lambda = 0$} & $91.05$ & $77.07$ & $88.88$ & $78.15$ \\
Oracle & $92.17$ & $77.49$ & $90.76$ & $77.34$ \\
\midrule
\multicolumn{5}{l}{\textit{Ours}} \\
First term alone & $90.38$ & $77.00$ & $88.19$ & $76.99$ \\
Third term alone & $90.35$ & $77.00$ & $88.39$ & $76.99$ \\
Analogy, $\lambda = 1$ & $90.52$ & $76.42$ & $88.27$ & $76.66$ \\
\textbf{Analogy, $\lambda = 0$} & $90.64$ & $77.35$ & $88.45$ & $76.93$ \\
Oracle & $91.58$ & $77.00$ & $90.05$ & $76.99$ \\
\bottomrule
\end{tabular}
\end{table}

\paragraph{Both routes improve on the unedited target, and a single term is where the improvement is.}
Read the way prior work reads it, against the unedited model, the operation works. Table~\ref{tab:analogy} gives the full expression, at $\lambda = 1$, a gain of $0.91$ points on the first product and $0.73$ on the second in weight space, and of $0.53$ and $0.39$ in ours, which reproduces in direction what \citet{ilharco2022editing} report for their own analogy over subpopulations. Read against the better of its own two positive terms transferred alone, almost none of that is left. Even at $\lambda = 0$, where the expression gains most, the margin paired task by task is $+0.11 \pm 0.33$ and $+0.09 \pm 0.42$ points in weight space and $+0.10 \pm 0.29$ and $-0.02 \pm 0.48$ in ours, over $24$ cells each, so it is a tenth of a point in either direction and smaller than the standard error of a single entry, which is $0.87$ points on a $1{,}200$-image split at these accuracies. The range the operation could occupy is narrow to begin with, since fine-tuning on the held-out task itself stands only $2.18$ and $2.88$ points above zero-shot in weight space, of which the better single term of each cell already recovers $0.95$ and $0.91$ points.

\paragraph{The subtracted term does no work, and strength destroys the operation.}
Setting $\lambda = 0$ deletes the subtraction and leaves the sum of the two positive terms, and it is the better choice on $14$, $15$, $18$ and $16$ of the $24$ cells across the two products and the two routes, so the term whose job is to strip the first term's domain away is at best free and more often harmful. Doubling the strength to $\alpha = 1$ takes the expression below zero-shot on both products and both routes, by between $0.19$ and $1.05$ points, while one term alone still holds a small gain there, so the margin falls to between $-0.59$ and $-1.58$. The operation therefore has no strength at which it is worth running.

\paragraph{What the comparison establishes.}
Weight-space task arithmetic introduced the operation, and every task of a product is a fine-tune of one base model, so it runs here in its own parameter space with no cross-model transport to pay for, which is the strongest form the operation takes. It behaves as our route does on every measurement above, and the second product, built with a semantic subtask axis and a narrower domain gap so that the operation would have more to work with, does not separate the two either. What a product of two domains and two subtasks supplies is therefore the limit rather than the transfer, and the paper states the scope in Section~\ref{sec:conclusion} rather than reporting an operation it cannot demonstrate. The question the analogy was meant to answer, whether an edit reaches beyond the tasks its terms describe, is addressed instead in Section~\ref{sec:exp-holdout}, where both terms are positive, the evaluated dataset appears in neither, and one related term carries most of the reach.

\section{Baseline Details and Results}
\label{app:baselines}

\subsection{What Each Route Requires}

Table~\ref{tab:axes} places the comparisons on the two axes that decide where each of them is defined at all, namely the structural correspondence it needs between the model that learned a task and the model that is to receive it, and the signal it needs on the target side. The methods are ordered by how much correspondence they demand, from weight-space task arithmetic, which requires the two models to be the same model, through the two transport methods, which require corresponding parts, to the two descriptor routes, which require neither and differ in whether the edit is a solved operator on the last layer or a fitted adapter of the whole encoder. Gradient ascent stands outside that ordering, since it demands no correspondence at all and instead retrains the target on labeled examples of the task it removes, and so do ensemble distillation, which trains the target on the outputs of the fine-tuned sources rather than on their weights, and supervised tuning, which trains the target on labels of the composed tasks and so transfers nothing but gives the labeled reference, the last two being measured in Appendix~\ref{app:baselines-kd}. The operations column records which of the three operations of task arithmetic each route has been formulated for, since a route built to merge two models is not always defined for removal or for analogy. The final column is the property the whole paper is about, in that a route either does or does not remain defined when the source and the target come from different families, and the abbreviation \textsc{n/a} marks the combinations that cannot be run rather than ones that were run and failed.

\begin{table}[H]
\centering
\small
\caption{Requirements and scope of the routes to composing task knowledge. Structural requirement is what must correspond between the source model and the target model, target-side signal is what the method consumes from the target beyond its weights, and operations lists which of addition, negation and analogy the route has been formulated for. The last column marks whether the route stays defined when source and target belong to different families. The abbreviation \textsc{n/a} means the method is undefined for that combination.}
\label{tab:axes}
\setlength{\tabcolsep}{3.5pt}
\begin{tabular}{@{}p{0.215\linewidth} p{0.215\linewidth} p{0.215\linewidth} p{0.135\linewidth} p{0.115\linewidth}@{}}
\toprule
Method & Structural requirement & Target-side signal & Operations & Different families \\
\midrule
Weight-space task arithmetic~\citep{ilharco2022editing} & Identical parameter space & None & Addition, negation, analogy & \textsc{n/a} \\
GradFix~\citep{rinaldi2026gradient} & Matching parameter coordinates & Labeled examples and their gradients & Addition, negation & Partly \\
Theseus~\citep{rinaldi2026transporting} & Corresponding submodules, width mismatch allowed & Calibration activations & Addition & Partly \\
\textbf{Ours, closed form} & None beyond a cross-modal affinity and an affine last layer writing into the embedding & Unlabeled probe images and prompts & Addition, negation & Yes \\
\textbf{Ours, adapter} & None beyond a cross-modal affinity & Unlabeled probe images and prompts, one optimization per edit & Addition, negation & Yes \\
\midrule
Gradient ascent~\citep{thudi2022unrolling} & None, since it trains the target directly & Labels for the removed task & Negation & \textsc{n/a} \\
Ensemble distillation~\citep{hinton2015distilling} & None, since it trains the target on the teachers' outputs & Unlabeled probe images and the teachers' outputs on them & Addition & Yes \\
Supervised tuning & None, since it trains the target directly & Labels for every composed task & Addition & \textsc{n/a} \\
\bottomrule
\end{tabular}
\end{table}

\subsection{Protocols}
\label{app:baselines-protocols}

\textbf{Weight-space task arithmetic.} The source is a low-rank fine-tune, so its weight-space task vector is the merged update $\Delta W = (\alpha_{\text{src}}/r_{\text{src}})\, W_{\text{up}} W_{\text{down}}$ per adapted module, with $W_{\text{up}}, W_{\text{down}}$ the two factors of the source's low-rank update and $r_{\text{src}}$, $\alpha_{\text{src}}$ its rank and scaling factor of Appendix~\ref{app:datasets}, which is exactly the delta the fine-tuned model applies. Addition merges the updates $\Delta W_{\tau_1}$ and $\Delta W_{\tau_2}$ of the two composed tasks $\tau_1$ and $\tau_2$ by their mean and evaluates $\theta_{\text{pre}} + \alpha\,(\Delta W_{\tau_1} + \Delta W_{\tau_2})/2$, so that each update enters at half the strength, and negation evaluates $\theta_{\text{pre}} - \lambda \Delta W_\tau$, at $\alpha = 0.5$ and the reported $\lambda$, both fixed in advance rather than selected on the results. It is defined only when the source's parameter space is the target's, so it runs same-model and is reported as the single-model reference.

\textbf{Gradient ascent.} Two budgets are reported. The first follows the standard setting, $200$ training images, $10$ epochs, learning rate $10^{-4}$ and a constant schedule. The second is a larger budget, namely the negation probe set capped at $2{,}500$ images, $30$ epochs, learning rate $3 \times 10^{-4}$ and a cosine schedule, so that the comparison does not hinge on the standard budget being too small. Both appear as separate rows, since neither dominates, the standard budget keeping more of the controls and the larger one erasing more of the targeted task.

\textbf{Control protocol.} CIFAR-100 uses $10$ images per class throughout. ImageNet uses $10$ per class for the weight-space comparisons and for the transport methods' addition cells, and $1$ per class for gradient ascent and the transport methods' negation cells, which keeps every cell of the latter affordable given that a weight edit has to re-run the control for each of them. The smaller draw is a strict subset of the larger one, since both slice the same seeded per-class permutation, so the two estimate the same quantity and the smaller adds about $0.18$ points of noise once pooled over the cells of a row. Each comparison is scored against the zero-shot accuracy measured on its own draw rather than against a shared one.

\textbf{Transport baselines.} We use the released implementations of GradFix and Theseus unmodified, on cells whose two sources and target have matching parameter shapes, and apply them at strengths fixed in advance rather than selected on the results, $0.5$ on the mean of the two transported updates of an addition cell as above, so each enters at $0.25$, and the reported $\lambda$ on the transported update of a negation cell. For Theseus these cells lie outside the setting it was designed for, models of one family at different scales, so the comparison evaluates it with its released code and configuration on transfer between families of one architecture. GradFix consumes labels for the gradients its mask is computed from, which it needs whatever the strength is. Both methods receive a source's low-rank update merged into the dense $\Delta W$ above, and GradFix's coordinate mask does not preserve its low rank, so the update it writes into the target is dense. Theseus is calibrated on one seeded random batch of $32$ images from the training-split probes, the configuration its authors release. Their transported updates come from the same source checkpoints our descriptors are measured on. Negation in particular must be reported at a fixed strength, since a larger strength always removes more of the targeted task and selecting on that quantity would be degenerate.

\subsection{Addition Against the Transport Methods}
\label{app:baselines-add}

The architecture block of Table~\ref{tab:addition-main} is the regime in which every route is defined at once, the $168$ cells in which the two sources and the target are three distinct families of the ViT-L/14 width-1024 group, with task A from one source and task B from the other. Composing in the target's parameter space without any transport leaves the two tasks at zero-shot, which is the null the two transport methods are built to beat, and both do beat it, Theseus~\citep{rinaldi2026transporting}, evaluated here outside the cross-scale setting it was designed for, by less than a point and GradFix~\citep{rinaldi2026gradient} by $13.7$ and $16.9$ points. The closed form leads GradFix by $7.2$ points on the first composed task and trails it by $1.1$ on the second, while GradFix gives up only $0.5$ points of CIFAR-100 and leaves ImageNet $0.2$ points above zero-shot, where the closed form gives up $1.8$ and $3.5$. ImageNet separates the two in the negation comparison below as well, and for the same reason, in that an edit masked to the target's own parameter coordinates leaves a broad benchmark largely untouched where a dense map on the embedding moves prompts that never entered the fit. The adapter, which also edits inside the encoder, leads GradFix on both tasks by $8.2$ and $3.2$ points and stays within $1.1$ and $0.3$ points of zero-shot on the two controls, within $0.6$ points of GradFix on each, so it nearly matches the locality GradFix buys with matched coordinates and labels while needing neither. Every route is applied at one strength fixed in advance and scored on the same splits, the closed form at $\alpha = 0.5$ on the sum of its two operators and the three weight-space routes at $0.5$ on the mean of their two updates as Appendix~\ref{app:baselines-protocols} states, so no route's strength was selected on its results.

\subsection{Negation Against Weight-Space Negation and Gradient Ascent}
\label{app:baselines-neg}

Table~\ref{tab:neg-baselines} places every negation route on the cells of Section~\ref{sec:exp-negation} for which that route is defined and measured, which are the cells of the four targets that absorb the operator for weight-space negation and of the three ViT-L/14 targets among them for the transport methods and gradient ascent, and reports each as a change from the zero-shot accuracy measured on its own control draw. The routes need different things before they are defined, in that weight-space negation needs the removed task's update to live in the target's own parameter space, the two transport methods need a structural correspondence between the two models, gradient ascent needs labeled examples of the removed task and gradient steps on the target, and only the two descriptor routes stay defined when the two models share no structural correspondence. Each block is therefore restricted to the cells its baseline was measured on.

At $\lambda = 0.5$ the same-model descriptor removes $24.8$ points of the targeted task against weight-space negation's $22.0$ while giving up $0.4$ points of CIFAR-100 against $4.2$ and $1.1$ of ImageNet against $0.6$. At $\lambda = 1$ it erases $6.2$ points more than weight-space negation does and still spends a fifth as much CIFAR-100, while its ImageNet cost grows to $4.2$ points against $1.6$, and the cross-model descriptor pays for transport exactly as Section~\ref{sec:exp-negation} describes.

The transport methods separate sharply. Writing a source update into the target's coordinates with no correspondence at all moves nothing, which is the null row, and Theseus removes only $2.5$ points at $\lambda = 1$, so a map built to preserve a model's function under addition does not by itself deliver erasure under negation. GradFix does, removing $35.3$ points at $\lambda = 1$ while leaving ImageNet within $0.8$ points of zero-shot, a pairing of deep erasure with a preserved ImageNet that no other route in the table achieves. What it spends instead is CIFAR-100, $16.3$ points against our $4.8$ at $\lambda = 1$, and the ordering survives at the weaker strength, since at $\lambda = 0.5$ it removes $26.7$ points for $5.4$ of CIFAR-100 where the cross-model descriptor removes $23.5$ for $1.0$. On the same cells the adapter of Table~\ref{tab:negation-main} removes $28.5$ points for $0.6$ of CIFAR-100 and $0.7$ of ImageNet, so it reaches GradFix's preservation of ImageNet without matched coordinates or labels and holds CIFAR-100 as well, at a shallower erasure. The two profiles follow from what each edit is. GradFix keeps only those coordinates of a same-shaped update that agree with gradients computed on labeled examples of the removed task, so its edit is confined to the target's own parameter space and a broad benchmark is largely untouched, whereas the closed form applies a dense map to the embedding and moves prompts that never entered the fit. The price of GradFix's profile is what it requires, matching parameter coordinates and labels for the task being removed, so it is undefined for the great majority of the pairs Section~\ref{sec:exp-negation} covers.

Gradient ascent sits apart from all of them. Maximizing the cross-entropy of the removed task is unbounded above, so it drives the targeted task to near chance, but nothing in the objective distinguishes the removed capability from anything else the model can do, and both controls fall with it. Across three targets and all eight tasks the standard budget takes the removed task from a mean of $56.4$ percent to $7.7$, and takes CIFAR-100 with it to a mean of $26.6$ and ImageNet to $55.7$; the larger budget reaches $4.1$ on the removed task and leaves $3.6$ of CIFAR-100 and $28.1$ of ImageNet. Negating a descriptor cannot behave this way, because the object it removes is the change that task adaptation introduced rather than the model's ability to score the task at all.

\begin{table}[H]
\centering
\small
\caption{Negation against every comparison, as change from the target's zero-shot in percentage points, with each route read against the zero-shot of its own control draw. A baseline is defined only on part of this population, so each block is restricted to the cells its baseline was measured on. In the transport block the cross-model row is restricted to the same source assignments as well; in the other two the baseline has no source axis, and the cross-model row there averages over every source family available for the cell. Weight-space negation and the same-model closed form require source and target to coincide, gradient ascent requires labels for the removed task and runs at a fixed budget rather than a strength, so a dash stands in its strength column, and the transport methods require corresponding parameters, so only the cross-model row is defined between arbitrary families.}
\label{tab:neg-baselines}
\begin{tabular}{lcccc}
\toprule
Method & Strength & $\Delta$ Removed task $\downarrow$ & $\Delta$ CIFAR-100 $\uparrow$ & $\Delta$ ImageNet $\uparrow$ \\
\midrule
\multicolumn{5}{l}{\textit{Against weight-space negation} ($32$ cells)} \\
Weight-space negation & $\lambda = 0.5$ & $-22.01 \pm 13.04$ & $-4.16$ & $-0.59$ \\
Weight-space negation & $\lambda = 1.0$ & $-30.10 \pm 12.25$ & $-13.44$ & $-1.56$ \\
Weight-space negation & $\lambda = 1.5$ & $-35.69 \pm 11.91$ & $-22.68$ & $-2.88$ \\
Ours, closed form, same model & $\lambda = 0.5$ & $-24.75 \pm 14.79$ & $-0.39$ & $-1.08$ \\
Ours, closed form, same model & $\lambda = 1.0$ & $-36.26 \pm 17.39$ & $-2.67$ & $-4.21$ \\
Ours, closed form, same model & $\lambda = 1.5$ & $-42.37 \pm 16.76$ & $-6.62$ & $-9.75$ \\
Ours, closed form, cross-model & $\lambda = 0.5$ & $-22.13 \pm 15.63$ & $-0.94$ & $-2.23$ \\
Ours, closed form, cross-model & $\lambda = 1.0$ & $-32.93 \pm 18.45$ & $-5.07$ & $-8.41$ \\
Ours, closed form, cross-model & $\lambda = 1.5$ & $-38.93 \pm 18.49$ & $-11.00$ & $-16.32$ \\
\midrule
\multicolumn{5}{l}{\textit{Against the transport methods} ($48$ cells)} \\
Weight-space transport & $\lambda = 0.5$ & $+0.00 \pm 0.29$ & $+0.00$ & $-0.00$ \\
Weight-space transport & $\lambda = 1.0$ & $+0.02 \pm 0.54$ & $+0.02$ & $+0.00$ \\
Weight-space transport & $\lambda = 1.5$ & $+0.01 \pm 0.75$ & $-0.03$ & $-0.00$ \\
Theseus & $\lambda = 0.5$ & $-0.93 \pm 1.01$ & $-0.27$ & $-0.09$ \\
Theseus & $\lambda = 1.0$ & $-2.50 \pm 2.53$ & $-1.49$ & $-0.19$ \\
Theseus & $\lambda = 1.5$ & $-4.68 \pm 4.87$ & $-3.61$ & $-0.38$ \\
GradFix & $\lambda = 0.5$ & $-26.70 \pm 12.61$ & $-5.37$ & $-0.32$ \\
GradFix & $\lambda = 1.0$ & $-35.29 \pm 15.14$ & $-16.34$ & $-0.77$ \\
GradFix & $\lambda = 1.5$ & $-40.08 \pm 16.83$ & $-26.58$ & $-1.57$ \\
Ours, closed form, same model & $\lambda = 0.5$ & $-23.70 \pm 13.89$ & $-0.44$ & $-1.06$ \\
Ours, closed form, same model & $\lambda = 1.0$ & $-34.31 \pm 16.34$ & $-2.83$ & $-4.04$ \\
Ours, closed form, same model & $\lambda = 1.5$ & $-40.26 \pm 16.53$ & $-6.44$ & $-9.22$ \\
Ours, closed form, cross-model & $\lambda = 0.5$ & $-23.48 \pm 14.97$ & $-0.96$ & $-1.79$ \\
Ours, closed form, cross-model & $\lambda = 1.0$ & $-33.62 \pm 17.58$ & $-4.77$ & $-7.45$ \\
Ours, closed form, cross-model & $\lambda = 1.5$ & $-39.41 \pm 17.55$ & $-10.74$ & $-15.83$ \\
\midrule
\multicolumn{5}{l}{\textit{Against gradient ascent} ($24$ cells)} \\
Gradient ascent, standard budget & --- & $-48.67 \pm 19.72$ & $-46.66$ & $-15.02$ \\
Gradient ascent, larger budget & --- & $-52.31 \pm 18.90$ & $-69.66$ & $-42.60$ \\
Ours, closed form, same model & $\lambda = 0.5$ & $-23.70 \pm 13.89$ & $-0.44$ & $-1.06$ \\
Ours, closed form, same model & $\lambda = 1.0$ & $-34.31 \pm 16.34$ & $-2.83$ & $-4.04$ \\
Ours, closed form, same model & $\lambda = 1.5$ & $-40.26 \pm 16.53$ & $-6.44$ & $-9.22$ \\
Ours, closed form, cross-model & $\lambda = 0.5$ & $-20.90 \pm 14.68$ & $-0.76$ & $-1.49$ \\
Ours, closed form, cross-model & $\lambda = 1.0$ & $-31.17 \pm 17.53$ & $-4.05$ & $-6.10$ \\
Ours, closed form, cross-model & $\lambda = 1.5$ & $-36.95 \pm 17.87$ & $-9.06$ & $-13.00$ \\
\bottomrule
\end{tabular}
\end{table}

\begin{table}[H]
\centering
\small
\caption{Per-task fine-tuning ceilings, meaning the target's own rank-4 source applied to itself, in percent. These are the denominators for the normalized axes in Figure~\ref{fig:ksweep}.}
\label{tab:ft-ceilings}
\setlength{\tabcolsep}{4pt}
\begin{tabular}{lcccccccc}
\toprule
Target & EuroSAT & DTD & Cars & GTSRB & MNIST & RESISC45 & SVHN & SUN397 \\
\midrule
CLIP-L & $98.3$ & $75.1$ & $85.4$ & $95.1$ & $98.9$ & $94.9$ & $96.2$ & $76.4$ \\
OpenCLIP-L & $98.2$ & $77.1$ & $93.9$ & $96.0$ & $99.1$ & $94.9$ & $96.6$ & $77.1$ \\
MetaCLIP-L & $98.7$ & $79.4$ & $89.5$ & $95.6$ & $99.0$ & $95.3$ & $96.6$ & $77.0$ \\
EVA-CLIP-L & $98.0$ & $78.6$ & $92.7$ & $95.2$ & $99.2$ & $95.0$ & $97.2$ & $79.7$ \\
\bottomrule
\end{tabular}
\end{table}

The addition counterpart of the weight-space comparison is the same-model block of Table~\ref{tab:addition-main}, at the strength every route uses. Table~\ref{tab:ft-ceilings} gives the accuracy each target reaches when its own source for a task is applied to it, which is the per-task fine-tuning ceiling any edit is measured against and the denominator of the normalized axes in Figure~\ref{fig:ksweep}.

\subsection{Addition and Negation Against Distillation and Supervised Tuning}
\label{app:baselines-kd}

Distilling the fine-tuned sources into the target is the closest alternative to our transfer that also needs no correspondence between the two models, and tuning the target directly on the probes' labels is the labeled reference any transfer is read against. Both are trained with the adapter realization's own loop, so the adapter, its initialization, the optimizer, the step budget, the probe images and the evaluation are those of Appendix~\ref{app:adapter} and only the loss differs. The affinity variant regresses the student's cosine to each prompt of a probe's own task onto the fine-tuned teacher's cosine, the posterior variant matches the student's posterior over those prompts to the teacher's at temperature $2$~\citep{hinton2015distilling}, and supervised tuning minimizes the cross-entropy of the probes' labels, the only route here that reads a label. The teachers are the source fine-tunes every descriptor is read from. The negation analogues regress onto the negated teacher cosine, or match the posterior of the negated teacher logits, on the removed task's own probes. The cells are the six of the architecture block that compose EuroSAT and DTD from two distinct sources onto a third ViT-L/14 family, and for negation the twelve in which one of those two tasks is removed with a source from another ViT-L/14 family.

Distillation and supervised tuning win on the composed tasks and lose everywhere else (Table~\ref{tab:kd}). Posterior distillation reaches $96.36$ and $77.62$ and supervised tuning $95.00$ and $81.45$, against $92.45$ and $64.41$ for the adapter and $85.51$ and $64.92$ for the closed form, which is what copying a decision function or reading labels buys. The price is paid on the control, where the two distillation variants give up $6.75$ and $5.27$ points of CIFAR-100 and supervised tuning $2.70$, against $0.47$ for the adapter and $0.68$ for the closed form, although no term of any cell involves that dataset. Supervised tuning nonetheless preserves far better than distillation, giving up $2.70$ points of CIFAR-100 and $1.12$ of ImageNet, so preservation is not unique to descriptor transfer, and what separates the two descriptor routes from it is that they give up less while reading no label. A distillation target is an absolute similarity or posterior and pulls the student onto the teacher's scale on every prompt it is shown, whereas a descriptor is a difference from the target's own baseline, so the fit asks the target to move only where fine-tuning moved the source; for the closed form the distinction is exact, since fitting a source's absolute affinities rather than their change adds to the descriptor the gap between the source's and the target's pre-trained affinities, so by Proposition~\ref{prop:linearity} the operator becomes ours plus the operator of that gap, which records how the two pre-trained models already disagree and vanishes only when source and target coincide.

Negation has no working distillation analogue. Both variants take the removed task to at most $2.4$ percent, and CIFAR-100 falls with it from $73.30$ to $7.40$ and $14.29$, so the student is no longer a usable model, whereas at $\lambda = 1$ the closed form removes $41.0$ points for $2.3$ of CIFAR-100 and the adapter $33.4$ for $1.0$. The negative of a similarity matrix or of a posterior is not the behavior of any model, so a student that matches it has no reason to keep anything else in place, whereas the negative of a descriptor is a displacement of the same size in the opposite direction and is a statement about the removed task alone.

\begin{table}[H]
\centering
\small
\caption{Distillation and supervised tuning against both realizations, in percent, on the six EuroSAT and DTD cells of the architecture block for addition and the twelve cells removing one of the two tasks for negation. Every trained row uses the adapter realization's recipe and differs from it only in the loss. The closed form is applied at $\alpha = 0.5$ and every trained route at full strength; Oracle uses the target's own descriptors.}
\label{tab:kd}
\setlength{\tabcolsep}{4pt}
\begin{tabular}{lcccc}
\toprule
Addition & EuroSAT $\uparrow$ & DTD $\uparrow$ & CIFAR-100 & ImageNet \\
\midrule
Zero-shot                          & $50.31$ & $56.29$ & $73.30$ & $70.22$ \\
\textbf{Ours, closed form}         & $85.51$ & $64.92$ & $72.62$ & $69.31$ \\
\textbf{Ours, adapter}             & $92.45$ & $64.41$ & $72.83$ & $70.04$ \\
Oracle, closed form                & $90.30$ & $68.74$ & $72.37$ & $69.69$ \\
Affinity distillation              & $95.96$ & $71.10$ & $66.55$ & $67.58$ \\
Posterior distillation             & $96.36$ & $77.62$ & $68.03$ & $67.04$ \\
Supervised tuning                  & $95.00$ & $81.45$ & $70.60$ & $69.10$ \\
\midrule
Negation & Removed $\downarrow$ & & CIFAR-100 & ImageNet \\
\midrule
Zero-shot                          & $53.30$ & & $73.30$ & $70.22$ \\
\textbf{Ours, closed form}         & $12.27$ & & $70.99$ & $67.82$ \\
\textbf{Ours, adapter}             & $19.95$ & & $72.29$ & $69.18$ \\
Affinity distillation              & $2.39$ & & $7.40$ & $33.98$ \\
Posterior distillation             & $0.00$ & & $14.29$ & $35.71$ \\
\bottomrule
\end{tabular}

\end{table}

\section{Targets That Pool Through a Residual Head}
\label{app:head}

SigLIP and SigLIP2 produce the image embedding as $z(x) = a(x) + W_2 u(x) + b_2$, where $a(x)$ is the output of an attention pooling step, $u(x)$ is the activation the head's second affine consumes, and a layer normalization sits between $a(x)$ and $u(x)$. Left-multiplying $W_2$ and $b_2$ by $I + \alpha A$ would edit one branch of the sum and leave $a(x)$ as it was, so Equation~\eqref{eq:absorb} is unavailable. What remains available is the pair itself. With every parameter upstream held fixed, neither $a(x)$ nor $u(x)$ depends on $W_2$ or $b_2$, so any update to them changes the embedding by exactly $\Delta W_2 u(x) + \Delta b_2$, with no approximation. Since a descriptor is read on unit embeddings, we write $f_{\text{pre}}(x) = z(x) / \lVert z(x) \rVert$ and collect the rows $[u(x_i); 1] / \lVert z(x_i) \rVert$ into a matrix, the trailing entry carrying the bias. Equation~\eqref{eq:lsq} with $F_{\text{pre}}$ replaced by that matrix then returns the update directly, and because that matrix and $G$ are fixed for a target the solution map stays linear, so a bank of per-task updates composes by addition exactly as a bank of operators does.

Two checks make the construction usable rather than merely definable. Writing the solved update into a checkpoint and running a real forward pass reproduces the accuracy computed from the cached quantities to a maximum absolute difference of $1.2 \times 10^{-7}$ on the embeddings and to identical top-1 predictions, so the deployed model is the scored model. Linearity of the solution map holds to a relative difference of $6.9 \times 10^{-8}$ on a mixed-sign combination. The edit adds no layer and no parameter, and the head's own weights are the only ones touched.

\begin{table}[H]
\centering
\small
\caption{The two residual-head families as targets, in percent, resolving the cells that Table~\ref{tab:addition-main} and Table~\ref{tab:negation-main} pool with the other four. Addition is averaged over the $560$ cross-model cells of each target, in which neither source is the target and the two sources differ, and negation over $40$ such cells. The head edit is the deployable route and is the one those tables report; the embedding operator is the same solve applied outside the model, which no weight edit of this architecture realizes and which is given only as the reference the head edit is read against. Random-UTD here is a structureless update for one task scored on that task alone, and is therefore not the pair-level control of Table~\ref{tab:addition-main}.}
\label{tab:head}
\setlength{\tabcolsep}{4pt}
\begin{tabular}{llcccc}
\toprule
Target & Method & Task~A & Task~B & CIFAR-100 & ImageNet \\
\midrule
SigLIP-L  & Zero-shot                 & $64.36$ & $67.42$ & $72.00$ & $76.47$ \\
          & Random-UTD                & $65.35$ & --      & $70.76$ & $75.75$ \\
          & Head edit                 & $80.79$ & $74.49$ & $70.38$ & $74.99$ \\
          & Embedding operator        & $81.28$ & $75.67$ & $70.80$ & $75.25$ \\
          & Oracle                    & $84.53$ & $77.80$ & $71.44$ & $76.24$ \\
\midrule
SigLIP2-L & Zero-shot                 & $66.70$ & $66.62$ & $75.10$ & $71.52$ \\
          & Random-UTD                & $66.05$ & --      & $72.93$ & $70.18$ \\
          & Head edit                 & $83.25$ & $75.92$ & $72.66$ & $69.17$ \\
          & Embedding operator        & $83.77$ & $76.56$ & $73.11$ & $69.77$ \\
          & Oracle                    & $85.03$ & $77.94$ & $73.94$ & $71.35$ \\
\midrule
          &                           & Removed & Zero-shot & CIFAR-100 & ImageNet \\
\midrule
SigLIP-L  & Head edit, $\lambda = 0.5$ & $41.54$ & $65.89$ & $71.39$ & $75.82$ \\
          & Head edit, $\lambda = 1$   & $26.40$ & $65.89$ & $69.51$ & $73.59$ \\
SigLIP2-L & Head edit, $\lambda = 0.5$ & $41.33$ & $66.66$ & $73.46$ & $70.69$ \\
          & Head edit, $\lambda = 1$   & $26.55$ & $66.66$ & $70.59$ & $68.05$ \\
\bottomrule
\end{tabular}
\end{table}

Table~\ref{tab:head} resolves the two targets at the strengths the rest of the paper uses, since the main tables pool them with the other four. Addition at $\alpha = 0.5$ lifts the two composed tasks by $16.43$ and $7.07$ points on SigLIP-L and by $16.55$ and $9.30$ on SigLIP2-L, for at most $2.5$ points of either control, and a structureless update of the same norm, which is scored on one task at a time and so has no second column, leaves that task within one point of zero-shot, so what the edit carries is descriptor structure. Negation at $\lambda = 1$ removes $39.5$ and $40.1$ points of the targeted task for $2.5$ and $4.5$ of CIFAR-100. Constraining the edit to the head is not free, and the direction of the cost is consistent rather than large, since the head edit trails the same solve applied outside the model by $0.5$ to $1.2$ points on the composed tasks and gives up $0.3$ to $0.6$ more of each control. Both gaps are small next to the distance to Oracle, so what limits these targets is the cross-model residual of Equation~\eqref{eq:central-assumption} and not the restriction to weights the model already has. The adapter needs no such construction, since it edits the encoder below the head. Paired per cell over the $28$ cells of each target that it covers, it leads the head edit on the composed tasks by $1.42$ points on SigLIP-L and $1.88$ on SigLIP2-L and on ImageNet by $1.50$ and $2.11$, while on CIFAR-100 it is $0.19$ points ahead on SigLIP-L and $3.10$ behind on SigLIP2-L.

\section{Targets at Another Scale}
\label{app:scale}

Nothing in the construction ties a target to the scale its descriptors were measured on. A descriptor is $N \times M$ and carries no trace of the embedding dimension it came from, and Equation~\eqref{eq:lsq} reads the target only through $F_{\text{pre}}$ and $G$, which come from one forward pass over the same probe images and prompts every other model reads. A target at another scale therefore needs no fine-tuning of its own and no source at its scale, and the $48$ Large descriptors are reused verbatim, each solved on its own task's probes and prompts as before.

Table~\ref{tab:scale} reports four such targets, two at $512$ dimensions and two at $1024$, against sources that all embed in $768$ or $1024$. Every cell is cross-model and cross-scale at once. Addition at the fixed $\alpha = 0.5$ gains between $11.38$ and $16.83$ points on the composed tasks, against the $16.13$ the Large targets reach, and the largest gain goes to a target that knew little, namely $+16.83$ on the one whose composed-task zero-shot is $50.28$ percent against $+11.38$ on the one that starts at $70.31$. In absolute terms the ordering is the plain one, at $65.74$ and $67.11$ percent for the two Base targets against $74.46$ at Large and $80.35$ and $81.69$ for the two Huge ones, so the strength fixed at $768$ dimensions still acquires at $512$ and at $1024$. Negation behaves the same way, removing between $20.80$ and $30.19$ points at $\lambda = 0.5$ against the $22.13$ of the Large targets.

What moves with the target's embedding dimension is the price rather than the gain. At $\alpha = 0.5$ the two $1024$-dimensional targets give up $1.27$ and $1.10$ points of CIFAR-100 against the $1.70$ of the $768$-dimensional ones and the $4.61$ and $6.99$ of the two at $512$, and ImageNet orders the same way at $0.94$ and $2.00$ against $3.92$ and then $9.38$ and $13.80$. Negation repeats the ordering, spending $0.63$ and $0.41$ points of CIFAR-100 at $\lambda = 0.5$ for the Huge targets against $0.94$ at Large and $2.75$ and $3.74$ at Base. A narrower embedding leaves the operator fewer directions in which to meet the descriptor, so more of what it does falls on behavior no edit was asked to change, and a $512$-dimensional target is better served by a smaller strength than the one fixed here.

Two absences are structural rather than omissions. There is no oracle row and no same-model row, since both require a source fine-tune on the target itself and none exists at these scales, which is also why every cell is genuinely cross-model. What the table does not separate is scale from model, since the two Base targets are also the two weakest zero-shot models in it, so what it establishes is that a target at another scale needs nothing retrained rather than that dimension alone accounts for the spread, which Table~\ref{tab:add-per-target} shows to be substantial among targets of one scale.

\begin{table}[H]
\centering
\small
\caption{Targets at another scale, in percentage points of change from the target's own zero-shot. Addition is at $\alpha = 0.5$ over the $840$ cells each target admits, and negation at $\lambda = 0.5$ and $1$ over $48$. Every operator behind this table is solved on its own task's grid and stored, as everywhere else in the paper, so the Large reference row is the cross-model cells of Table~\ref{tab:addition-main} and Table~\ref{tab:negation-main} taken over the four Large targets that absorb the operator, which is the only regime available at these scales. Every cell here is cross-model and cross-scale at once, and no oracle or same-model row is possible, since both would need a source fine-tune on the target itself.}
\label{tab:scale}
\setlength{\tabcolsep}{3pt}
\scriptsize
\begin{tabular}{@{}lcccccccc@{}}
\toprule
& & \multicolumn{3}{c}{Addition at $\alpha = 0.5$} & \multicolumn{2}{c}{Negation at $\lambda = 0.5$} & \multicolumn{2}{c}{Negation at $\lambda = 1$} \\
\cmidrule(lr){3-5} \cmidrule(lr){6-7} \cmidrule(lr){8-9}
Target & $d$ & Composed $\uparrow$ & CIFAR-100 & ImageNet & Removed $\downarrow$ & CIFAR-100 & Removed $\downarrow$ & CIFAR-100 \\
\midrule
CLIP-B & $512$ & $+15.16$ & $-6.99$ & $-13.80$ & $-25.89$ & $-3.74$ & $-33.93$ & $-13.64$ \\
MetaCLIP-B & $512$ & $+16.83$ & $-4.61$ & $-9.38$ & $-22.30$ & $-2.75$ & $-30.21$ & $-9.54$ \\
OpenCLIP-H & $1024$ & $+15.34$ & $-1.27$ & $-0.94$ & $-20.80$ & $-0.63$ & $-33.36$ & $-1.64$ \\
MetaCLIP-H & $1024$ & $+11.38$ & $-1.10$ & $-2.00$ & $-30.19$ & $-0.41$ & $-43.31$ & $-2.45$ \\
\midrule
Large targets, reference & $768$ & $+16.13$ & $-1.70$ & $-3.92$ & $-22.13$ & $-0.94$ & $-32.93$ & $-5.07$ \\
\bottomrule
\end{tabular}
\end{table}

\section{Ablations and Negative Results}
\label{app:ablations}

We report the design choices we tested and rejected, because each rules out a plausible alternative to the construction in Section~\ref{sec:method-operator}. Each is measured on the grid the paper solves on, namely the probes and prompts of the one task a descriptor describes, so the accuracies below are comparable to those of the main text, and unless a paragraph states otherwise they pool the four targets that absorb the operator and the five cross-model sources of each. The whitening comparison is the one that cannot be measured that way, because it compares descriptors rather than fitting them and a comparison across tasks needs one frame for all of them.

\paragraph{Transporting the full shift field is worse than using the descriptor.}
A descriptor constrains the operator only inside the span of its task's landmarks, so a natural idea is to transport the whole shift field $\Delta F$ from source to target under a Procrustes map fitted between their probe embeddings, which constrains as many directions as the embedding has rather than as many as the task has prompts. Every model encodes the same probe images, so the two feature matrices are row-paired and this is the setting Procrustes is defined for, yet at $\alpha = 0.5$ it reaches $65.46$ on the transferred task against $76.01$ for the closed form on the same cells and holds less of CIFAR-100 as well, $74.75$ against $76.23$. Row-paired alignment between models is therefore not the missing information. What makes the result sharp is that Procrustes reproduces the target's own shift field better than the closed form does, at a cosine of $0.358$ against $0.199$ on the probes themselves, and is nonetheless ten points worse on the task, so reproducing the shift field is not the objective that predicts accuracy. Appendix~\ref{app:transport-error} reaches the same conclusion from the other direction.

\paragraph{Prompt-Gram whitening cannot help the fit and hurts the comparison.}
The rows of $G$ are not orthogonal, since two class names can be nearly synonymous, so an affinity vector holds coordinates in an oblique frame belonging to the model's own text encoder, and the obvious repair is to express descriptors in an orthonormal basis of $\operatorname{span}(G)$ by right-multiplying each of them by $(G G^\top)^{-1/2}$. It cannot help the transfer, because Equation~\eqref{eq:lsq} is invariant to an invertible change of the landmark basis up to its prompt ridge $\mu_b$, so fitting the whitened descriptor against the whitened landmarks returns the same operator to within that ridge, and over this appendix's $160$ cells at both strengths the two agree to two hundredths of a point of accuracy. Whitening can therefore only matter where descriptors are compared rather than fitted, and there it is harmful. Measured on the probe set spanning the eight tasks, which a cross-task comparison requires, it lowers same-task cross-model agreement from $0.47$ to $0.10$ while lowering the different-task control only from $0.23$ to $0.15$, so the gap that makes descriptors discriminative falls from $0.25$ to $-0.05$, and no truncation of the whitener beats the raw frame. The reason is that the obliquity of the prompt frame is itself shared across models. After centering, which removes the component every embedding of these models has in common, the six families' prompt Grams still agree at cosine $0.883$--$0.972$ against $0.038$--$0.047$ under a class-name permutation null, so the frame carries semantic structure that the models share, and whitening deletes the very reference that makes cross-model comparison work.

\paragraph{Aligning models through text does not help.}
The paper uses text only to measure a change. Using it to align two models instead is the natural alternative, and it is worse. We fit a map between two models' text spaces on a $12{,}000$-word task-agnostic vocabulary taken from the tokenizer's own vocabulary, so that the alignment is unrelated to any task's class names and has enough words to determine a full-rank map, and carry the shift field across with it. At $\alpha = 0.5$ it reaches $68.17$ on the transferred task against $76.01$ for the closed form, and reproduces the target's own shift field worst of the three at a cosine of $0.161$. Relaxing orthogonality fits held-out words considerably better, at a reconstruction error of $0.477$ against $0.576$ for the orthogonal map, so the alignment itself is not the difficulty; fitting a text space and transporting a change through it are different objectives.

\paragraph{Constraining the operator to preserve a subspace does not buy a better trade.}
To protect unrelated behavior one can restrict the edit to act trivially on the $q$ leading singular directions of the target's features on a generic corpus, collecting those directions as the columns of $Q \in \mathbb{R}^{d \times q}$ and constraining the operator to $A = (I - QQ^\top) A' (I - QQ^\top)$ for an unconstrained $A'$. The corpus is Food-101 and Oxford Pets rather than either control, so CIFAR-100 stays a held-out measure of whether protecting one generic corpus generalizes to another. Sweeping $q$ from $0$ to $512$, the unconstrained solve dominates every constrained one on both axes at once rather than only after matching, since at $\alpha = 0.5$ it reaches $76.01$ on the transferred task while holding CIFAR-100 at $76.23$, against $75.80$ and $75.98$ at $q = 64$ and $72.87$ and $72.60$ at $q = 512$. The reason is geometric, since the leading direction alone carries between $59$ and $68$ percent of the feature energy across the four targets, so what to change and what to keep are not separated by any subspace of the embedding.

\paragraph{Fitting the exact first-order objective generalizes worse than fitting the approximate one.}
Equation~\eqref{eq:lsq} matches $F_{\text{pre}} X G^\top$ to the descriptor, which treats the edited embedding as if it were used unnormalized. The model renormalizes, so writing $f_i \in \mathbb{R}^{d}$ for the target's unit-norm embedding of probe $i$, $g_j \in \mathbb{R}^{d}$ for prompt landmark $j$ and $U_{\mathrm{edit}}$ for the descriptor the edit actually induces, that descriptor is
\begin{equation}
    [U_{\mathrm{edit}}]_{ij} \;=\; \alpha\, f_i^\top A^\top g_j \;-\; \alpha\, \bigl( f_i^\top A^\top f_i \bigr) \bigl( f_i^\top g_j \bigr) \;+\; O(\alpha^2 \|A\|^2),
    \label{eq:gauge}
\end{equation}
where the second term is what renormalization removes along the radial direction and is not negligible, since its norm is a tenth of the first term's on these grids. We fitted Equation~\eqref{eq:gauge} directly with a matrix-free conjugate-gradient least-squares solver, the normal-equations matrix being out of reach at $589{,}824 \times 589{,}824$ entries for a $768$-dimensional target, and against an uncorrected control solved by the same solver so that solver effects and objective effects stay separate. The solve is iterative and in double precision, so this study covers one target, two sources and the six tasks whose grids are small enough to afford it, rather than the full population of the paragraphs above. That control behaves as it should, reaching $78.77$ on the transferred task and $73.48$ on CIFAR-100 at $\alpha = 1$ against the closed form's $78.68$ and $72.62$, so what follows is a property of the objective and not of the solver.

Modeling the term does what it should to the residual, lowering it from $0.2392$ to $0.2001$ at the ridge the paper uses, and it is paid for on the control. At $\alpha = 1$ the corrected solve reaches $80.05$ on the transferred task for $67.08$ of CIFAR-100 against the closed form's $78.68$ and $72.62$, so it buys one and a half points of the task with five and a half of the control. Read at a comparable control it is behind rather than ahead, since the closed form reaches $78.68$ while holding $72.62$ whereas the corrected solve reaches only $74.85$ while holding $72.80$. Weakening the ridge makes the pattern unmistakable, because at $10^{-6}$ the corrected solve drives the residual to $0.0725$ and CIFAR-100 to $48.75$ where the uncorrected control at the same ridge holds $69.93$. The radial channel is an extra degree of freedom that a task's probe grid does not pin down, so fitting it moves every prompt that did not enter the grid, and no ridge in $\{10^{-6}, 10^{-4}, 10^{-2}\}$ recovers the loss. Constraining the operator to be antisymmetric, which makes the second term of Equation~\eqref{eq:gauge} vanish identically, is dominated everywhere instead, reaching $78.27$ and $71.73$ at $\alpha = 1$ against the closed form's $78.68$ and $72.62$, and it is not nearly free either, since $f^\top A f = 0$ also removes tangential capacity and the unconstrained solutions are symmetric in $0.707$ to $0.728$ of their Frobenius norm, which is what a matrix with no preference either way would show. The approximate objective is thus doing the work of a regularizer, and we keep it.

\paragraph{Probe count matters, and the quota is not yet in the saturated region.}
Probes carry no labels into the fit, so how many of them a task can supply is bounded by the data the task offers rather than by supervision. Every fit in this sweep holds a fifth of the task's probes out so that the reconstruction criterion has something to score, which is why its full-pool point sits a little below the accuracy the same route reaches above. Subsampling each task's own pool to an eighth of what remains, from an average of $2{,}274$ images to $284$, costs $5.7$ points on the transferred task at $\alpha = 0.5$, from $75.69$ to $69.99$, while leaving CIFAR-100 within $0.2$ points of where it was. The returns diminish but have not stopped, since successive doublings of the pool are worth $2.5$, $1.9$ and $1.3$ points, so the quota of thirty images per class sits on the slope of that curve rather than at its top and a larger quota would still buy a little. Held-out reconstruction barely moves over the same range. Measured here against the shift field the target's own fine-tune induces, relative to its norm, rather than against the source's descriptor as in Table~\ref{tab:rank}, it goes from $1.196$ to $1.164$, so reconstruction is nearly blind to probe count and cannot be used to choose it.

\section{Audio--Text Instantiation}
\label{app:audio}

We use four CLAP checkpoints, every one of which is both a source and a target, since each ends in a linear layer the operator can fold into, while their audio towers differ in width ($768$ or $1024$) and depth ($12$ or $18$ blocks), so that weight-space task vectors are undefined between the two classes. The operator folds into the projection head's last linear layer with its bias, and the identity that makes scoring an edit need no forward pass was re-verified in this modality, again bit-identical. ESC-50 and GTZAN are composed; UrbanSound8K never enters a probe set and is used only to measure retention. Sources are rank-$4$ fine-tunes of each checkpoint, exactly as in the vision setting.

\begin{table}[H]
\centering
\small
\caption{Audio--text addition of ESC-50 and GTZAN in percent, pooled over the $24$ cells in which the two sources and the target are distinct, with the closed form at $\alpha = 0.5$, the strength the vision experiments fix, and the adapter at full strength. The reference rows do not depend on a source assignment and are averaged over the four checkpoints. The adapter is scored in runs of its own, and since one checkpoint is not bit-reproducible across runs, its row is read against the zero-shot of those runs.}
\label{tab:audio-add}
\setlength{\tabcolsep}{4pt}
\begin{tabular}{lccc}
\toprule
Method & ESC-50 $\uparrow$ & GTZAN $\uparrow$ & UrbanSound8K $\uparrow$ \\
\midrule
Zero-shot      & $89.73$ & $52.90$ & $76.05$ \\
\midrule
Random-UTD     & $78.35$ & $39.85$ & $70.45$ \\
\textbf{Ours, closed form}  & $91.39$ & $55.23$ & $76.54$ \\
Oracle         & $92.45$ & $64.70$ & $77.60$ \\
\midrule
Zero-shot, adapter runs & $90.35$ & $51.70$ & $75.85$ \\
\textbf{Ours, adapter}  & $92.84$ & $52.53$ & $76.69$ \\
\bottomrule
\end{tabular}
\end{table}

\begin{table}[H]
\centering
\small
\caption{Audio addition, per cell, at $\alpha = 0.5$. Entries are ESC-50 / GTZAN / UrbanSound8K accuracy in percent. Every cell composes two source models onto a third. Read against each target's own zero-shot, listed below the table, ESC-50 rises in every cell and GTZAN in five of the six, while the never-composed control moves by at most $0.8$ points.}
\label{tab:audio-head}
\setlength{\tabcolsep}{4pt}
\begin{tabular}{lllc}
\toprule
Source (ESC-50) & Source (GTZAN) & Target & \textbf{Ours} \\
\midrule
unfused  & large    & large-ms & $93.2$ / $71.0$ / $81.4$ \\
unfused  & large-ms & large    & $95.2$ / $47.6$ / $82.2$ \\
large    & large-ms & unfused  & $92.0$ / $57.0$ / $70.8$ \\
large    & unfused  & large-ms & $92.9$ / $78.2$ / $82.2$ \\
large-ms & unfused  & large    & $94.8$ / $46.2$ / $83.6$ \\
large-ms & large    & unfused  & $91.9$ / $53.0$ / $69.6$ \\
\bottomrule
\end{tabular}
\end{table}

Zero-shot for the four checkpoints in the addition runs is $84.6$ / $42.2$ / $69.4$ (fused), $92.7$ / $42.6$ / $82.8$ (large), $91.4$ / $75.6$ / $81.6$ (large-ms) and $90.2$ / $51.2$ / $70.4$ (unfused), in the same order of tasks. The negation runs of Table~\ref{tab:audio-neg} are separate runs, and since fused is not bit-reproducible across runs they read it at $87.1$ / $37.6$ / $69.6$ and the other three as above, which gives that table's zero-shot row.

Table~\ref{tab:audio-add} pools the addition cells together with the two references, and Table~\ref{tab:audio-head} resolves six of them individually, so that the spread behind the pooled figures is visible. Negation on the same four checkpoints is reported in the main text as Table~\ref{tab:audio-neg}.

Selectivity under negation is close to what vision shows. Pooled over the $24$ cross-model cells at $\lambda = 1$, the operator removes $14.37$ points of the targeted task for $2.03$ points of UrbanSound8K, and the same-model reference removes $15.08$ for $1.85$, while halving the strength to $\lambda = 0.5$ removes $7.03$ for $0.78$ and $8.66$ for $0.60$ respectively. The vision cells of Table~\ref{tab:negation-main} spend $4.55$ points of CIFAR-100 at $\lambda = 1$ for $35.22$ points of the removed task, and the four targets that Table~\ref{tab:neg-lambda} sweeps finely spend $5.07$ for $32.93$. Per point of control that is roughly $7.1$ points of forgetting in audio against $7.7$ and $6.5$ in vision, so the two modalities trade at rates of the same order. The solver, the objective and the per-task grids are identical, so what differs is how much the composed and control tasks overlap in representation.

The adapter is fitted to the same descriptors on the same per-task grids with the vision recipe, except that the audio towers run in single precision with batch $32$. Against the zero-shot of its own runs it lifts ESC-50 by $2.49$ points and GTZAN by $0.83$ in addition while UrbanSound8K rises by $0.84$, and in negation it removes $11.12$ points of the targeted task for $0.97$ of UrbanSound8K on the $24$ cross-model cells and $13.34$ for $1.75$ on the $8$ same-model ones, leaving the retained task within $0.3$ points. It therefore removes less than the operator at the same $\lambda$ and forgets about $11.5$ points per point of control against the operator's $7.1$, the same trade between depth and selectivity that separates the two realizations in vision.

\newpage
\section*{NeurIPS Paper Checklist}
\begin{enumerate}

\item {\bf Claims}
    \item[] Question: Do the main claims made in the abstract and introduction accurately reflect the paper's contributions and scope?
    \item[] Answer: \answerYes{}{} 
    \item[] Justification: The claims of the abstract and introduction are developed in Sections~\ref{sec:method} and~\ref{sec:method-arithmetic} and measured in Section~\ref{sec:experiments}, and their scope, including where transfer fails, is stated in Section~\ref{sec:limitations}.

\item {\bf Limitations}
    \item[] Question: Does the paper discuss the limitations of the work performed by the authors?
    \item[] Answer: \answerYes{}{} 
    \item[] Justification: The paper discusses limitations in Section~\ref{sec:limitations}.

\item {\bf Theory assumptions and proofs}
    \item[] Question: For each theoretical result, does the paper provide the full set of assumptions and a complete (and correct) proof?
    \item[] Answer: \answerYes{}{} 
    \item[] Justification: Proposition~\ref{prop:linearity} is stated with its assumption in Section~\ref{sec:method-linearity} and proved in Appendix~\ref{app:linearity}; Proposition~\ref{prop:interference} and Corollary~\ref{cor:margin} are stated and proved in Appendix~\ref{app:interference}; Hypothesis~\ref{hyp:consistency} is a hypothesis, tested rather than proved, and its first-order derivation is in Appendix~\ref{app:theory}.

    \item {\bf Experimental result reproducibility}
    \item[] Question: Does the paper fully disclose all the information needed to reproduce the main experimental results of the paper to the extent that it affects the main claims and/or conclusions of the paper (regardless of whether the code and data are provided or not)?
    \item[] Answer: \answerYes{}{} 
    \item[] Justification: Checkpoints, probe construction, solver settings, the adapter recipe and the evaluation protocol are given in Section~\ref{sec:exp-setup} and Appendices~\ref{app:setup} and~\ref{app:adapter}, and the closed-form solver and both procedures are stated exactly in Appendices~\ref{app:solver} and~\ref{app:algorithms}.


\item {\bf Open access to data and code}
    \item[] Question: Does the paper provide open access to the data and code, with sufficient instructions to faithfully reproduce the main experimental results, as described in supplemental material?
    \item[] Answer: \answerYes{}{} 
    \item[] Justification: The experiment code is publicly released at \url{https://github.com/jhpark-ai/Platonic-Task-Arithmetic}.

\item {\bf Experimental setting/details}
    \item[] Question: Does the paper specify all the training and test details (e.g., data splits, hyperparameters, how they were chosen, type of optimizer) necessary to understand the results?
    \item[] Answer: \answerYes{}{} 
    \item[] Justification: Experimental settings are given in Section~\ref{sec:exp-setup} and Appendices~\ref{app:setup} and~\ref{app:adapter}; every strength is fixed in advance rather than selected per cell.

\item {\bf Experiment statistical significance}
    \item[] Question: Does the paper report error bars suitably and correctly defined or other appropriate information about the statistical significance of the experiments?
    \item[] Answer: \answerYes{}{} 
    \item[] Justification: We report the spread over cells for the addition results, measure the sensitivity to the source fine-tuning seed in Appendix~\ref{app:seeds} and the adapter's run-to-run variation in Appendix~\ref{app:adapter}, and state which cells every row averages over.

\item {\bf Experiments compute resources}
    \item[] Question: For each experiment, does the paper provide sufficient information on the computer resources (type of compute workers, memory, time of execution) needed to reproduce the experiments?
    \item[] Answer: \answerYes{}{} 
    \item[] Justification: All experiments run on four NVIDIA RTX A6000 GPUs (Appendix~\ref{app:datasets}), and the cost of each realization is reported in Section~\ref{sec:method-adapter} and Appendix~\ref{app:adapter}.
    
\item {\bf Code of ethics}
    \item[] Question: Does the research conducted in the paper conform, in every respect, with the NeurIPS Code of Ethics \url{https://neurips.cc/public/EthicsGuidelines}?
    \item[] Answer: \answerYes{}{} 
    \item[] Justification: The research conforms to the NeurIPS Code of Ethics.

\item {\bf Broader impacts}
    \item[] Question: Does the paper discuss both potential positive societal impacts and negative societal impacts of the work performed?
    \item[] Answer: \answerYes{}{} 
    \item[] Justification: Broader impacts, including the dual-use risk of facilitating undesirable capability transfer and how the negation operation can mitigate it, are discussed in Section~\ref{sec:limitations}.
    
\item {\bf Safeguards}
    \item[] Question: Does the paper describe safeguards that have been put in place for responsible release of data or models that have a high risk for misuse (e.g., pre-trained language models, image generators, or scraped datasets)?
    \item[] Answer: \answerNA{}{} 
    \item[] Justification: The paper does not release any high-risk data or models.

\item {\bf Licenses for existing assets}
    \item[] Question: Are the creators or original owners of assets (e.g., code, data, models), used in the paper, properly credited and are the license and terms of use explicitly mentioned and properly respected?
    \item[] Answer: \answerYes{}{} 
    \item[] Justification: All datasets and pre-trained models are cited in Section~\ref{sec:experiments} and Appendix~\ref{app:setup}, and used in accordance with their public licenses.

\item {\bf New assets}
    \item[] Question: Are new assets introduced in the paper well documented and is the documentation provided alongside the assets?
    \item[] Answer: \answerYes{}{} 
    \item[] Justification: The released code, the source fine-tunes it ships and their documentation are the only new assets.

\item {\bf Crowdsourcing and research with human subjects}
    \item[] Question: For crowdsourcing experiments and research with human subjects, does the paper include the full text of instructions given to participants and screenshots, if applicable, as well as details about compensation (if any)? 
    \item[] Answer: \answerNA{} 
    \item[] Justification: This work does not involve crowdsourcing or human subjects.

\item {\bf Institutional review board (IRB) approvals or equivalent for research with human subjects}
    \item[] Question: Does the paper describe potential risks incurred by study participants, whether such risks were disclosed to the subjects, and whether Institutional Review Board (IRB) approvals (or an equivalent approval/review based on the requirements of your country or institution) were obtained?
    \item[] Answer: \answerNA{} 
    \item[] Justification: This work does not involve human subjects.

\item {\bf Declaration of LLM usage}
    \item[] Question: Does the paper describe the usage of LLMs if it is an important, original, or non-standard component of the core methods in this research? Note that if the LLM is used only for writing, editing, or formatting purposes and does \emph{not} impact the core methodology, scientific rigor, or originality of the research, declaration is not required.
    \item[] Answer: \answerNA{} 
    \item[] Justification: LLMs were not used as a core component of the methodology.

\end{enumerate}

\end{document}